\documentclass{article}

\usepackage{xarxiv}

\usepackage[round]{natbib}
\usepackage[utf8]{inputenc} %
\usepackage[T1]{fontenc}    %
\usepackage{hyperref}       %
\usepackage{url}            %
\usepackage{booktabs}       %
\usepackage{amsfonts}       %
\usepackage{nicefrac}       %
\usepackage{microtype}      %
\usepackage{xcolor}         %

\usepackage{amsmath}
\usepackage{amssymb}
\usepackage{mathtools}
\usepackage{amsthm}

\usepackage[capitalize,noabbrev]{cleveref}

\usepackage[textsize=tiny]{todonotes}

\usepackage{chngcntr}   %
\usepackage{enumitem,nicefrac}
\usepackage{stmaryrd}		%
\usepackage[ngerman,english]{babel}     %
\usepackage{color}
\usepackage{algorithm} 
\usepackage{algorithmic} 
\usepackage[titletoc,title]{appendix}

\usepackage{dsfont}
\usepackage{tikz}
\usepackage{subfigure}
\usetikzlibrary{arrows.meta}
\usepackage{tabularx}
\usepackage{wrapfig}
\usepackage{bbm}
\usepackage{bm}

\newcommand{\R}{{\mathbb{R}}}

\newcommand{\N}{{\mathbb{N}}}

\newcommand{\E}{{\mathbb{E}}}
\renewcommand{\P}{{\mathbb{P}}}

\newcommand{\bsym}[1]{{\boldsymbol{ #1 }}}

\newcommand{\norm}[1]{\left| \left| #1 \right| \right|}

\newcommand{\ra}{\,\rightarrow\,}

\newcommand{\eps}{{\varepsilon}}

\newcommand{\argmax}{{\mathrm{argmax}}}
\newcommand{\ol}[1]{\overline{#1}}

\def\XXint#1#2#3{{\setbox0=\hbox{$#1{#2#3}{\int}$}
\vcenter{\hbox{$#2#3$}}\kern-.5\wd0}}

\DeclareMathOperator*{\argmin}{arg\,min}

\newcommand{\id}{1\hspace{-0,9ex}1}

\newtheorem{Def}{Definition}[section]

\newtheorem{Cor}[Def]{Corollary}

\newtheorem{Le}[Def]{Lemma}
\newtheorem{Thm}[Def]{Theorem}

\newtheorem{Rem}[Def]{Remark}
\counterwithout{figure}{section}
\counterwithout{figure}{section}
\title{Finding Optimal Arms in Non-stochastic Combinatorial Bandits with Semi-bandit Feedback and Finite Budget}

\author{
    Jasmin Brandt$^{a}$, Viktor Bengs$^{b},$ Bj\"orn Haddenhorst$^{a},$
	Eyke H\"ullermeier${}^{b,c}$\\
	${}^{a}$Department of Computer Science, Paderborn University, Germany\\
	${}^{b}$Institute of Informatics, University of Munich (LMU), Germany\\
	${}^{c}$Munich Center for Machine Learning, Germany\\
	\texttt{jasmin.brandt@upb.de, viktor.bengs@lmu.de,  bjoernha@mail.upb.de, eyke@lmu.de} 
}

\begin{document}

\maketitle

\begin{abstract}
We consider the combinatorial bandits problem with semi-bandit feedback under finite sampling budget constraints, in which the learner can carry out its action only for a limited number of times specified by an overall budget. The action is to choose a set of arms, whereupon feedback for each arm in the chosen set is received. Unlike existing works, we study this problem in a non-stochastic setting with subset-dependent feedback, i.e., the semi-bandit feedback received could be generated by an oblivious adversary and also might depend on the chosen set of arms. In addition, we consider a general feedback scenario covering both the numerical-based as well as preference-based case and introduce a sound theoretical framework for this setting guaranteeing sensible notions of optimal arms, which a learner seeks to find. We suggest a generic algorithm suitable to cover the full spectrum of conceivable arm elimination strategies from aggressive to conservative. Theoretical questions about the sufficient and necessary budget of the algorithm to find the best arm are answered and complemented by deriving lower bounds for any learning algorithm for this problem scenario.
\end{abstract}

\section{Introduction} \label{sec:introduction}

The multi-armed bandits (MAB) problem is an intensively studied problem class in the realm of machine learning, in which a learner is facing a sequential decision making problem under uncertainty \citep{lattimore2020bandit}.
A decision (action) corresponds to making a choice between a finite set of specific choice alternatives (objects, items, etc.), also called \emph{arms} in reference to the metaphor of gambling machines in casinos. 
After each decision to choose a particular arm, the learner receives some form of \emph{feedback} -- typically a numerical reward --  determined by a \emph{feedback mechanism} of the chosen arm.
The learner is not aware of the arms' feedback mechanisms and consequently tries to learn these in the course of time by performing actions according to its learning strategy.
The concrete design of its learning strategy depends essentially on two main components of the learning setting: the assumptions on the feedback mechanisms and the learning task.

Traditionally, and even up to now the most prevalent assumption is that the feedback received by choosing one arm is generated by means of a probability distribution of the chosen arm \citep{robbins1952some,lai1985asymptotically}.
In this way, any useful learning strategy revolves around learning specific probabilistic features of the arms' distributions such as the means.
These features, in turn, quite naturally provide a way to define a notion of \emph{(sub-)optimality of an arm} as well as a \emph{best arm}.
A relaxation of this \emph{stochastic setting} is the \emph{non-stochastic setting}, in which no assumption is made in the form of probabilistic laws of the feedback mechanisms.
Instead, either no assumptions are made on the feedback mechanisms, so that these can also be generated by an adversary \citep{auer1995gambling}, or that the sequence of feedback observations (or a transformation thereof) of an arm converges asymptotically to a fixed point \citep{Jamieson2016}.
In the latter case, the notion of an arm's (sub-)optimality is again straightforward, given that the limit points can be ordered, while in the former usually the best action in hindsight plays the role of the best arm.

Regarding the learning task, the most prominent one is that of \emph{regret minimization}, where in each learning round the learner suffers regret unless the optimal decision is made (determined by the feedback mechanisms).
The main challenge for the learner is to manage the trade-off between exploration and exploitation, i.e., constantly balancing (i) the degree of new information acquisition about some arms' feedback mechanism in order to appropriately expand the current knowledge base (exploration), and (ii) the degree of choosing arms considered to be optimal given the current knowledge base in order to keep the overall regret low (exploitation).
In many practical applications, however, the learning task is of a quite different kind, as the focus is rather on finding the (approximately) correct answer to a problem specific question, e.g., which arm is the (approximately) optimal one, within a reasonable time (number of learning rounds).
This \emph{pure exploration} learning task can be considered in two variants, namely the fixed confidence and the fixed budget setting.
In the former the learner tries to find the answer within as few as possible learning rounds, while guaranteeing a given pre-determined confidence for the correctness of its returned answer.
In the latter, it is the other way around, as the learner is provided with a limit on the possible number of learning rounds (budget) and the confidence for the returned answer should be as high as possible.
In both variants the main challenge for designing a suitable learner is to specify a clever exploration strategy for finding the correct answer.

In order to model more complex learning settings in practice, the basic setup of MAB problems has been generalized in various ways, such as incorporating additional side information \citep{abe1999associative,Au02,ChLiReSc11,AbPaSz11} or infinite number of arms \citep{bubeck2011x,munos2014bandits}, just to name a few.
Of special practical interest is the generalization of the basic setup, where the learner is allowed to choose specific sets of arms as its action.
Consider as an example the online algorithm selection problem with parallel runs, where for  sequentially arriving problem instances one selects a subset of available algorithms (solvers) to be run in parallel in order to solve the current problem instance.

If the feedback received is of a numerical nature, this variant has manifested itself under the term \emph{combinatorial bandits}  \citep{cesa2012combinatorial} while for feedback of a qualitative nature this variant can be referred to as \emph{preference-based bandits} as put forward by \cite{bengs2021preference}.
Combinatorial bandits are further distinguished with respect to the type of feedback between semi-bandit feedback, where feedback of each single arm in the selected subset is observed, and full bandit feedback, where only some aggregated value of the individual numerical feedback of the arms in the selected subset is observed.
Although both combinatorial and preference-based bandits consider a similar action set and for both the learner needs to deal with the possibly exponential size of the action set, the process of learning is quite different due to the nature of the observed feedback.
The main reason for this is that in preference-based feedback the mutual correlations that may exist between the arms in the chosen subset play a major role in both the assumption about the feedback mechanisms and the learning task from the outset.
In contrast to this, the standard setting in combinatorial bandits with semi-bandit feedback is that the individual reward generation mechanisms are independent of the chosen subsets.
However, this modeling assumption is questionable in a couple of practical applications, especially when humans provide the feedback. 
For example, in opinion polls or rating systems where humans rate a subset of objects (political parties/candidates, products, etc.), it is well known that the ratings of the objects may be affected by context effects, i.e., preferences in favor of an object may depend on what other objects are available. 
In the fields of economics and psychology, context effects are among others divided into compromise \citep{simonson1989choice}, attention \citep{huber1983market} and similarity \citep{tversky1972elimination} effects.

In this paper, we take a step towards unifying these two variants for the best arm identification (BAI) problem in a pure exploration learning setting with fixed budget and non-stochastic feedback mechanisms.
The main motivation for this unification is to derive a general purpose learner, which can tackle the BAI problem in both feedback variants.  
In this way, for example, one can transform a learning problem with numerical signals into a preference-based learning problem and thus conveniently apply such a general purpose learner. 
Recent works have demonstrated in two different learning scenarios with numerical feedback that such a transformation has great potential \citep{mohr2021single,kirschner2021bias}.

Needless to say, the main challenge is to unify both feedback variants through suitable abstractions allowing them to be treated as instantiations of the same problem class.
This bridge is built by dropping the common independence assumption (of the chosen arm set) for the numerical combinatorial bandits and abstracting the nature of the observations.
Additionally, we simply assume that the learner is provided with an appropriate statistic customized to the explicit nature of the feedback.
By appropriate choice of the statistic one obtains the respective setting, e.g., the empirical mean for the case of numerical feedback and relative frequencies for the case of preference feedback.

\textbf{Our contribution.} 
Under mild assumptions on the asymptotic behavior of these statistics, we derive a proper definition of a best arm a learner seeks to find (Section \ref{sec:problem_defi}) as well as lower bounds on the necessary budget for this task (Section \ref{sec_main_LBs}).
To the best of our knowledge, such lower bounds are novel for non-stochastic settings and the derivation is rather non-standard due to the combinatorial setup of the problem. 
We suggest a general algorithmic framework suitable to cover the full spectrum of conceivable arm elimination strategies from aggressive to conservative, which we analyze theoretically regarding the algorithms' sufficient and necessary budget to find the best arm (Section \ref{sec:Algos}).
As a consequence, we obtain to the best of our knowledge the first algorithm(s) for non-stochastic preference-based bandits as well as for combinatorial bandits under semi-bandit feedback, in which the individual (numerical) feedback received for an arm depends on the chosen subset due to possibly existing mutual correlations between the arms in the chosen subset.
The mild assumptions on the asymptotics of the statistics allow to transfer our theoretical results to the stochastic counterparts of the semi-bandit combinatorial and preference-based bandits (Section \ref{sec:application_stochastic}).
We demonstrate the usefulness of the generality of our setting in an experimental study for an algorithm selection problem with parallel runs (Section \ref{sec:experiments}), where once again the transformation of numerical feedback to preference feedback plays a key role.
Additional experiments are given in the supplementary material, where also all proofs of the theoretical results are collected.

\textbf{Related Work.} \label{sec:related_work}
A large body of literature considers the combinatorial bandit problem under preference-based feedback, see \cite{bengs2021preference} for an overview.
Although \emph{dueling bandits} \citep{YuBrKlJo12} has established as an overall agreed term for the scenario with actions of size two, the terminology for action sets of larger sizes is still discordant, e.g., multi-dueling \citep{BrSeCoLi16}, battling \citep{SaGo18a}, choice \citep{agarwal2020choice}
or preselection bandits \citep{bengs20}, mainly due to subtle nuances of the motivating practical applications.
While pure exploration settings with a stochastic preference-based feedback haven been considered by a series of works \citep{MoSuEl17,ren2019sample,chen2020combinatorialduel,ren20a,Saha2020b,haddenhorst2021identification}, a pure exploration setting under a non-stochastic feedback mechanism as in our case has yet to be studied.

Pure exploration has been intensively studied in the basic multi-armed bandits (MAB) setting with stochastic feedback mechanisms as well, see Section 33.5 in \cite{lattimore2020bandit} for a detailed overview.
The non-stochastic variant of the fixed budget MAB setting is considered in \cite{Jamieson2016}, which is the backbone for the well-known Hyperband algorithm \cite{Li2017} and additionally inspired in some part the assumptions we make for our work.
Initiated by the work of \cite{bubeck2012best} to design learners for regret minimization frameworks which can perform well in both stochastic and non-stochastic settings, the fixed budget framework has been the subject of research by \cite{abbasi2018best} and \cite{shen2019universal}.

Combinatorial bandits with numerical feedback have been introduced by \cite{cesa2012combinatorial} and \cite{chen2013combinatorial} in a regret minimization framework.
The fixed confidence setting for stochastic combinatorial bandits with semi-bandit feedback is studied by \cite{jourdan2021efficient}, and full bandit feedback in  \cite{chen2020combinatorial,kuroki2020polynomial}.
Finally, the best-of-$k$ bandits game introduced in \cite{simchowitz2016best}, which in some way unifies the combinatorial bandits with binary set-dependent feedback and preference-based bandits in one joint framework similarly as we do in this work.
However, they consider a fixed confidence setting with stochastic feedback mechanisms and do not provide a learner for the dependent arm case, although they derive lower bounds on the worst case sample complexity for this case.
The only work assuming a set-dependent feedback mechanism in combinatorial bandits with semi-bandit feedback is by \citet{yang2021combinatorial}, where, however, the regret minimization task is studied under stochastic feedback mechanisms.
In summary, there seems to be no existing work which considers a pure exploration setting for combinatorial bandits with non-stochastic or even stochastic semi-bandit feedback, where the (mean) rewards of the arms in the chosen subset of arms depends on the subset.
Accordingly, our results provide new theoretical contributions to this field.

\section{Problem Formulation} \label{sec:problem_defi}
In our setup, we assume a set $\mathcal{A}$ of $n$ arms, which we simply identify by their indices, i.e., $\mathcal{A} = [n] = \{1,\dots,n\}.$
For some fixed $k<n$ we denote the set of all possible subsets of arms with size of at least $2$ and at most $k$ by $\mathcal{Q}_{\le k} = \{Q \subseteq \mathcal{A} ~|~ 2 \le |Q| \le k \}.$ 
Further, we assume that for any $Q \in \mathcal{Q}_{\le k}$ we can query some feedback, in the form of a feedback vector $\mathbf{o}_Q = (o_{i|Q})_{i\in Q} \in D^{|Q|}$ which in turn can be of numerical or qualitative nature specified by the domain $D$.
If we query a subset of arms $Q$ for $t$ many times, then $\mathbf{o}_Q(t)$ is the corresponding feedback vector. 
We suppose that we are given a suitable statistic $s$ for the type of observation vectors, which maps a multiset of observations to some specific value relevant for the decision making.
With this, $\mathbf{s}_{Q}(t) = (s_{i|Q}(t))_{i\in Q}$ is the statistic vector derived by the sequence of feedback  $(\mathbf{o}_{Q}(t))_t$ of the query set $Q \in \mathcal{Q}_{\le k}$, and $s_{i|Q}(t) = s(\{o_{i|Q}(1),\ldots,o_{i|Q}(t)\}) $  is the relevant statistic for decision making about arm $i$ in the ``context'' $Q$ after querying $Q$ for $t$ many times.

\paragraph{Examples.}
For  combinatorial bandits  with semi-bandit feedback, the observation $\mathbf{o}_{Q}(t) = (o_{i|Q}(t))_{i\in Q}$ corresponds to the reward one obtains for each arm $i\in Q$ by using $Q$ for the $t$-th time, so that in particular $D=\R.$
The most natural statistic in this case is the empirical mean given for a multiset $O$ of observations by $s(O)=\frac{1}{|O|}\sum_{x \in O} x$, such that $s_{i|Q}(t) = \frac{1}{t} \sum_{t'=1}^{t} o_{i|Q}(t'),$ which is also the arguably most prevalent statistic used in the realm of bandit problems for guiding the decision making process. 
However, other statistics $s$ such as quantiles or the expected shortfall are of interest as well \citep{cassel2018general}.   

In the preference-based bandit setting with winner feedback we observe after the $t$-th usage of the query set $Q$ only a binary (winner) information, i.e., $o_{i|Q}(t) = 1$ if arm $i$ is preferred over the other arms in $Q$ at ``pull'' $t$ and $o_{i|Q}(t) = 0$ otherwise, so that $D=\{0,1\}.$ 
Once again the empirical mean of these binary observations is a quite intuitive choice for the statistic $s,$ as in a stochastic feedback setting the corresponding statistic vector $s_{i|Q}(t) = \frac{1}{t} \sum_{t'=1}^{t} o_{i|Q}(t)$ would converge to the probability vector determining how likely an arm will be preferred over all the other arms in the query set $Q$.
For preference-based bandits with full ranking feedback we observe after the $t$-th usage of the query set $Q$ an entire ranking of the arms in $Q$, i.e., $o_{i|Q}(t)$ is arm $i$'s rank among the arms in $Q$ at ``pull'' $t$, so that $D=\{1,\ldots,k\}.$ 
In such a case the statistic $s$ might be a positional scoring rule \citep{korba2018learning}.

\paragraph{Goal.} The goal of the learner in our setting is to find a or the best arm (specified below) within a fixed budget of at most $B$ samples (numbers of queries).
For any $Q$, write $n_{Q}(t)$ for the number of times $Q$ has been queried until (including) time $t$. 
An algorithm, which tackles the problem, chooses at time $t$ a set $Q_{t}\in \mathcal{Q}_{\le k}$ and observes as feedback $\mathbf{o}_{Q_t}(n_{Q_t}(t))$ leading to an update of the relevant statistic vector $\mathbf{s}_{Q_t}(n_{Q_t}(t)) = (s_{i|Q_t}(n_{Q_t}(t)))_{i\in Q_t}$.

\paragraph{Best arm.} 
Inspired by the theoretical groundings of Hyperband \citep{Jamieson2016,Li2017} for best arm identification problems in numerical bandit problems with non-stochastic rewards, we make the following assumption regarding the limit behavior of the statistics
\begin{align*}
   (A1) : \forall Q \in \mathcal{Q}_{\le k}\  \forall i\in Q \, :  
   \, S_{i|Q} \coloneqq \lim\nolimits_{t\ra \infty} s_{i|Q}(t) \text{ exists}.
\end{align*}
This assumption is in general slightly looser than assuming (stationary) stochastic feedback mechanisms, as (A1) is fulfilled for many prevalent statistics by means of a limits theorem such as the strong law of large numbers.
Conceptionally, our Assumption (A1) is similar to the assumption on the sequence of losses in \citet{Jamieson2016}, as both have in common that the statistics (losses in \cite{Jamieson2016}) converge to some fixed point, respectively.
However, due to the difference of the action spaces (single arms vs.\ set of arms) and the nature of the feedback (scalar vs.\ vector observation), our assumption can be seen as a combinatorial extension of the one in \cite{Jamieson2016}.

Given assumption (A1) a straightforward notion of a best arm is obtained by leveraging the idea of a Borda winner from dueling bandits with pairs of arms as the possible subsets to more general subsets of arms.
A \emph{generalized Borda winner} (GBW) is then an arm which has on average the largest asympotical statistic, i.e.,
\begin{align*}
	i_{\mathcal B}^{\ast} \in \arg \max\limits_{i \in \mathcal{A}} S_{i}^{\mathcal B} = \arg \max\limits_{i \in \mathcal{A}} \frac{\sum\nolimits_{Q \in \mathcal{Q}_{=k}(i)} S_{i|Q}}{|\mathcal{Q}_{=k}(i)|},
\end{align*}
where $\mathcal{Q}_{=k}(i) = \{Q \in \mathcal{Q}_{=k} ~|~ i\in Q\}$ and $S_{i}^{\mathcal B}$ are the asymptotic Borda scores, i.e., the limits according to (A1) of $s_{i}^{\mathcal B}(t) \coloneqq \frac{\sum\nolimits_{Q \in \mathcal{Q}_{=k}(i)} s_{i|Q}(t)}{|\mathcal{Q}_{=k}(i)|} .$
Similarly, a \emph{generalized Copeland winner} (GCopeW) is an arm $i$, which wins w.r.t. the asymptotic statistics on average on the most query sets, i.e., 
\begin{equation*}
    i_{\mathcal C}^{\ast} \in \arg \max\limits_{i \in \mathcal{A}} S_{i}^{\mathcal C} = \arg \max\limits_{i \in \mathcal{A}} \frac{\sum_{Q\in \mathcal{Q}_{=k}(i)} \id \{S_{i|Q} = S_{(1)|Q}\}}{|\mathcal{Q}_{=k}(i)|}.
\end{equation*}

However, both these notions of best arm have two major drawbacks, as there might be multiple GBWs and GCopeWs and due to averaging over all subsets in their definition, there is no way to identify a GBW or a GCopeW within a sampling budget of $o(\binom{n-1}{k-1})$ in the worst case (see  Theorem~\ref{Thm_Main_LB_GCW}).

In light of these drawbacks, we specify another reasonable notion of a best arm, for which we leverage the concept of the \emph{generalized Condorcet winner} \citep{haddenhorst2021identification,agarwal2020choice} from the preference-based bandits literature.
For this purpose, we introduce the following assumption
\begin{align*}
    (A2):~ &\exists i^{\ast} \in \mathcal{A} \mbox{ such that } \forall Q \in \mathcal{Q}_{\le k}(i^{\ast}) ~ \forall j\in Q\setminus \{i^{\ast}\} \mbox{ it holds that } S_{i^{\ast}|Q} > S_{j|Q},
\end{align*}
where $\mathcal{Q}_{\leq k}(i) = \{Q \in \mathcal{Q}_{\leq k} ~|~ i\in Q\}$  for $i\in[n].$
We call $i^{\ast}$ the 
generalized Condorcet winner (GCW), which is the arm dominating all the other arms in each possible query set containing it.
It is worth noting that such an arm may not exist, but if it exists, then it is arguably the most natural way to define the optimal arm, even though it may differ from the GBW. 
Nevertheless, the existence of the generalized Condorcet winner (or simply the Condorcet winner for the case $k=2$) is a common assumption in the preference-based bandits literature \cite{agarwal2020choice,haddenhorst2021identification,bengs2021preference}.
Additionally, we will show below that identifying a GCW is possible for a sampling budget of size $\Omega(n/k)$ even in worst case scenarios.

\paragraph{Problem characteristics.}
In light of (A1) and (A2), there are two key characteristics which will determine the appeal of any learner in our setting.
The first one is the speed of convergence of the statistics $s_{i|Q}$ to their limit values $S_{i|Q}.$
More precisely, the function $\gamma_{i|Q}:\N \ra \R$, which is the point-wise smallest non-increasing function fulfilling
$ |s_{i|Q}(t)-S_{i|Q}| \leq  \gamma_{i|Q}(t)$ for any $t \in \N,$
 plays a major role in characterizing the difficulty of the learning problem.
Moreover, the worst speed of convergence function of a  query set $Q\in \mathcal{Q}_{\leq k}$ given by  $\ol{\gamma}_{Q}(t) \coloneqq \max\nolimits_{i\in Q} \gamma_{i|Q}(t)$
and the overall worst speed of convergence function $    \ol{\gamma}(t) \coloneqq \max\nolimits_{Q\in \mathcal{Q}_{\leq k}} \ol{\gamma}_{Q}(t) $
will be of relevance as well.
Assuming a stochastic setting, the role of $\gamma_{i|Q}$ is played by the minimax rate of convergence of the statistic to its population counterpart, e.g., $1/\sqrt{t}$ for the empirical mean and the expected value. 
Usually the speed of convergence functions will appear implicitly by means of their (quasi-) inverses given by $\gamma_{i|Q}^{-1}(\alpha) \coloneqq \min\{t \in \mathbb{N} ~|~ \gamma_{i|Q}(t) \le \alpha \}$, $\bar{\gamma}_{Q}^{-1}(t) \coloneqq \min_{i\in Q} \gamma^{-1}_{i|Q}(t)$ and $\bar{\gamma}^{-1}(t) \coloneqq \min_{Q \in \mathcal{Q}_{\le k}} \bar{\gamma}_{Q}^{-1}(t)$.

The other relevant problem characteristic are the gaps of the limits statistics, i.e.,  
$\Delta_{i|Q} \coloneqq S_{i^*|Q} - S_{i|Q}$   for $i\in[n],$ $Q \in \mathcal{Q}_{\leq k}(i) \cap \mathcal{Q}_{\leq k}(i^*).$ 
Such gaps are prevalent in the realm of bandit problems, as they can be used to define a measure of (sub-)optimality of an arm  in the stochastic feedback case. 
Note that in our setting this is not straightforward, as the gaps are depending on the query set and more importantly the speed of convergence has a decisive impact on the severeness of these gaps.

\section{Lower Bounds}\label{sec_main_LBs}
\newcommand{\Alg}{\mathrm{Alg}}
Let us abbreviate $\mathbf{S}\coloneqq (S_{i|Q})_{Q\in \mathcal{Q}_{\leq k},i\in Q}$ and $\bsym{\gamma} \coloneqq (\gamma_{i|Q}(t))_{Q\in \mathcal{Q}_{\leq k},i\in Q,t\in \N}$.  Given $\mathbf{S}$ and $\bsym{\gamma}$, write $\mathfrak{S}(\mathbf{S},\bsym{\gamma})$ for the set of all $\mathbf{s} = (s_{i|Q}(t))_{Q\in \mathcal{Q}_{\leq k},i\in Q,t\in \N}$ that fulfill 
\begin{itemize}	[noitemsep,topsep=0pt,leftmargin=7.5mm]
	\item[(i)]
	$\forall Q\in \mathcal{Q}_{\leq k}, i \in Q$: $S'_{i|Q} = \lim_{t\ra \infty} s_{i|Q}(t)$ exists,
	\item[(ii)]
	$\forall Q \in \mathcal{Q}_{\leq k}, i \in Q, t \in \N$: $|s_{i|Q}(t)-S'_{i|Q}| \leq \gamma_{i|Q}(t)$,
	\item[(iii)]
	$\forall Q \in \mathcal{Q}_{\leq k}$: $\exists \pi_{Q}:Q \ra Q$ bijective such that $S'_{i|Q} = S_{\pi(i)|Q}$ for all $i\in Q$. 
\end{itemize}
For $Q$ and $l\in \{1,\dots,|Q|\}$ write $S_{(l)|Q}$ for the $l$-th order statistic of $\{S_{i|Q}\}_{i\in Q}$, i.e., $\{S_{i|Q}\}_{i\in Q} = \{S_{(l)|Q}\}_{l\in Q}$ and $S_{(1)|Q} \geq \dots \geq S_{(|Q|)| Q}. $
If $\Alg$ is a (possibly probabilistic) sequential algorithm, we denote by $B(\Alg,\mathbf{s})$ the number of queries made by $\Alg$ before termination  when started on $\mathbf{s}$.
In the following, we provide lower bounds on $\E[B(\Alg,\mathbf{s})]$ for algorithms $\Alg$, which identify, for any instance $\mathbf{s} \in \mathfrak{S}(\mathbf{S},\bsym{\gamma})$, almost surely one GCW resp.\ GBW resp.\ GCopeW of $\mathbf{s}$.
\begin{Thm}\label{Thm_Main_LB_GCW}
    Let $\mathbf{S}$ be such that $S_{(1)|Q} > S_{(2)|Q}$ for all $Q\in \mathcal{Q}_{\leq k}$.\\
    (i) There exists $\mathbf{s} \in \mathfrak{S}(\mathbf{S},\bsym{\gamma})$ such that if $\Alg$ correctly identifies a GCW for any instance in $\mathfrak{S}(\mathbf{S},\bsym{\gamma})$, then 
    \begin{align*}
        \E\left[ B(\Alg,\mathbf{s}) \right] \geq \left\lceil \frac{n}{k} \right\rceil   \min\limits_{Q\in \mathcal{Q}_{\leq k},j\in Q} \, \gamma_{j|Q}^{-1}\left( \frac{S_{(1)|Q} - S_{(|Q|)|Q}}{2}\right).
    \end{align*}
    (ii) Assume $(S_{(1)|Q},\dots,S_{(|Q|)|Q})$ does not depend on $Q$ for any $Q\in \mathcal{Q}_{=k}.$ If $\Alg$ correctly identifies a GBW for any instance in $\mathfrak{S}(\mathbf{S},\bsym{\gamma})$, then 
    \begin{align*}
				     \sup\limits_{\mathbf{s} \in \mathfrak{S}(\mathbf{S},\bsym{\gamma})}\E \left[ B(\Alg,\mathbf{s}) \right] = \Omega\left(\binom{n-1}{k-1}\right).
	\end{align*}
	(iii) If $\Alg$ correctly identifies a GCopeW for any instance in $\mathfrak{S}(\mathbf{S},\bsym{\gamma})$, then it fulfills
	\begin{align*}
				     \sup\limits_{\mathbf{s} \in \mathfrak{S}(\mathbf{S},\bsym{\gamma})}\E \left[ B(\Alg,\mathbf{s}) \right] = \Omega\left(\binom{n-1}{k-1}\right).
	\end{align*}
\end{Thm}
The theorem is proven in Section \ref{sec:proofs_lower_bounds} in the supplement, where we provide in fact slightly stronger versions of these bounds.

\section{Algorithms} \label{sec:Algos}
In this section, we present a class of algorithms along with three possible instantiations solving the corresponding learning task for the case of the generalized Condorcet winner being the best arm.
In particular, we analyze all algorithms theoretically in terms of their sufficient and necessary budget to find the respective best arm.
In light of the results in Theorem \ref{Thm_Main_LB_GCW} for GBW and GCopeW identification, it is straightforward that a simple algorithm, which enumerates all possible subsets, pulls each of them equally often in a round-robin fashion, and returns the empirical GBW (or GCopeW) is already optimal.
For sake of completeness, we show this result for the case of GBW identification in Section \ref{sec:round_robin} and also include this simple algorithm in the experimental study below (called \textsc{RoundRobin}).

\subsection{Generalized Condorcet Winner Identification}  \label{sec:gen_framework}
\begin{minipage}{0.45\textwidth}
In the following we introduce a general class of algorithms in Algorithm \ref{alg:Framework} which is instantiable with various different elimination strategies of the arms. 
Below, we present some instantiations which build on commonly used elimination strategies in the standard multi-armed bandit setting.
The idea of Algorithm \ref{alg:Framework} is simply to maintain a set of active arms, which is successively reduced by following a specific arm elimination strategy (Algorithm \ref{alg:ae}) referred to as elimination rounds.
In each elimination round $r \in \{1,2,\ldots,R\}$ the set of active arms $\mathbb{A}_r$ is partitioned into $P_r$ many sets of size $k$ (up to a possible remainder set) denoted by $(\mathbb{A}_{r,j})_{j\in P_r}$ for which the elimination strategy is applied with a roundwise-dependent budget $b_r.$
The budget allocated to a partition $\mathbb{A}_{r,j}$ in round $r$ is of the form $b_r = \lceil B / (R \cdot P_r) \rceil$ following the idea to split up the available budget equally first for each round and second for all partitions in each round.
The explicit arm elimination strategy used in Algorithm \ref{alg:Framework} is specified by Algorithm \ref{alg:ae} and corresponds to pulling the chosen query set $Q$ for a fixed number of times and afterwards keeping only the best $f(|Q|)$ arms of $Q$. 
Here, $f:[k] \ra [k]$ is an arbitrary function with $f(x) \leq x-1$ for all $x,$ which essentially determines the aggressiveness or conservativeness of an arm elimination strategy as we will see below. 

In the following we provide three possible instantiations of Algorithm \ref{alg:Framework} inspired by commonly used elimination strategies in the standard multi-armed bandit setting for pure exploration tasks.
\paragraph{Combinatorial Successive Winner Stays.} 
The most aggressive elimination strategy is to keep only the arm with the best statistic (the winner) of each partition in each round and discard all others from the set of active arms for the next round.
Concretely, we use $f^{\mathrm{CSWS}}(s) = 1$ for $f$ in Algorithm \ref{alg:Framework} in this case.

\end{minipage}
\hfill
\begin{minipage}{0.52\textwidth}
\begin{algorithm}[H]
	\caption{Combinatorial Successive Elimination} \label{alg:Framework}
	 \textbf{Input:} set of arms $[n]$, subset size $k\le n$, sampling budget $B,$ 
	 a function $f:[k] \ra [k]$,
	 sequence $\{P_r\}_{r}$ (number of partitions at round $r$), $R$ (number of rounds in total)\\
	 \textbf{Initialization:}
	$\mathbb{A}_1 \leftarrow [n]$, 
	$r \leftarrow 1$\\
	\vspace*{-.3cm}
	\begin{algorithmic}[1]
	    \WHILE {$|\mathbb{A}_r| \ge k$}
			\STATE    $b_{r} \leftarrow \lfloor \nicefrac{B}{(P_r R)}\rfloor$ , $J \leftarrow P_{r}$
	        \STATE $\mathbb{A}_{r,1}, \mathbb{A}_{r,2},\dots,\mathbb{A}_{r,J} \leftarrow \mathrm{Partition}(\mathbb{A}_r, k)$
	        \IF {$|\mathbb{A}_{r,J}| < k $}
	            \STATE $\mathcal{R} \leftarrow \mathbb{A}_{r,J}$, $J \leftarrow J-1$
            \ELSE 
	            \STATE $\mathcal{R} \leftarrow \emptyset$
	       \ENDIF
	       \STATE $\mathbb{A}_{r+1} \leftarrow \mathcal{R}$
	       \FOR {$j \in [J]$}
	       		\STATE $\mathcal{R} \leftarrow \mathrm{ArmElimination}(\mathbb{A}_{r,j},b_r,$ \\
	       		$f(|\mathbb{A}_{r,j}|))$
				\STATE $ \mathbb{A}_{r+1} \leftarrow \mathbb{A}_{r+1} \cup \mathcal{R}$
	       \ENDFOR
	       \STATE $r \leftarrow r+1$
	   \ENDWHILE
	   \STATE $\mathbb{A}_{r+1} \leftarrow \emptyset$
	   \WHILE {$|\mathbb{A}_r| > 1$}
	        \STATE $\mathbb{A}_{r+1} \leftarrow  \mathrm{ArmElimination}(\mathbb{A}_{r+1},b_r,$\\
	        $f(|\mathbb{A}_{r+1}|))$
	        , $r \leftarrow r+1$
	   \ENDWHILE
	\end{algorithmic} 
	\hspace*{3pt} \textbf{Output:} The remaining item in $\mathbb{A}_r$
\end{algorithm}

\begin{algorithm}[H]
	\caption{ArmElimination$(\mathbb{A}',b,l)$} 
	\label{alg:ae}
	\begin{algorithmic}[1]
		\STATE Use $\mathbb{A}'$ for $b$ times
		\STATE For all $i\in \mathbb{A}'$, update $s_{i|\mathbb{A}'}(b)$
		\STATE Choose an ordering $i_{1},\dots,i_{|\mathbb{A}'|}$ of $(s_{i|\mathbb{A}'}(b))_{i\in \mathbb{A}'}$
		\STATE \textbf{return } $\{i_{1},\dots,i_{l}\}$
	\end{algorithmic} 
\end{algorithm}
\end{minipage}



The resulting instantiation of Algorithm \ref{alg:Framework} is called \emph{Combinatorial Successive Winner Stays} (CSWS), which has at most $ R^{CSWS} =  \lceil \log_k(n) \rceil +1 $ many rounds in total (at most $\lceil \log_k(n) \rceil$ rounds in the first while-loop and at most 1 in the second).
The total number of partitions in round $r$ is at most $ P_{r}^{CSWS} =  \left\lceil \nicefrac{n}{k^r} \right\rceil.$

\paragraph{Combinatorial Successive Reject.}  
On the other extreme regarding the aggressiveness of the arm elimination strategy is to dismiss only the worst arm of each partition in each round and keep all others in the set of active arms for the next round.
More specifically, we use $f^{\mathrm{CSR}}(s)=s-1$ for this variant, which can be seen as a variant of the Successive Reject algorithm \citep{Audibert2010} for best arm identification adopted to the combinatorial bandit problem.
Consequently, we call the resulting instantiation of Algorithm \ref{alg:Framework}  the \emph{Combinatorial Successive Reject} (CSR) algorithm, whose number of rounds in the first while-loop is at most $\lceil\log_{\nicefrac{k-1}{k}}\left(\nicefrac{1}{n}\right)\rceil$ and in the second at most $k-1$. 
Overall, we have a maximal number of rounds $R^{CSR} = \lceil\log_{\nicefrac{k-1}{k}}\left(\nicefrac{1}{n}\right)\rceil+k-1$ and a maximal number of partitions per round $P_{r}^{CSR} = \lceil\nicefrac{n(1-\frac{1}{k})^{(r-1)}}{k}\rceil$.

\paragraph{Combinatorial Successive Halving.} 
As a compromise between the aggressive elimination strategy of CSWS and the conservative elimination strategy of CSR one could discard in every elimination round the worse half of all arms in the partition, i.e., using $f^{\mathrm{CSH}}(s) = \left\lceil \nicefrac{s}{2} \right\rceil$ for $f$ in Algorithm \ref{alg:Framework}.
This can be seen as a generalization of the successive halving algorithm ~\citep{Jamieson2016} adopted to the combinatorial bandit problem we are considering.
Thus, the instantiation of Algorithm \ref{alg:Framework} in this spirit will be called the \emph{Combinatorial Successive Halving} (CSH) algorithm.
Note that we have at most $\lceil \log_2(n) \rceil$ rounds in the first while-loop and additional $\lceil \log_2(k) \rceil$ in the second while-loop resulting in at most $R^{CSH} = \lceil \log_2(n) \rceil + \lceil \log_2(k) \rceil$ many rounds throughout a run of CSH. 
Furthermore, we have at most $P_{r}^{CSH} = \left\lceil \frac{n}{2^{r-1} k} \right\rceil$ partitions in round $r$.

\subsection{Theoretical Guarantees} \label{sec:theoretical_guarantees}
In the following, we derive the sufficient budget for Algorithm \ref{alg:Framework} to return under assumptions (A1) and (A2) the best arm $i^{\ast}$, i.e., the generalized Condorcet winner.
For this purpose, we write $\mathbb{A}_{r}(i^{\ast})$ for the unique set $\mathbb{A}_{r,j} \in \{\mathbb{A}_{r,1},\dots,\mathbb{A}_{r, P_r}\}$ with $i^{\ast} \in \mathbb{A}_{r,j}$ 
and define $\Delta_{(l)|Q} = S_{i^{\ast}|Q} - S_{(l)|Q}$ for any $Q \subseteq[n]$ with $i^{\ast}\in Q.$
\begin{Thm}\label{Thm_generalSufficient}
    Assume $P_{r}$, $R$  are such that Algorithm \ref{alg:Framework} called with $B$ does not exceed $B$ as total budget.
	Under Assumptions (A1) and (A2) Algorithm \ref{alg:Framework} returns $i^{\ast}$ if $B$ is larger than 
	\begin{align*}
	   z&\left(f,R,\{P_{r}\}_{1\leq r\leq R}\right) \coloneqq R \max\nolimits_{r\in [R]} P_{r} \cdot \big\lceil \bar{\gamma}_{\mathbb{A}_{r}(i^{\ast})}^{-1}\big(\nicefrac{\Delta_{(f(|\mathbb{A}_{r}(i^{\ast})|)+1)|\mathbb{A}_{r}(i^{\ast})}}{2}\big) \big\rceil
	\end{align*}	
\end{Thm}
The following theorem indicates optimality of $z$ in the theorem above (cf.\ Sec.\ \ref{sec_proofs_general_framework} for the proofs).
\begin{Thm}\label{Thm:general_necessity}
    For any distinct asymptotic values $\mathbf{S}$, there exists a family of statistics $\{s_{i|Q}(t)\}_{t\in \N, Q\in \mathcal{Q}_{\le k},i\in Q}$ with $ s_{i|Q}(t) \ra S_{i|Q}$ for all $i\in [n],Q\in \mathcal{Q}_{\le k}$ such that if Algorithm \ref{alg:Framework} is used with a budget $B<z\left(f,R,\{P_{r}\}_{1\leq r\leq R}\right)$
    then it does \textbf{not} return $i^{\ast}$.
\end{Thm}
By means of Theorem \ref{Thm_generalSufficient}, we can infer the following result regarding the sufficient sampling budget $B$ for the three instantiations to output $i^{\ast}$ (cf.\ Sec.\ \ref{sec_proofs_corollaries} of the appendix for the proof).

\begin{Cor} \label{Cor:CSXsufficiency}
	Under Assumptions (A1) and (A2), $\mathrm{CSX} \in \{\mathrm{CSWS}, \mathrm{CSR}, \mathrm{CSH}\}$ returns $i^{\ast}$ if it is executed with a budget $B \geq z_{\mathrm{CSX}}$, where
		$
		z_{\mathrm{CSX}} \coloneqq z\left(f^{\mathrm{CSX}}, R^{\mathrm{CSX}} , \{ P_{r}^{\mathrm{CSX}} \}_{1\leq r\leq R^{\mathrm{CSX}}}\right).$
\end{Cor}
By substituting the concrete values for $P_r$, $R$ and $f$ of the corresponding instantiation into Corollary \ref{Cor:CSXsufficiency} and using a rough estimate for the inverse function of the speed of convergence, we see that all of the resulting sufficient budgets are essentially $\tilde{\mathcal{O}}( \nicefrac{n}{k} )$ (see Table \ref{tab:AlgoComparison_main})  almost\footnote{Here, $\tilde{\mathcal{O}}$ hides logarithmic factors.} matching the dependency on $n$ and $k$ in Theorem \ref{Thm_Main_LB_GCW}. 
If we would allow the special case of singleton sets of arms as query sets, i.e., $k=1,$ the sufficient budget for CSH matches the one derived in \cite{Jamieson2016} for its non-combinatorial counterpart in the special case of numerical feedback.

\begin{table*}   [!ht]
\caption{Sufficient budget for CSWS, CSR and CSH. Here, $\pi$ is as in (iii) in Section \ref{sec_main_LBs}.}
  \label{tab:AlgoComparison_main}
      \centering
		\resizebox{0.95\textwidth}{!}{
    \begin{tabular}[t]{l|l}
        \toprule[1pt]
        \rule{0pt}{20pt} $z_{CSWS} $
        & $ \left\lceil \frac{n}{k} \right\rceil (\lceil \log_k(n) \rceil +1)  \cdot \max_{Q \in \mathcal{Q}_{\leq k} : i^{\ast}\in Q} \max_{i\in Q \setminus \{i^{\ast}\}} \left\lceil \bar{\gamma}^{-1} \left( \frac{S_{i^{\ast}|Q} - S_{i|Q}}{2}\right) \right\rceil$  \\
        \rule{0pt}{20pt} $z_{CSR} $
        & $\left\lceil\frac{n}{k}\right\rceil \left(\left\lceil\log_{1-\frac{1}{k}}\left(\frac{1}{n}\right)\right\rceil+k-1\right)$
        $ \cdot \max_{Q \in \mathcal{Q}_{\leq k}: i^{\ast}\in Q}  \min_{i\in Q\backslash \{i^{\ast}\}} \left\lceil \bar{\gamma}^{-1} \left( \frac{S_{i^{\ast}|Q} - S_{i|Q}}{2}\right) \right\rceil$\\
        \rule{0pt}{20pt} $z_{CSH}$
        & $ \lceil \frac{n}{k} \rceil \left(\left\lceil \log_2(n) \right\rceil + \lceil \log_2(k) \rceil \right)$
        $\cdot\max_{Q \in \mathcal{Q}_{\leq k}: i^{\ast}\in Q} \left\lceil \bar{\gamma}^{-1} \left( \frac{S_{i^{\ast}|Q} - S_{\pi(Q)|Q}}{2}\right) \right\rceil$  \\
         \bottomrule[1pt]
    \end{tabular}}
\end{table*}

Regarding $n$ and $k$ both lower and upper bounds coincide, but the gap-term in the lower bounds include a min-term over $\mathcal{Q}_{\leq k}$, while the gap-term in the upper bound are coming with a max-term over $\mathcal{Q}_{\leq k}$. The difference between these terms depends on the underlying hardness of the bandit problem in terms of $\bar{ \gamma}^{-1}$, i.e., how fast the considered statistics converge to their limit values. Due to the generality of our setting it is difficult to specify this difference more explicitly and it would be worth considering this for special cases, i.e., the numerical bandits or preference-based bandits separately.

Finally, it is worth mentioning that all of the three instantiations of Algorithm \ref{alg:Framework} have only been studied for the case of single arm pulls, but not for pulls of subsets of arms, where additionally a dependency on the set might be present. Thus, the theoretical guarantees are novel in this regard.

\section{Applications to Stochastic Settings} \label{sec:application_stochastic}

\textbf{Numerical feedback.}
In stochastic combinatorial bandits  \citep{chen2013combinatorial}, each arm-query set pair $(i,Q)$ is associated with a probability distribution $\nu_{i|Q}$ and querying $Q$ for the $t$-th time results in the feedback $o_{i|Q}(t) \sim \nu_{i|Q},$ usually referred to as a reward (i.e., $D=\R$).
The sequence of rewards  $\{o_{i|Q}(t)\}_{t}$ is supposed to be independent and the statistic $s$ is the empirical mean such that (A1) holds by the law of large numbers with $S_{i|Q} = \E_{X\sim \nu_{i|Q}}[X].$ 
If the $\nu_{i|Q}$ are sub-Gaussian, an anytime confidence bound by \cite{Jamieson14} based on the law of iterated logarithm ensures $|s_{i|Q}(t) - S_{i|Q}| \leq c_{\delta}(t)$ for all $t \in \N$ with probability at least $1-\delta$ for some appropriate function $c_{\delta}(t) \in \mathcal{O}(\sqrt{t\ln(\ln(t)/\delta)})$. This implies the following result, the proof of which is deferred to Section \ref{Appendix:proofsAppl}.

\begin{Cor}\label{Cor:ApplReward}
    Let $f,R$ and $\{P_{r}\}_{r\in [R]}$ be as in Theorem \ref{Thm_generalSufficient} and suppose the reward distributions $\nu_{i|Q}$ to be $\sigma$-sub-Gaussian and such that their means $S_{i|Q}$ satisfy (A2). There is a function $C(\delta,\eps,k,R, \sigma)$ in $ \mathcal{O}\left( \sigma^2 \eps^{-2}\ln\left( \nicefrac{kR}{\delta}\ln\left(\nicefrac{kR\sigma}{\eps \delta}\right)\right)\right)$
    such that if $i^{\ast}$ is the optimal arm for  $(S_{i|Q})_{Q\in \mathcal{Q}_{\leq k},i\in Q}$ and  $\sup_{Q\in \mathcal{Q}_{\leq k}(i^{\ast})} \Delta_{(f(|Q|)+1)|Q} \leq \eps$, then Algorithm \ref{alg:Framework} used with a budget $B$ larger than $  C(\delta,\eps,k,R,\sigma) \cdot R \max\nolimits_{r\in [R]} P_{r} $
    returns  $i^{\ast}$ with probability at least $ 1-\delta$.
\end{Cor}

\textbf{Other statistics for numerical feedback.}
A rich class of statistics can be obtained by applying a linear functional $U(F)= \int r(x) \mathrm{d}F(x)$, where $F$ is a cumulative distribution function (CDF) and $r: \R \rightarrow \mathbb{R}$ some measurable function  \citep{wasserman2013all}, on the empirical CDF, i.e.,  
    $\tilde s(O,x) = |O|^{-1} \sum_{o \in O} \id\{ x\leq o\},$
for any $x \in \R$ and any multiset of (reward) observations $O.$ 
This leads to the statistics 
    $$s_{i|Q}(t) = U(\tilde s(o_{i|Q}(1), \dots, o_{i|Q}(t),\cdot)) =  \sum\nolimits_{s=1}^{t}  \frac{r(o_{i|Q}(s))}{t},$$
which converge to $S_{i|Q} = \E_{X\sim \nu_{i|Q}}[r(X)]$ by the law of large numbers, provided these expected values exist.
For this class of statistics we can show a quite similar result to Corollary \ref{Cor:ApplReward} by generalizing the findings in \cite{Jamieson14}  (see Section \ref{Appendix:proofsAppl}). 
However, the result is in fact more general than Corollary \ref{Cor:ApplReward} as for $r$ being the identity function we can recover the empirical mean.

\textbf{Preference feedback.}
In the preference-based bandits, we observe when querying $Q$ for the $t$-th time a categorical random variable with values in $Q$, i.e., $o_{i|Q}(t) \sim \mathrm{Cat}_{Q}(\mathbf{p}_{Q})$ for some underlying unknown parameter $\mathbf{p}_{Q}=(p_{i|Q})_{i\in Q}$. 
Let $w_{i|Q}(t)\coloneqq \sum_{s\leq t} \id\{ o_{i|Q}(t) = i  \}$ be the number of times arm $i$ has won in the query set $Q$ until the $t$-th pull of $Q$. We consider as the relevant statistics $s_{i|Q}(t) = \frac{w_{i|Q}(t)}{t},$ which 
converge  to $p_{i|Q} \eqqcolon S_{i|Q}$ by the law of large numbers.
The Dvoretzky-Kiefer-Wolfowitz inequality \citep{Dvoretzky1956} ensures a concentration inequality on $ \sup\nolimits_{i\in Q} |s_{i|Q}(t) - S_{i|Q}|,$
 which can be used to deduce the following result (cf.\ Sec.\ \ref{Appendix:proofsAppl} for the proof).

\begin{Cor}\label{Cor:ApplPB}
    Let $f$,$R$ and $\{P_{r}\}_{r\in [R]}$ be as in Theorem \ref{Thm_generalSufficient} and suppose preference-based winner feedback 
    with parameter $(p_{i|Q})_{Q\in \mathcal{Q}_{\leq k},i\in Q}$, which satisfies (A2).
    There is a function $       C(\delta,\eps,k,R) \in\mathcal{O}\left( \eps^{-2}\ln\left( \nicefrac{R}{\delta \eps^{4}}\right)\right)$
    with the following property: If $i^{\ast}$ is the optimal arm and $\sup_{Q\in \mathcal{Q}_{\leq k}(i^{\ast})} \Delta_{(f(|Q|)+1)|Q} \leq \eps$, then Algorithm \ref{alg:Framework} used with a budget $B$ larger than $
        C(\delta,\eps,k,R) \cdot R\max_{r\in [R]} P_{r}$
    returns $i^{\ast}$ with probability at least $1-\delta$.
\end{Cor}
By substituting the concrete values for $P_r$, $R$ and $f$ of the corresponding instantiation of Algorithm \ref{alg:Framework} into the bound on the budget in Corollary \ref{Cor:ApplPB} (compare to Table \ref{tab:AlgoComparison_main}), we see that each of the three resulting bounds almost matches the optimal sample complexity bounds for identifying the (generalized) Condorcet Winner under fixed confidence in preference-based bandits \cite{bengs2021preference,haddenhorst2021identification} indicating near optimality of the algorithms in stochastic settings.
However, since no stochastic counterpart of our combinatorial setting for the numerical case exists, it would be interesting to investigate whether the analogous implication by means of Corollary \ref{Cor:ApplReward} for the three algorithms is nearly optimal as well.
We leave this to future work, as it is beyond the scope of our work.

\section{Experimental Section} \label{sec:experiments}
In this section we present an experimental study for our proposed algorithms %
on an algorithm selection problem.
Further experiments, also on synthetic data and with other statistics are provided in the supplementary material in Section \ref{sec:further_exp}.

\textbf{Setting.}
In the following, we consider an algorithm selection problem, where the goal is to select the most efficient algorithm for solving an instance of a satisfiability (SAT) problem. For this, we randomly chose $n=20$ parameterizations of the SAPS solver \citep{TomHutHoo07} which represent our candidate algorithms and correspond to the arms in our terminology. Our possible problem instances are sampled from the first 5000 problem instances from the sat$\_$SWGCP folder of the AClib\footnote{http://www.aclib.net}.
We compare CSWS, CSR, CSH and \textsc{RoundRobin} on this problem with the Successive Halving (SH) algorithm \cite{Jamieson2016}. 
To the best of our knowledge, there are no algorithms available as baselines, which are designed for the pure exploration problem with finite budget and subsets of arms as the actions, e.g., \cite{agarwal2020choice} investigates a regret minimization problem, while \cite{haddenhorst2021identification} is dealing with a stochastic pure exploration setting with fixed-confidence.
However, Successive Halving serves as a baseline, which we included as a representative for the algorithms dealing with a pure exploration problem with finite budget and single arms as the actions.
In each learning round, we randomly draw a problem instance from the 5000 problem instances without replacement and then start a parallel solution process with the SAPS parameterizations chosen by the corresponding learning algorithm (only one parameterization for the case of SH), where the process is stopped as soon as the first algorithm has solved the current instance.
In particular, one obtains only for the ``finisher'' SAPS parameterization an explicit numerical value (its runtime) among the chosen set of SAPS parameterizations, as the others are right-censored.
Since our proposed algorithms are designed for the case, in which feedback for all arms in the pulled query set is observed, while SH is designed  for the case in which only a single arm is queried resulting in a single feedback, we enlarge the available budget for SH to $k \cdot B$ for a fairer comparison.

\textbf{Instantiation of CSE.}
Although we could consider the negative runtimes of the parameterizations as rewards (i.e., runtimes correspond to losses) and use a statistic suitable for numerical feedback for the combinatorial successive elimination (CSE) approaches, there might be a major disadvantage due to the censored runtimes.
Indeed, in order to apply a statistic suitable for numerical feedback,  some sensible imputation technique is required to deal with the censored observations, which in turn could introduce a severe bias.
However, thanks to the generality of our framework, we can simply interpret the observed feedback as a preference in the sense that the ``finisher'' SAPS parameterization is preferred over the others in the chosen set of parameterizations.
In this way, using a statistic based on preference-based feedback defuses the bias issue. 
Quite naturally, we use the relative frequency statistic for preference feedback as specified in Section \ref{sec:application_stochastic}.

\textbf{Analysis.}
As the best arm (SAPS parameterization) we use the one having the smallest empirical runtime over all problem instances such that \textsc{RoundRobin} and SH will tendentially return this arm if the budget is sufficiently large.
The resulting success rates for our proposed algorithms and SH of identifying the best arm are illustrated in the top panel of Figure \ref{fig:ACSuccessRate}.
One can see, that the algorithms which follow our CSE strategy significantly outperform SH if the budget is sufficiently large. In addition,  CSWS, CSR and CSH identify the best arm more often than \textsc{RoundRobin} if the subset sizes $k$ are small, which is a realistic situation in practice.
Moreover, the bottom panel in Figure \ref{fig:ACSuccessRate} shows the overall runtimes of the algorithms revealing that SH takes much longer than CSWS, CSR and CSH and as expected the difference in the runtimes gets larger with the subset sizes $k$. 
Quite interestingly, even \textsc{RoundRobin} needs a longer runtime than the CSE approaches, although it queries the same number of subsets and also stops the respective run as soon as the ``finisher'' SAPS parameterization is clear. Thus, the differences in the runtimes of \textsc{RoundRobin} and the CSE approaches are only due to the fact that the latter discard the slowest SAPS parameterizations quickly and do not run them again, while \textsc{RoundRobin} uses throughout all subsets the same amount of time, even if they contain only bad performing parameterizations.
In other words, the differences are due to the sophisticated strategies of the CSE approaches.

\begin{figure}[t]
    \centering
    \includegraphics[width=0.87\linewidth]{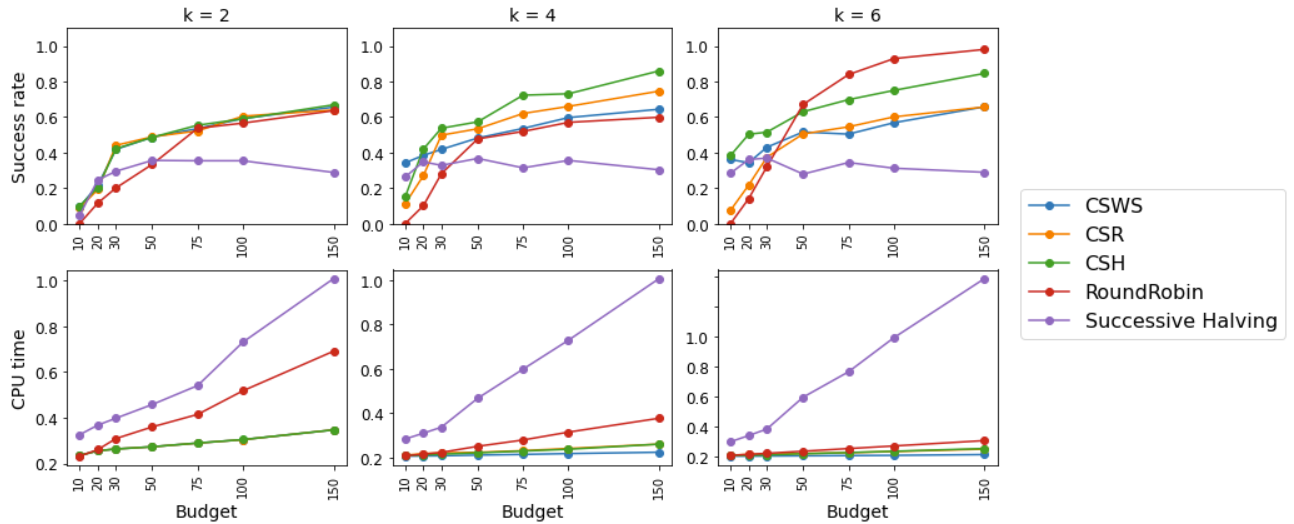}
    \caption{Success rates and runtimes for different subset sizes $k$ and budgets $B$.}
    \label{fig:ACSuccessRate}
\end{figure}

\section{Future Work}

For future work, it would be interesting to investigate whether switching the elimination strategy during the learning process leads to any performance improvements both theoretically and empirically.
A similar question could be asked regarding the considered statistic for numerical feedback variants.
Further, the goal of identifying the best set of arms in our scenario would also be interesting. 
However, in the case where the observations depend on the chosen set of arms, it is far from obvious how to define a suitable optimality term (cf. Sec.\ 6.3.2 in \cite{bengs2021preference}).
Finally, a more extensive experimental study would definitely be a relevant future direction of research, especially for hyperparameter optimization problems with possible parallelization options such as in \citep{li2020system} or for more general algorithm configuration problems \cite{schede2022AC}.

\section*{Acknowledgments and Disclosure of Funding}

This work was partially supported by the German Research Foundation (DFG) within the project ``Online Preference Learning with Bandit Algorithms'' (project no. 317046553)
and by the research training group ``Dataninja'' (Trustworthy AI for Seamless Problem Solving: Next Generation Intelligence Joins Robust Data Analysis) funded by the German federal state of North Rhine-Westphalia.

\bibliography{references}

\begin{thebibliography}{52}
\providecommand{\natexlab}[1]{#1}
\providecommand{\url}[1]{\texttt{#1}}
\expandafter\ifx\csname urlstyle\endcsname\relax
  \providecommand{\doi}[1]{doi: #1}\else
  \providecommand{\doi}{doi: \begingroup \urlstyle{rm}\Url}\fi

\bibitem[Abbasi-Yadkori et~al.(2011)Abbasi-Yadkori, P\'{a}l, and
  Szepesv\'{a}ri]{AbPaSz11}
Yasin Abbasi-Yadkori, D\'{a}vid P\'{a}l, and Csaba Szepesv\'{a}ri.
\newblock {Improved Algorithms for Linear Stochastic Bandits}.
\newblock In \emph{Proceedings of Advances in Neural Information Processing
  Systems (NeurIPS)}, pages 2312--2320, 2011.

\bibitem[Abbasi-Yadkori et~al.(2018)Abbasi-Yadkori, Bartlett, Gabillon, Malek,
  and Valko]{abbasi2018best}
Yasin Abbasi-Yadkori, Peter Bartlett, Victor Gabillon, Alan Malek, and Michal
  Valko.
\newblock Best of both worlds: Stochastic \& adversarial best-arm
  identification.
\newblock In \emph{Proceedings of Annual Conference on Learning Theory
  {(COLT)}}, pages 918--949, 2018.

\bibitem[Abe and Long(1999)]{abe1999associative}
Naoki Abe and Philip~M Long.
\newblock Associative reinforcement learning using linear probabilistic
  concepts.
\newblock In \emph{Proceedings of the International Conference on Machine
  Learning (ICML)}, pages 3--11, 1999.

\bibitem[Agarwal et~al.(2020)Agarwal, Johnson, and Agarwal]{agarwal2020choice}
Arpit Agarwal, Nicholas Johnson, and Shivani Agarwal.
\newblock Choice bandits.
\newblock In \emph{Proceedings of Advances in Neural Information Processing
  Systems (NeurIPS)}, 2020.

\bibitem[Audibert et~al.(2010)Audibert, Bubeck, and Munos]{Audibert2010}
Jean~Yves Audibert, S{\'{e}}bastien Bubeck, and R{\'{e}}mi Munos.
\newblock {Best arm identification in multi-armed bandits}.
\newblock In \emph{Proceedings of Annual Conference on Learning Theory
  {(COLT)}}, pages 41--53, 2010.

\bibitem[Auer(2002)]{Au02}
Peter Auer.
\newblock Using confidence bounds for exploitation-exploration trade-offs.
\newblock \emph{Journal of Machine Learning Research}, 3\penalty0
  (Nov):\penalty0 397--422, 2002.

\bibitem[Auer et~al.(1995)Auer, Cesa-Bianchi, Freund, and
  Schapire]{auer1995gambling}
Peter Auer, Nicolo Cesa-Bianchi, Yoav Freund, and Robert~E Schapire.
\newblock Gambling in a rigged casino: The adversarial multi-armed bandit
  problem.
\newblock In \emph{Proceedings of IEEE 36th Annual Foundations of Computer
  Science}, pages 322--331. IEEE, 1995.

\bibitem[Bengs and H{\"u}llermeier(2020)]{bengs20}
Viktor Bengs and Eyke H{\"u}llermeier.
\newblock Preselection bandits.
\newblock In \emph{Proceedings of the International Conference on Machine
  Learning (ICML)}, pages 778--787, 2020.

\bibitem[Bengs et~al.(2021)Bengs, Busa-Fekete, El~Mesaoudi-Paul, and
  H{\"u}llermeier]{bengs2021preference}
Viktor Bengs, R{\'o}bert Busa-Fekete, Adil El~Mesaoudi-Paul, and Eyke
  H{\"u}llermeier.
\newblock Preference-based online learning with dueling bandits: A survey.
\newblock \emph{Journal of Machine Learning Research}, 22:\penalty0 1--108,
  2021.

\bibitem[Brost et~al.(2016)Brost, Seldin, Cox, and Lioma]{BrSeCoLi16}
Brian Brost, Yevgeny Seldin, Ingemar~J. Cox, and Christina Lioma.
\newblock Multi-dueling bandits and their application to online ranker
  evaluation.
\newblock In \emph{Proceedings of ACM International Conference on Information
  and Knowledge Management {(CIKM)}}, pages 2161--2166, 2016.

\bibitem[Bubeck and Slivkins(2012)]{bubeck2012best}
S{\'e}bastien Bubeck and Aleksandrs Slivkins.
\newblock The best of both worlds: Stochastic and adversarial bandits.
\newblock In \emph{Proceedings of Annual Conference on Learning Theory
  {(COLT)}}, pages 42.1--42.23, 2012.

\bibitem[Bubeck et~al.(2011)Bubeck, Munos, Stoltz, and
  Szepesv{\'a}ri]{bubeck2011x}
S{\'e}bastien Bubeck, R{\'e}mi Munos, Gilles Stoltz, and Csaba Szepesv{\'a}ri.
\newblock X-armed bandits.
\newblock \emph{Journal of Machine Learning Research}, 12\penalty0 (5), 2011.

\bibitem[Cassel et~al.(2018)Cassel, Mannor, and Zeevi]{cassel2018general}
Asaf Cassel, Shie Mannor, and Assaf Zeevi.
\newblock A general approach to multi-armed bandits under risk criteria.
\newblock In \emph{Proceedings of Annual Conference on Learning Theory
  {(COLT)}}, pages 1295--1306, 2018.

\bibitem[Cesa-Bianchi and Lugosi(2012)]{cesa2012combinatorial}
Nicol\`{o} Cesa-Bianchi and G{\'a}bor Lugosi.
\newblock Combinatorial bandits.
\newblock \emph{Journal of Computer and System Sciences}, 78\penalty0
  (5):\penalty0 1404--1422, 2012.

\bibitem[Chen et~al.(2013)Chen, Wang, and Yuan]{chen2013combinatorial}
Wei Chen, Yajun Wang, and Yang Yuan.
\newblock Combinatorial multi-armed bandit: General framework and applications.
\newblock In \emph{Proceedings of the International Conference on Machine
  Learning (ICML)}, pages 151--159, 2013.

\bibitem[Chen et~al.(2020)Chen, Du, Huang, and Zhao]{chen2020combinatorialduel}
Wei Chen, Yihan Du, Longbo Huang, and Haoyu Zhao.
\newblock Combinatorial pure exploration for dueling bandit.
\newblock In \emph{Proceedings of the International Conference on Machine
  Learning (ICML)}, pages 1531--1541. PMLR, 2020.

\bibitem[Chu et~al.(2011)Chu, Li, Reyzin, and Schapire]{ChLiReSc11}
Wei Chu, Lihong Li, Lev Reyzin, and Robert~E. Schapire.
\newblock Contextual bandits with linear payoff functions.
\newblock In \emph{Proceedings of International Conference on Artificial
  Intelligence and Statistics {(AISTATS)}}, pages 208--214, 2011.

\bibitem[Du et~al.(2021)Du, Kuroki, and Chen]{chen2020combinatorial}
Yihan Du, Yuko Kuroki, and Wei Chen.
\newblock Combinatorial pure exploration with full-bandit or partial linear
  feedback.
\newblock In \emph{Proceedings of the {AAAI} Conference on Artificial
  Intelligence {(AAAI)}}, pages 7262--7270, 2021.

\bibitem[Dvoretzky et~al.(1956)Dvoretzky, Kiefer, and Wolfowitz]{Dvoretzky1956}
Aryeh Dvoretzky, Jack Kiefer, and Jacob Wolfowitz.
\newblock Asymptotic minimax character of the sample distribution function and
  of the classical multinomial estimator.
\newblock \emph{The Annals of Mathematical Statistics}, 27\penalty0
  (3):\penalty0 642 -- 669, 1956.

\bibitem[Garivier and Kaufmann(2016)]{Garivier2016}
Aur{\'e}lien Garivier and Emilie Kaufmann.
\newblock Optimal best arm identification with fixed confidence.
\newblock In \emph{Proceedings of Annual Conference on Learning Theory
  {(COLT)}}, pages 998--1027, 2016.

\bibitem[Haddenhorst et~al.(2021)Haddenhorst, Bengs, and
  H{\"u}llermeier]{haddenhorst2021identification}
Bj{\"o}rn Haddenhorst, Viktor Bengs, and Eyke H{\"u}llermeier.
\newblock Identification of the generalized {C}ondorcet winner in multi-dueling
  bandits.
\newblock \emph{Proceedings of Advances in Neural Information Processing
  Systems (NeurIPS)}, 34, 2021.

\bibitem[Huber and Puto(1983)]{huber1983market}
Joel Huber and Christopher Puto.
\newblock Market boundaries and product choice: Illustrating attraction and
  substitution effects.
\newblock \emph{Journal of Consumer Research}, 10\penalty0 (1):\penalty0
  31--44, 1983.

\bibitem[Hutter et~al.(2002)Hutter, Tompkins, and Hoos]{TomHutHoo07}
Frank Hutter, Dave~AD Tompkins, and Holger~H Hoos.
\newblock Scaling and probabilistic smoothing: Efficient dynamic local search
  for {SAT}.
\newblock In \emph{International Conference on Principles and Practice of
  Constraint Programming}, pages 233--248. Springer, 2002.

\bibitem[Jamieson and Talwalkar(2016)]{Jamieson2016}
Kevin Jamieson and Ameet Talwalkar.
\newblock Non-stochastic best arm identification and hyperparameter
  optimization.
\newblock In \emph{Proceedings of International Conference on Artificial
  Intelligence and Statistics {(AISTATS)}}, pages 240--248, 2016.

\bibitem[Jamieson et~al.(2014)Jamieson, Malloy, Nowak, and Bubeck]{Jamieson14}
Kevin Jamieson, Matthew Malloy, Robert Nowak, and S\'ebastien Bubeck.
\newblock lil'{UCB} : An optimal exploration algorithm for multi-armed bandits.
\newblock In \emph{Proceedings of Annual Conference on Learning Theory
  {(COLT)}}, volume~35, pages 423--439, 2014.

\bibitem[Jourdan et~al.(2021)Jourdan, Mutn{\`y}, Kirschner, and
  Krause]{jourdan2021efficient}
Marc Jourdan, Mojm{\'\i}r Mutn{\`y}, Johannes Kirschner, and Andreas Krause.
\newblock Efficient pure exploration for combinatorial bandits with semi-bandit
  feedback.
\newblock In \emph{Proceedings of the International Conference on Algorithmic
  Learning Theory {(ALT)}}, pages 805--849, 2021.

\bibitem[Kirschner and Krause(2021)]{kirschner2021bias}
Johannes Kirschner and Andreas Krause.
\newblock Bias-robust {B}ayesian optimization via dueling bandits.
\newblock In \emph{Proceedings of the International Conference on Machine
  Learning (ICML)}, pages 5595--5605, 2021.

\bibitem[Korba(2018)]{korba2018learning}
Anna Korba.
\newblock \emph{Learning from ranking data: theory and methods}.
\newblock PhD thesis, Universit{\'e} Paris-Saclay (ComUE), 2018.

\bibitem[Kosorok(2008)]{Kosorok2008}
Michael~R. Kosorok.
\newblock \emph{Introduction to Empirical Processes and Semiparametric
  Inference}.
\newblock Springer, New York, USA, 2008.

\bibitem[Kuroki et~al.(2020)Kuroki, Xu, Miyauchi, Honda, and
  Sugiyama]{kuroki2020polynomial}
Yuko Kuroki, Liyuan Xu, Atsushi Miyauchi, Junya Honda, and Masashi Sugiyama.
\newblock Polynomial-time algorithms for multiple-arm identification with
  full-bandit feedback.
\newblock \emph{Neural Computation}, 32\penalty0 (9):\penalty0 1733--1773,
  2020.

\bibitem[Lai and Robbins(1985)]{lai1985asymptotically}
Tze~Leung Lai and Herbert Robbins.
\newblock Asymptotically efficient adaptive allocation rules.
\newblock \emph{Advances in Applied Mathematics}, 6\penalty0 (1):\penalty0
  4--22, 1985.

\bibitem[Lattimore and Szepesv{\'a}ri(2020)]{lattimore2020bandit}
Tor Lattimore and Csaba Szepesv{\'a}ri.
\newblock \emph{Bandit Algorithms}.
\newblock Cambridge University Press, 2020.

\bibitem[Li et~al.(2020)Li, Jamieson, Rostamizadeh, Gonina, Ben-tzur, Hardt,
  Recht, and Talwalkar]{li2020system}
Liam Li, Kevin Jamieson, Afshin Rostamizadeh, Ekaterina Gonina, Jonathan
  Ben-tzur, Moritz Hardt, Benjamin Recht, and Ameet Talwalkar.
\newblock A system for massively parallel hyperparameter tuning.
\newblock \emph{Proceedings of Machine Learning and Systems (MLSys)},
  2:\penalty0 230--246, 2020.

\bibitem[Li et~al.(2017)Li, Jamieson, DeSalvo, Rostamizadeh, and
  Talwalkar]{Li2017}
Lisha Li, Kevin Jamieson, Giulia DeSalvo, Afshin Rostamizadeh, and Ameet
  Talwalkar.
\newblock Hyperband: A novel bandit-based approach to hyperparameter
  optimization.
\newblock \emph{Journal of Machine Learning Research}, 18\penalty0
  (1):\penalty0 6765–6816, 2017.

\bibitem[Massart(1990)]{Massart1990}
Pascal Massart.
\newblock The tight constant in the {D}voretzky{-K}iefer{-W}olfowitz
  inequality.
\newblock \emph{The Annals of Probability}, 18\penalty0 (3):\penalty0 1269 --
  1283, 1990.

\bibitem[Mohajer et~al.(2017)Mohajer, Suh, and Elmahdy]{MoSuEl17}
Soheil Mohajer, Changho Suh, and Adel Elmahdy.
\newblock Active learning for top-$k$ rank aggregation from noisy comparisons.
\newblock In \emph{Proceedings of International Conference on Machine Learning
  {(ICML)}}, pages 2488--2497, 2017.

\bibitem[Mohr et~al.(2021)Mohr, Bengs, and H{\"u}llermeier]{mohr2021single}
Felix Mohr, Viktor Bengs, and Eyke H{\"u}llermeier.
\newblock Single player {Monte-Carlo} tree search based on the {Plackett-Luce
  Model}.
\newblock In \emph{Proceedings of the {AAAI} Conference on Artificial
  Intelligence {(AAAI)}}, volume~35, pages 12373--12381, 2021.

\bibitem[Munos(2014)]{munos2014bandits}
R{\'e}mi Munos.
\newblock From bandits to {Monte-Carlo} tree search: {T}he optimistic principle
  applied to optimization and planning.
\newblock \emph{Foundations and Trends{\textregistered} in Machine Learning},
  7\penalty0 (1):\penalty0 1--129, 2014.

\bibitem[Ren et~al.(2019)Ren, Liu, and Shroff]{ren2019sample}
Wenbo Ren, Jia Liu, and Ness Shroff.
\newblock On sample complexity upper and lower bounds for exact ranking from
  noisy comparisons.
\newblock In \emph{Proceedings of Advances in Neural Information Processing
  Systems (NeurIPS)}, pages 10014--10024, 2019.

\bibitem[Ren et~al.(2020)Ren, Liu, and Shroff]{ren20a}
Wenbo Ren, Jia Liu, and Ness Shroff.
\newblock The sample complexity of best-$k$ items selection from pairwise
  comparisons.
\newblock In \emph{Proceedings of the International Conference on Machine
  Learning (ICML)}, pages 8051--8072, 2020.

\bibitem[Robbins(1952)]{robbins1952some}
Herbert Robbins.
\newblock Some aspects of the sequential design of experiments.
\newblock \emph{Bulletin of the American Mathematical Society}, 58\penalty0
  (5):\penalty0 527--535, 1952.

\bibitem[Saha and Gopalan(2018)]{SaGo18a}
Aadirupa Saha and Aditya Gopalan.
\newblock Battle of bandits.
\newblock In \emph{Proceedings of Conference on Uncertainty in Artificial
  Intelligence {(UAI)}}, pages 805--814, 2018.

\bibitem[Saha and Gopalan(2020)]{Saha2020b}
Aadirupa Saha and Aditya Gopalan.
\newblock From {PAC} to instance{-}optimal sample complexity in the
  {P}lackett{-}{L}uce model.
\newblock In \emph{Proceedings of International Conference on Machine Learning
  {(ICML)}}, pages 8367--8376, 2020.

\bibitem[Schede et~al.(2022)Schede, Brandt, Tornede, Wever, Bengs,
  H{\"{u}}llermeier, and Tierney]{schede2022AC}
Elias Schede, Jasmin Brandt, Alexander Tornede, Marcel Wever, Viktor Bengs,
  Eyke H{\"{u}}llermeier, and Kevin Tierney.
\newblock A survey of methods for automated algorithm configuration.
\newblock \emph{Journal of Artificial Intelligence Research}, 75, 2022.

\bibitem[Shen(2019)]{shen2019universal}
Cong Shen.
\newblock Universal best arm identification.
\newblock \emph{IEEE Transactions on Signal Processing}, 67\penalty0
  (17):\penalty0 4464--4478, 2019.

\bibitem[Simchowitz et~al.(2016)Simchowitz, Jamieson, and
  Recht]{simchowitz2016best}
Max Simchowitz, Kevin Jamieson, and Benjamin Recht.
\newblock Best-of-k-bandits.
\newblock In \emph{Proceedings of Annual Conference on Learning Theory
  {(COLT)}}, pages 1440--1489, 2016.

\bibitem[Simonson(1989)]{simonson1989choice}
Itamar Simonson.
\newblock Choice based on reasons: The case of attraction and compromise
  effects.
\newblock \emph{Journal of Consumer Research}, 16\penalty0 (2):\penalty0
  158--174, 1989.

\bibitem[Tversky(1972)]{tversky1972elimination}
Amos Tversky.
\newblock Elimination by aspects: A theory of choice.
\newblock \emph{Psychological Review}, 79\penalty0 (4):\penalty0 281, 1972.

\bibitem[Wasserman(2013)]{wasserman2013all}
Larry Wasserman.
\newblock \emph{All of Statistics: A concise Course in Statistical Inference}.
\newblock Springer Science \& Business Media, 2013.

\bibitem[Yang et~al.(2021)Yang, Ren, Dhillon, and
  Sanghavi]{yang2021combinatorial}
Shuo Yang, Tongzheng Ren, Inderjit~S Dhillon, and Sujay Sanghavi.
\newblock Combinatorial bandits without total order for arms.
\newblock \emph{arXiv preprint arXiv:2103.02741}, 2021.

\bibitem[Yao(1977)]{yao1977probabilistic}
Andrew Chi-Chin Yao.
\newblock Probabilistic computations: Toward a unified measure of complexity.
\newblock In \emph{18th Annual Symposium on Foundations of Computer Science
  {(SFCS)}}, pages 222--227, 1977.

\bibitem[Yue et~al.(2012)Yue, Broder, Kleinberg, and Joachims]{YuBrKlJo12}
Yisong Yue, Josef Broder, Robert Kleinberg, and Thorsten Joachims.
\newblock The $k$-armed dueling bandits problem.
\newblock \emph{Journal of Computer and System Sciences}, 78\penalty0
  (5):\penalty0 1538--1556, 2012.

\end{thebibliography}

\clearpage

\appendix
\onecolumn

\thispagestyle{empty}

\section{List of Symbols}
The following table contains a list of symbols that are frequently used in the main paper as well as in the following supplementary material. \\ \medskip
\small
\begin{tabularx}{\columnwidth}{lX}
    \hline
	\multicolumn{2}{c}{\textbf{Basics}} \\
	\hline
	$\id\{ \cdot  \}$ & indicator function \\
	$\N$  & set of natural numbers (without 0), i.e., $\N = \{1,2,3,\dots\}$ \\
	$\R$ & set of real numbers\\
	$D$ & observation domain (categorical or numerical)\\
	$\mathcal{A} = [n]$ & set of arms \\
    $n$ & number of arms\\
    $k$ & maximal possible subset size \\
    $B$ & budget for the learner \\ 
    $\mathcal{Q}_{\leq k}$ & all subsets of $\mathcal{A}$ of size $\leq k$: %
    $\{Q \subseteq \mathcal{A} ~|~ 2 \leq |Q| \leq k \}$\\
    $\mathcal{Q}_{\leq k}(i)$ & all subsets in $\mathcal{Q}_{\leq k }$ which contain arm $i$: $\{ Q \in \mathcal{Q}_{\leq k} ~|~ i \in Q\}$\\
    $\mathcal{Q}_{= k}$ & all subsets of $\mathcal{A}$ of size $k:$ $\{Q \subseteq \mathcal{A} ~|~ |Q| = k \}$\\
    $\mathcal{Q}_{= k}(i)$ & all subsets of $\mathcal{A}$ of size $k$ which contain arm $i:$ $\{ Q \in \mathcal{Q}_{= k} ~|~ i \in Q\}$\\
    $\mathbf{o}_Q(t)$ & observed feedback vector by querying $Q$ for the $t$-th time \\
    \hline
	\multicolumn{2}{c}{\textbf{Modelling related}} \\
	\hline
	$s$ & relevant statistic for the decision making process \\
    $s_{i|Q}(t)$ & statistics for arm $i \in Q$ derived by the observed feedback at the $t$-th usage of query set $Q$\\
    $\mathbf{s}_Q(t)$ & vector of statistics for all arms in the query set $Q$ after its $t$-th usage: $(s_{i|Q})_{i\in Q}(t)$\\
    $S_{i|Q}$ & limit of the statistics for arm $i$ in query set $Q$: %
    $\lim_{t\rightarrow \infty} s_{i|Q}(t)$\\
    $i^{\ast}$ & best arm or generalized Condorcet winner: $\forall Q \in \mathcal{Q}_{\leq k}$ with $i^{\ast} \in Q$ it holds that $S_{i^{\ast}|Q} > S_{j|Q} $ for any $j\in Q\backslash \{i^{\ast}\} $ \\
    $s_{i}^{\mathcal B}(t)$ & Borda score of arm $i$ at time $t$:  $\nicefrac{\sum\nolimits_{Q \in \mathcal{Q}_{=k}(i)} s_{i|Q}(t)}{|\mathcal{Q}_{=k}(i)|}$ \\
    $S_{i}^{\mathcal B}$ & limit Borda score of arm $i$: $ \lim_{t \to \infty} s_{i}^{\mathcal B}(t)$ \\
    $i_{\mathcal B}^{\ast}  $ & generalized Borda winner: $i_{\mathcal B}^{\ast}   \in \arg \max_{i \in \mathcal{A}} S_{i}^{\mathcal B}$\\
    $n_Q(t)$ & number of times query set $Q$ was used until time $t$\\
    $\gamma_{i|Q}(t)$ & point-wise smallest non-increasing function bounding the difference $|s_{i|Q}(t) - S_{i|Q}|$ (rate of convergence)\\
    $\bar{\gamma}_Q(t)$ & maximal $\gamma_{i|Q}(t)$ over all $i\in Q$\\
    $\bar{\gamma}(t)$ & maximal $\bar{\gamma_Q(t)}$ over all $Q \in \mathcal{Q}_{\leq k }$\\
    $\gamma_{i|Q}^{-1}(\alpha)$ & quasi-inverse of $\gamma_{i|Q}:$ $\min\{ t\in \mathbb{N} ~|~ \gamma_{i|Q}(t) \leq \alpha \}$\\
    $\bar{\gamma}_Q^{-1}(t)$ & minimal $\gamma_{i|Q}(t)$ over all $i\in Q$\\
    $\bar{\gamma}^{-1}(t)$ & minimal $\gamma_Q(t)$ over all $Q \in \mathcal{Q}_{\leq k}$ \\
    $\hat{\gamma}_{i}(t)$ & rate of convergence of the Borda score for arm $i$: $ \frac{1}{|\mathcal{Q}_{=k}(i)|} \sum_{Q \in \mathcal{Q}_{=k}(i)} \gamma_{i|Q}(t)$ \\
    $\hat{\gamma}_{i,j}^{\max}(t)$ & $\max\{\hat{\gamma}_{i}(t), \hat{\gamma}_{j}(t)\}$. \\
    $\Delta_{i|Q}$ & gap of the limit statistic of arm $i \in Q$ to the limit statistic of the generalized Condorcet winner: $|S_{i^{\ast}|Q} - S_{i|Q}|$ for any $Q \in \mathcal{Q}_{\leq k}(i) \cap \mathcal{Q}_{\leq k}(i^{\ast})$ \\
    $S_{(l)|Q}$,  $\Delta_{(l)|Q}$   & $l$-th order statistic of $\{S_{i|Q}\}_{i\in Q}$ for $l\in\{1,2,\ldots,|Q|\}$ and its  gap $\Delta_{(l)|Q} = S_{i^{\ast}|Q} - S_{(l)|Q}$ \\
    \hline
	\multicolumn{2}{c}{\textbf{Algorithm related}} \\
	\hline
    $f$ & function from $[k]$ to $[k]$ specifying the nature of the arm elimination strategy \\
    $R, R^{\mathbb{A}}$ & number of rounds of the learning algorithm ($\mathbb{A}$) \\
    $P_r,P_r^{\mathbb{A}}$ & number of partitions of the learning algorithm  ($\mathbb{A}$) in round $r$ \\
    $\mathbb{A}_{r,j}$ & $j$-th partition in round $r$ \\
    $\mathbb{A}_{r}(i^*)$ & the partition in round $r$ containing $i^*$ (emptyset otherwise)\\
    $b_r$ & budget used in round $r$ for a partition \\
    $z_{\mathbb{A}}$ & sufficient budget for learning algorithm $\mathbb{A}$ to return  $i^{\ast}$ (or $i_{\mathcal B}^{\ast}  $ if $\mathbb{A}$ is \textsc{RoundRobin}) \\
    \textsc{RoundRobin} & the na\"ive algorithm introduced in Section \ref{sec:round_robin}  \\
    \textsc{CSE} & the generic \emph{combinatorial successive elimination} algorithm (Algorithm \ref{alg:Framework}) \\
    \textsc{CSWS} & the \emph{combinatorial successive winner stays} algorithm resulting by using $f(x)=1$ in \textsc{CSE} \\
    \textsc{CSR} & the \emph{combinatorial successive rejects}  algorithm resulting by using $f(x)=x-1$ in \textsc{CSE}\\
    \textsc{CSH} & the \emph{combinatorial successive halving} algorithm resulting by using $f(x)=\lceil x/2 \rceil $ in \textsc{CSE} \\
    \textsc{SH} & the \emph{successive halving} algorithm for pure exploration settings in standard multi-armed bandits (cf.\ \cite{Jamieson2016}) \\
    GBW & Generalized Borda winner \\
    GCW & Generalized Condorcet winner \\
\end{tabularx}
\normalsize
\clearpage

\section{Proofs for Section~\ref{sec_main_LBs}} \label{sec:proofs_lower_bounds}
In this section, we prove the general lower bounds on the necessary budget for identifying the generalized Condorcet winner (GCW), the generalized Borda winner (GBW) or the generalized Copeland winner (GCopeW). For this purpose, let us first fix some further notation.
If $\Alg$ is a possibly probabilistic algorithm and $\mathbf{s}$ is fixed, we write $\Alg(\mathbf{s})$ for the output of $\Alg$ executed on the instance $\mathbf{s}$. 
We restrict ourselves only to algorithms whose output is solely determined by the sequence of observations it has received as well as the corresponding statistics.
Moreover, for $Q\in \mathcal{Q}_{\leq k}$, we write $B_{Q}(\Alg,\mathbf{s}) \in \N \cup \{\infty\}$ for the number of times $\Alg$ queries $Q$ when started on instance $\mathbf{s}$. Note that $\Alg(\mathbf{s})$ as well as $B_{Q}(\Alg,\mathbf{s})$ and $B(\Alg,\mathbf{s}) = \sum_{Q\in \mathcal{Q}_{\leq k}} B_{Q}(\Alg,\mathbf{s})$ are random variables, because they depend on the innate randomness of $\Alg$.\\
Given $\mathbf{s}$, let us write $\mathrm{GCW}(\mathbf{s})$, $\mathrm{GBW}(\mathbf{s})$ and  $\mathrm{GCopeW}(\mathbf{s})$ for the set of all GCWs, GBWs and GCopeWs of $\mathbf{s}$, respectively. In case $|\mathrm{GCW}(\mathbf{s})| = 1$, $|\mathrm{GBW}(\mathbf{s})| = 1$ resp. $|\mathrm{GCopeW}(\mathbf{s})|=1$,  with a slight abuse of notation, we may denote by $\mathrm{GCW}(\mathbf{s})$, $\mathrm{GBW}(\mathbf{s})$ resp. $\mathrm{GCopeW}(\mathbf{s})$ simply the only GCW, GBW resp. GCopeW of $\mathbf{s}$.
Recall that the  GCW, the GBWs and the GCopeWs of $\mathbf{s}$ only depend on the limits $\mathbf{S} = (S_{i|Q})_{Q\in \mathcal{Q}_{\leq k},i\in Q}$ with $S_{i|Q} = \lim_{t\ra \infty} s_{i|Q}(t)$.\\
\begin{Def}
    Let $\Alg$ be a (possibly probabilistic) sequential algorithm.
    \begin{itemize}
        \item[(i)]
        $\Alg$ solves  $\mathcal{P}_{\mathrm{GCW}}(\mathbf{S},\bsym{\gamma})$ if $\P(\Alg(\mathbf{s}) \in  \mathrm{GCW}(\mathbf{s})) =1$ for any $\mathbf{s}$ in $\mathfrak{S}(\mathbf{S},\bsym{\gamma})$.
        \item[(ii)]
         $\Alg$ solves $\mathcal{P}_{\mathrm{GBW}}(\mathbf{S},\bsym{\gamma})$ if $\P(\Alg(\mathbf{s}) \in \mathrm{GBW}(\mathbf{s})) =1$ for any $\mathbf{s}$ in $\mathfrak{S}(\mathbf{S},\bsym{\gamma})$.
         \item[(iii)]
         $\Alg$ solves $\mathcal{P}_{\mathrm{GCopeW}}(\mathbf{S},\bsym{\gamma})$ if $\P(\Alg(\mathbf{s}) \in \mathrm{GCopeW}(\mathbf{s})) =1$ for any $\mathbf{s}$ in $\mathfrak{S}(\mathbf{S},\bsym{\gamma})$.
    \end{itemize}
\end{Def}

\subsection{Proof of Theorem~\ref{Thm_Main_LB_GCW} (i): Lower Bound for GCW Identification}\label{sec_proof_Thm_Main_LB_GCW}

The proof of (i) in Theorem \ref{Thm_Main_LB_GCW} is prepared with the next lemma.

\begin{Le}\label{Le_help}
		Let $\Alg$ be a deterministic solution to $\mathcal{P}_{\mathrm{GCW}}(\mathbf{S},\bsym{\gamma})$ and $\mathbf{s},\mathbf{s'} \in \mathfrak{S}(\mathbf{S},\bsym{\gamma})$.
		\begin{itemize}
			\item[(i)]
			If $\Alg(\mathbf{s}) \not= \Alg(\mathbf{s'})$, then 
			\begin{equation*}
			\exists Q \in \mathcal{Q}_{\leq k} ,  i\in Q, \, t\in \{1,\dots, \min\{ B_{Q}(\Alg,\mathbf{s}), B_{Q}(\Alg,\mathbf{s'}) \}\} \, : \, s_{i|Q}(t) \not= s'_{i|Q}(t).
			\end{equation*}
			\item[(ii)]
			If $\mathbf{s}$ and $\mathbf{s'}$ coincide on $\{t<B'\}$ and on $\widetilde{\mathcal{Q}} \subseteq \mathcal{Q}_{\leq k}$ in the sense that 
			\begin{equation}\label{eq_a} \forall Q\in \mathcal{Q}_{\leq k}, \forall i\in Q, \forall t<B' \, :  \, s_{i|Q}(t) =
			 s'_{i|Q}(t)
			\end{equation}
			and 
			\begin{equation}\label{eq_b}
			\forall Q\in \widetilde{\mathcal{Q}}, \forall i\in Q, \forall t\in \N \, :  \, s_{i|Q}(t) = s'_{i|Q}(t),
			\end{equation}
			then $\Alg(\mathbf{s}) \not= \Alg(\mathbf{s'})$ implies 
			\begin{equation*}
	\exists Q\in \mathcal{Q}_{\leq k} \setminus \widetilde{\mathcal{Q}} \, : \, \min \{ B_{Q}(\Alg,\mathbf{s}), B_{Q}(\Alg,\mathbf{s'}) \} \geq B'.
\end{equation*}
		\end{itemize}
\end{Le}
\begin{proof}
	\begin{itemize}
	\item[(i)]
	To prove the contraposition, suppose that 
	\begin{equation}\label{eq_proof_Le_help}
		\forall  Q \in \mathcal{Q}_{\leq k} ,  i\in Q, \, t\in \{1,\dots, \min\{ B_{Q}(\Alg,\mathbf{s}), B_{Q}(\Alg,\mathbf{s'}) \}\} \, : \, s_{i|Q}(t) = s'_{i|Q}(t)
	\end{equation}
	holds.\\
	\noindent \textbf{Claim 1:} $B_{Q}(\Alg,\mathbf{s}) = B_{Q}(\Alg,\mathbf{s'})$ for any $\mathcal{Q} \in \mathcal{Q}_{\leq k}$.\\
	\noindent \textbf{Proof:} Assume this was not the case. Let $Q\in \mathcal{Q}_{\leq k}$ be the first set, for which $\Alg$ exceeds its budget on  one of $\mathbf{s}$, $\mathbf{s'}$ but does not reach it on the other instance, and suppose w.l.o.g.\ $B_{Q}(\Alg,\mathbf{s}) > B_{Q}(\Alg,\mathbf{s'})$. Since $\Alg$ has observed until this point exactly the same feedback on $\mathbf{s}$ as on $\mathbf{s'}$, this is a contradiction as $\Alg$ is deterministic.
	\hfill $\blacksquare$\\ 
	\noindent Combining Claim 1 and  \eqref{eq_proof_Le_help} yields that $\Alg$ observes on $\mathbf{s}$ exactly the same feedback as on $\mathbf{s'}$ until its termination. Since $\Alg$ is deterministic, this implies $\Alg(\mathbf{s}) = \Alg(\mathbf{s'})$. 
	\item[(ii)]
	If $\Alg(\mathbf{s}) \not= \Alg(\mathbf{s'})$, then (i) together with \eqref{eq_b} yields
	\begin{equation*}
		\exists Q \in \mathcal{Q}_{\leq k} \setminus \widetilde{\mathcal{Q}}, i\in Q, t\leq \min \{B_{Q}(\Alg,\mathbf{s}),B_{Q}(\Alg,\mathbf{s'})\} \, : \, 
	s_{i|Q}(t) \not= s'_{i|Q}(t),
	\end{equation*}
	and thus \eqref{eq_a} implies 
	\begin{equation*}
		\exists Q\in \mathcal{Q}_{\leq k} \setminus \widetilde{\mathcal{Q}} \, : \, \min \{ B_{Q}(\Alg,\mathbf{s}), B_{Q}(\Alg,\mathbf{s'}) \} \geq B'.
	\end{equation*}
	\end{itemize}
\end{proof}

Lemma \ref{Le_help} is the main ingredient for the proof of Theorem \ref{Thm_Main_LB_GCW}, as we first analyze the lower bound for deterministic algorithms and then apply Yao's minimax principle \citep{yao1977probabilistic} to infer the lower bound for any randomized algorithm.

\begin{proof}[Proof of Theorem \ref{Thm_Main_LB_GCW} (i)]
    We split the proof into two parts.\\[0.5em]
    \noindent \textbf{Part 1: The statement holds in case $\Alg$ is a deterministic algorithm.}\\
    \noindent 
		Abbreviate $B'\coloneqq \min_{Q\in \mathcal{Q}_{\leq k}} \min_{j\in Q} \gamma_{j|Q}^{-1}\left( \frac{S_{(1)|Q} - S_{(|Q|)|Q}}{2}\right)$. 
		Fix a family $\{\pi_{Q}\}_{Q\in \mathcal{Q}_{\leq k}}$ of permutations $\pi_{Q} : Q \mapsto Q$ such that $S_{\pi_{Q}(1)|Q} = S_{(1)|Q}$ holds for any $Q\in \mathcal{Q}_{\leq k}(1)$, and define $\mathbf{s} = (s_{i|Q}(t))_{Q\in \mathcal{Q}_{\leq k},i\in Q, t\in \N}$ via 
		\begin{equation*}
			s_{i|Q}(t) \coloneqq \begin{cases} \frac{S_{(1)|Q} + S_{(|Q|)|Q}}{2},\quad &\text{if } t<B', \\ S_{\pi_{Q}(i)|Q}, \quad &\text{if } t\geq B'. \end{cases}
		\end{equation*}
	    Regarding our assumption on $\mathbf{S}$, $\mathrm{GCW}(\mathbf{s}) = 1$ holds by construction. For $t<B' \leq  \gamma_{i|Q}^{-1}\left(\frac{S_{(1)|Q}-S_{(|Q|)|Q}}{2}\right)$, which implies $\gamma_{i|Q}(t) \geq \frac{S_{(1)|Q}-S_{(|Q|)|Q}}{2}$, 
	we have due to $S_{(1)|Q} \geq S_{i|Q} \geq S_{(|Q|)|Q}$  the inequality 
	\begin{align*}
		&\left| s_{i|Q}(t) - \lim\nolimits_{t\ra \infty} s_{i|Q}(t) \right| = \left| \frac{S_{(1)|Q}+S_{(|Q|)|Q}}{2} - S_{i|Q} \right| \\
		&\leq 
		\max \left\{ S_{(1)|Q} - \frac{S_{(1)|Q}+S_{(|Q|)|Q}}{2} ,  \frac{S_{(1)|Q}+S_{(|Q|)|Q}}{2} - S_{(|Q|)|Q}\right\}   \\
		&= \frac{S_{(1)|Q}-S_{(|Q|)|Q}}{2} \\
		&\leq \gamma_{i|Q}(t)
	\end{align*}
	for any $i\in Q$.
	This shows $\mathbf{s} \in \mathfrak{S}(\mathbf{S},\bsym{\gamma})$.\par 
		For any $l\in \{2,\dots,n\}$ define an instance $\mathbf{s}^{l} = (s^{l}_{i|Q}(t))_{Q\in \mathcal{Q}_{\leq k},i\in Q, t\in \N}$ such that $s^{l}_{\cdot|Q}(\cdot) = s_{\cdot|Q}(\cdot)$ for any $Q\in \mathcal{Q}_{\leq k}$ with $l\not\in Q$ and 
		\begin{equation*}
			s^{l}_{i|Q}(t) \coloneqq \begin{cases} \frac{S_{(1)|Q}+S_{(|Q|)|Q}}{2}, \quad &\text{ if } t<B', \\
			S_{(1)|Q}, 
			\quad &\text{ if } t\geq B' \text{ and } i=l,\\
			S_{l|Q}, \quad &\text{ if } t\geq B' \text{ and } i=\argmax_{j\in Q} S_{j|Q} \\
			s_{i|Q}(t), \quad &\text{ else},
			\end{cases}
		\end{equation*}
		for all $Q\in \mathcal{Q}_{\leq k}(l)$, $i\in Q$ and $t\in \N$. According to its definition, we have $\mathrm{GCW}(\mathbf{s}^{l}) = l$, and similarly as above one may check $\mathbf{s}^{l} \in \mathfrak{S}(\mathbf{S},\bsym{\gamma})$.\par 
		Since $\Alg$ solves $\mathcal{P}_{\mathrm{GCW}}(\mathbf{S},\bsym{\gamma})$, it satisfies $\Alg(\mathbf{s}) = 1 \not= 2 = \Alg(\mathbf{s}^{2})$. Regarding that  $\mathbf{s}$ and $\mathbf{s}^{2}$ coincide on $\{t< B'\}$ and on $\{Q \in \mathcal{Q}_{\leq k} \, | \,  1\not\in Q \mbox{ or }  2\not\in Q\}$ in the sense of \eqref{eq_a} and \eqref{eq_b}, Lemma \ref{Le_help} (ii) assures the existence of some $Q_{1} \in \mathcal{Q}_{\leq k}$ with $1\in Q_{1}$ and $i_{1} \coloneqq 2 \in Q_{1}$ such that $B_{Q_{1}}(\Alg,\mathbf{s}) \geq \min \{ B_{Q_{1}}(\Alg,\mathbf{s}), B_{Q_{1}}(\Alg,\mathbf{s}^{i_{1}})\} \geq B'$.
		Let $F_{1} \coloneqq [n] \setminus Q_{1}$ and fix an arbitrary $i_{2} \in F_{1}$. 
		Then, $\Alg(\mathbf{s}) = 1 \not= i_{2} = \Alg(\mathbf{s}^{i_{2}})$
		and since $\mathbf{s}$ and $\mathbf{s}^{i_{2}}$ coincide on $\{t<B'\}$ and $\{Q\in \mathcal{Q}_{\leq k} \, | \,  i_{2} \not\in Q\}$, Lemma \ref{Le_help} (ii) yields the existence of some $Q_{2} \in \mathcal{Q}_{\leq k}$ with $i_{2} \in Q_{2}$ such that $B_{Q_{2}}(\Alg,\mathbf{s}) \geq \min \{ B_{Q_{2}}(\Alg,\mathbf{s}), B_{Q_{2}}(\Alg,\mathbf{s}^{i_{2}})\} \geq B'$. From $i_{2} \in F_{1} = [n]\setminus Q_{1}$ and $i_{2} \in Q_{2}$ we infer $Q_{1} \not= Q_{2}$.
		With this, we define $F_{2} \coloneqq F_{1} \setminus Q_{2} = [n] \setminus (Q_{1} \cup Q_{2})$.\par 
		Inductively, whenever $F_{l} \not= \emptyset$, we may select an element $i_{l+1} \in  F_{l}$ and infer from Lemma \ref{Le_help} (ii), due to $\Alg(\mathbf{s}) = 1 \not= i_{l+1} = \Alg(\mathbf{s}^{i_{l+1}})$ and the similarity of $\mathbf{s}$ and $\mathbf{s}^{i_{l+1}}$ on $\{t<B'\}$ and $\{Q\in \mathcal{Q}_{\leq k} | i_{l+1} \not\in Q\}$,  the existence of a set $Q_{l+1} \in \mathcal{Q}_{\leq k}$ with $i_{l+1} \in Q_{l+1}$ such that $B_{Q_{l+1}}(\Alg,\mathbf{s}) \geq B'$, and define $F_{l+1} \coloneqq F_{l} \setminus Q_{l+1}$. Then, $i_{l+1} \in F_{l} = [n] \setminus (Q_{1} \cup \dots \cup Q_{l})$ and $i_{l+1} \in Q_{l+1}$ assure $Q_{l+1} \not\in \{Q_{1},\dots,Q_{l}\}$. This procedure terminates at the smallest $l'$ such that $F_{l'} = \emptyset$, and $Q_{1},\dots,Q_{l'}$ are distinct.  Regarding that $|F_{l+1}|-|F_{l}| \leq |Q_{l}| \leq k$ for all $l\in \{1,\dots,l'-1\}$, we have $l' \geq \lceil \frac{n}{k} \rceil $. Consequently,
		\begin{align*}
			B(\Alg,\mathbf{s}) \geq \sum\limits_{l=1}^{l'} B_{Q_{l}}(\Alg,\mathbf{s}) \geq \left\lceil \frac{n}{k} \right\rceil B' 
		\end{align*}
		holds, which shows the claim for deterministic algorithms with regard to the definition of $B'$.\\[0.5em]
        \noindent \textbf{Part 2: The statement holds for arbitrary $\Alg$.}\\
        \noindent 
		Let $\mathfrak{A}$ be the set of all deterministic algorithms\footnote{At any time $t\in \N$, a deterministic algorithm $\Alg \in \mathfrak{A}$ may either make a query $Q\in \mathcal{Q}_{\leq k}$ or terminate with a decision $\mathrm{X} \in \{1,\dots,n\}$. Thus, $\mathfrak{A}$ is a countable set.} and $\mathbf{s}$ be the instance from the first part. Write $\delta_{\mathbf{s}}$ for the  probability distribution on $\{\mathbf{s}\}$, which assigns $\mathbf{s}$ probability one, i.e., the Dirac measure on $\mathbf{s}$. Note that for any randomized algorithm $\Alg$ there exists a probability distribution $P$ %
		on $\mathfrak{A}$ such that $\Alg \sim P$. 
		By applying Yao's minimax principle \citep{yao1977probabilistic} and using part one we conclude 
		\begin{align*}
			\E [B(\Alg,\mathbf{s})] = \E_{\Alg' \sim P} [B(\Alg',\mathbf{s})] &\geq \inf\nolimits_{\Alg \in \mathfrak{A}} \E_{\mathbf{s'} \sim \delta_{\mathbf{s}}} [B(\Alg,\mathbf{s'})] \\
			&= \inf\nolimits_{\Alg \in \mathfrak{A}} B(\Alg,\mathbf{s}) \geq \left\lceil \frac{n}{k} \right\rceil B',
		\end{align*}
		where $B'$ is as in part one.
\end{proof}

\begin{Rem}\label{Rem_Main_LB_GCW}
\begin{itemize}
    \item[(i)]
The above proof reveals even a stronger version of Theorem \ref{Thm_Main_LB_GCW} (i). 
Indeed, in the proof we explicitly construct $n$ distinct instances $\mathbf{s}^{1}\coloneqq \mathbf{s},\dots,\mathbf{s}^{n} \in \mathfrak{S}(\mathbf{S},\bsym{\gamma})$ with $\mathrm{GCW}(\mathbf{s}^{l}) = l$ for all $l\in [n]$, and in fact show:
Any (possibly random) algorithm $\Alg$, which is able to correctly identify the best arm for any $\mathbf{s}' \in \{\mathbf{s}_{1},\dots,\mathbf{s}_{n}\}$ (i.e., $\Alg$ does not necessarily have to solve $\mathcal{P}_{\mathrm{GCW}}(\mathbf{S},\bsym{\gamma})$) fulfills
\begin{equation*}
    \E\left[ B(\Alg,\mathbf{s}) \right] \geq   \left\lceil \frac{n}{k} \right\rceil \min_{Q\in \mathcal{Q}_{\leq k}} \min_{j\in Q} \gamma_{j|Q}^{-1}\left( \frac{S_{(1)|Q} - S_{(|Q|)|Q}}{2}\right).
\end{equation*}
\item[(ii)]
Condition (iii) in the definition of $\mathfrak{S}(\mathbf{S},\bsym{\gamma})$ assures that the term $S_{(1)|Q}$ resp. $S_{(|Q|)|Q}$ in our lower bound from Theorem \ref{Thm_Main_LB_GCW} coincides with $S'_{(1)|Q}$ resp. $S'_{(|Q|)|Q}$, when $S'_{i|Q} \coloneqq \lim_{t\ra \infty} s_{i|Q}(t)$ for $\mathbf{s} \in \mathfrak{S}(\mathbf{S},\bsym{\gamma})$.
\end{itemize}
\end{Rem}

\subsection{Proof of Theorem~\ref{Thm_Main_LB_GCW} (ii): Lower Bound for GBW Identification}\label{sec_proof_Thm_Main_LB_BW}

Recall that $\mathrm{GBW}(\mathbf{s})$ is the set of elements $i\in [n]$, 
for which the limits $S_{i|Q} =   \lim_{t\ra \infty} s_{i|Q}(t)$ have the highest Borda score 
\begin{equation*}
	S_{i}^{\mathcal B} = \frac{\sum_{Q\in \mathcal{Q}_{=k}(i)} S_{i|Q}}{|\mathcal{Q}_{=k}(i)|} = \frac{\sum_{Q\in \mathcal{Q}_{=k}(i)} S_{i|Q}}{\binom{n-1}{k-1}}.
\end{equation*}
We call $\mathbf{S} = (S_{i|Q})_{Q\in \mathcal{Q}_{\leq k},i\in Q}$ \textbf{homogeneous} if $(S_{(1)|Q},\dots,S_{(|Q|)|Q})$ does not depend on $Q$. 
Thus, if $\mathbf{S}$ is homogeneous, we may simply write $S_{(l)}$ for $S_{(l)|Q}$ for any $Q\in \mathcal{Q}_{=k}$.

The next two lemmata serves as a preparation for the proof of (ii) and (iii) in Theorem~\ref{Thm_Main_LB_GCW}.
\begin{Le}\label{Le_LB_Borda_new}
	For any $\mathcal{W}\subseteq \mathcal{Q}_{=k}$ we have $\sum\nolimits_{j=1}^{n} |\mathcal{Q}_{=k}(j) \cap \mathcal{W}| = k|\mathcal{W}|$.
\end{Le}
\begin{proof}[Proof of Lemma~\ref{Le_LB_Borda_new}]
	Let $\mathcal{W}\subseteq \mathcal{Q}_{=k}$ be fixed. For any $Q=\{i_{1},\dots,i_{k}\} \in \mathcal{Q}_{=k} \cap \mathcal{W}$ we have that $Q \in \mathcal{Q}_{=k}(i_{l}) \cap \mathcal{W}$ for any $l\in [k]$, whereas $Q \not\in \mathcal{Q}_{=k}(j) \cap \mathcal{W}$ for any $j\in [n] \setminus \{i_{1},\dots,i_{k}\}$. Hence,
	\begin{align*}
		\sum\nolimits_{j=1}^{n} |\mathcal{Q}_{=k}(j) \cap \mathcal{W}| = k \left| \bigcup\nolimits_{j=1}^{n} (\mathcal{Q}_{=k}(j) \cap \mathcal{W} ) \right| = k \left| \left( \bigcup\nolimits_{j=1}^{n} \mathcal{Q}_{=k}(j)\right) \cap \mathcal{W}  \right| = k|\mathcal{W}|.
	\end{align*}
\end{proof}

\begin{Le} \label{Le_LB_Borda}
    For any $\mathcal{W}'\subseteq \mathcal{Q}_{=k}$ and $\mathcal{W} \coloneqq \mathcal{Q}_{=k} \setminus \mathcal{W}'$ with $|\mathcal{W}'| <\frac{(1-1/n) k}{k+n-2} \binom{n}{k}$ there exists $j\in [n]\setminus \{1\}$ with $|\mathcal{Q}_{=k}(j) \cap \mathcal{W}| > |\mathcal{Q}_{=k}(1) \cap \mathcal{W}'|$.
\end{Le}
\begin{proof}[Proof of Lemma~\ref{Le_LB_Borda}]
    For $j\in [n] \setminus \{1\}$ abbreviate $a_{j} \coloneqq |\mathcal{Q}_{=k}(j) \cap \mathcal{W}| - |\mathcal{Q}_{=k}(1) \cap \mathcal{W}'|$. Due to
\begin{align*}
	|\mathcal{W}| 
	&=\binom{n}{k} - |\mathcal{W}'| \\
	&>	\binom{n}{k} - \left(1-\frac{1}{n}\right)
	\frac{k}{k+n-2} \binom{n}{k} \\
	&=\binom{n}{k} -\frac{k\binom{n}{k} - \frac{k}{n} \binom{n}{k}}{k+n-2}\\
	&= \frac{1}{k+n-2} \left( \binom{n-1}{k-1} + (n-2) \binom{n}{k}\right)
\end{align*}
we have 
\begin{equation*}
	k|\mathcal{W}| - \binom{n-1}{k-1} - (n-2) \left( \binom{n}{k} - |\mathcal{W}| \right) > 0.
\end{equation*}
By using Lemma \ref{Le_LB_Borda_new} and the fact that $(\mathcal{W} \cap \mathcal{Q}_{=k}(1)) \cup (\mathcal{W}' \cap \mathcal{Q}_{=k}(1)) = \mathcal{Q}_{=k}(1)$ is a disjoint union, we obtain
\begin{align*}
	&\sum\nolimits_{j\not= 1} a_{j} = \sum\nolimits_{j\not= 1} |\mathcal{Q}_{=k}(j) \cap \mathcal{W}| - (n-1)| \mathcal{Q}_{=k}(1) \cap \mathcal{W}'|\\
	&= \sum\nolimits_{j\in [n]} |\mathcal{Q}_{=k}(j) \cap \mathcal{W}| - |\mathcal{Q}_{=k}(1) \cap \mathcal{W}| - |\mathcal{Q}_{=k}(1) \cap \mathcal{W}'| - (n-2) |\mathcal{Q}_{=k}(1) \cap \mathcal{W}'| \\
	&= k|\mathcal{W}| - |\mathcal{Q}_{=k}(1)| - (n-2) |\mathcal{Q}_{=k}(1) \cap \mathcal{W}'|\\
	&\geq k|\mathcal{W}| - \binom{n-1}{k-1} - (n-2) |\mathcal{W}'|\\
	&= k|\mathcal{W}| - \binom{n-1}{k-1} - (n-2) \left( \binom{n}{k} - |\mathcal{W}| \right)
	> 0.
\end{align*}
Consequently, there exists $j\in [n]\setminus \{1\}$ with $a_{j} >0$.
\end{proof}

\begin{proof}[Proof of Theorem~\ref{Thm_Main_LB_GCW} (ii)]

Similarly as in the proof of Theorem~\ref{Thm_Main_LB_GCW} (i), we proceed in two steps.\\[0.5em]
\noindent \textbf{Part 1: The statement holds in case $\Alg$ is deterministic.}\\
\noindent 
	Abbreviate $B' \coloneqq \ol{\gamma}^{-1}\left( \frac{S_{(1)}-S_{(|Q|)}}{2}\right)$ and fix a family of permutations $(\pi_{Q})_{Q\in \mathcal{Q}_{\leq k}}$ with $S_{(1)|Q} = S_{\pi_{Q}(1)|Q}$ for all $Q\in \mathcal{Q}_{\leq k}(1)$. Exactly as in the proof of Theorem \ref{Thm_Main_LB_GCW} (i), we define $\mathbf{s} = (s_{i|Q}(t))_{Q\in \mathcal{Q}_{\leq k},i\in Q,t\in \N}$ via 
	\begin{equation*}
		s_{i|Q}(t) \coloneqq \begin{cases} 
		\frac{S_{(1)|Q} + S_{(|Q|)|Q}}{2}, \quad &\text{ if } t<B' \\  		
		S_{\pi_{Q}(i)|Q},\quad &\text{ if } t\geq B'.
		\end{cases}
	\end{equation*}
	In the proof of Theorem \ref{Thm_Main_LB_GCW} (i) we have already verified $\mathbf{s} \in \mathfrak{S}(\mathbf{S},\bsym{\gamma})$. 
	For any $j\in \{2,\dots,m\}$ and $Q\in \mathcal{Q}_{=k}(1) \cap \mathcal{Q}_{=k}(j)$ we have $S_{1|Q} > S_{j|Q}$,  and using that $|\mathcal{Q}_{=k}(i') \setminus \mathcal{Q}_{=k}(j')|$ is the same for every distinct $i',j'\in [n]$ we thus have 
	\begin{align*}
		\sum\nolimits_{Q \in \mathcal{Q}_{=k}(1)} S_{1|Q} &=  \sum\nolimits_{Q \in \mathcal{Q}_{=k}(1) \cap \mathcal{Q}_{=k}(j)} S_{1|Q} + S_{(1)} \cdot |\mathcal{Q}_{=k}(1) \setminus \mathcal{Q}_{=k}(j)| \\
		&> \sum\nolimits_{Q \in \mathcal{Q}_{=k}(j) \cap \mathcal{Q}_{=k}(1)} S_{j|Q} + S_{(1)} \cdot |\mathcal{Q}_{=k}(j) \setminus \mathcal{Q}_{=k}(1)|\\
		&> \sum\nolimits_{Q \in \mathcal{Q}_{=k}(j)} S_{j|Q}.
	\end{align*}
	As $|\mathcal{Q}_{=k}(1)| = |\mathcal{Q}_{=k}(j)|$, this shows $\mathrm{GBW}(\mathbf{s}) = 1$.\par 
	In the following, we will show that 
	\begin{equation*}
		\mathcal{W}' \coloneqq \{Q  \in \mathcal{Q}_{=k} \, : \, \Alg \text{ started on } \mathbf{s} \text{ queries } Q \text{ at least } B' \text{ times}\}
	\end{equation*}
	contains at least $\frac{(1-1/n) k}{k+n-2} \binom{n}{k}$ elements. For this, let us assume on the contrary $|\mathcal{W}'| <\frac{(1-1/n) k}{k+n-2} \binom{n}{k}$ and write $\mathcal{W} \coloneqq \mathcal{Q}_{=k} \setminus \mathcal{W}'$. 
Lemma~\ref{Le_LB_Borda} allows us to fix a $j\in [n]\setminus \{1\}$ with $|\mathcal{Q}_{=k}(j) \cap \mathcal{W}| > |\mathcal{Q}_{=k}(1) \cap \mathcal{W}'|$. Now, define $\mathbf{s}' = (s'_{i|Q}(t))_{Q\in \mathcal{Q}_{\leq k},i\in Q,t\in \N}$ via 
$s'_{\cdot|Q}(\cdot) = s_{\cdot|Q}(\cdot)$ for any $Q\in (\mathcal{Q}_{\leq k} \setminus (\mathcal{Q}_{=k}(1) \cup \mathcal{Q}_{=k}(j))) \cup \mathcal{W}'$ and\footnote{That is, for constructing $\mathbf{s}'$, we proceed for $Q\in \mathcal{W}$ as follows: If $\{1,j\} \subseteq Q$, we exchange $S_{1|Q}$ with $S_{j|Q}$. If $1\in Q \not\ni j$, we exchange $S_{1|Q}$ with $S_{(Q)|Q}$. And if $j\in Q \not\ni 1$, we exchange $S_{j|Q}$ with $S_{(1)|Q}$.}
\begin{equation*}
	s'_{i|Q}(t) \coloneqq \begin{cases} s_{i|Q}(t),\quad &\text{ if } t<B' \text{ or } \{1,j\} \not\subseteq Q,\\
	S_{(1)}, \quad &\text{ if } i=j\in Q \text{ and }  t\geq B',\\
	S_{(|Q|)}, \quad &\text{ if } i=1\in Q \text{ and }  t\geq B',\\
	S_{1|Q}, \quad &\text{ if } t\geq B',  i=\argmin_{l' \in Q} S_{l'|Q} \text{ and } 1\in Q \not\ni j,\\
	S_{j|Q}, \quad &\text{ if } t\geq B', i=\argmax_{l' \in Q} S_{l'|Q} \text{ and }  j\in Q \not\ni 1,\\
	S_{i|Q}, \quad &\text{ otherwise},\\
	\end{cases}
\end{equation*}
for $Q\in (\mathcal{Q}_{=k}(1) \cup \mathcal{Q}_{=k}(j)) \cap \mathcal{W}$. Similarly as for $\mathbf{s}$, we see $\mathbf{s}'\in \mathfrak{S}(\mathbf{S},\bsym{\gamma})$. The corresponding limit values $S'_{i|Q} = \lim_{t\ra \infty} s'_{i|Q}(t)$ fulfill 
\begin{align*}
	\forall Q\in \mathcal{Q}_{=k}(1) \cap \mathcal{W} : S'_{1|Q} = S_{(|Q|)} \quad \text{ and } \quad 
	\forall Q\in \mathcal{Q}_{=k}(j) \cap \mathcal{W} : 
	S'_{j|Q} = S_{(1)},
\end{align*}
and trivially also $S_{(|Q|)} \leq S'_{i|Q} \leq S_{(1)}$ for any $Q\in \mathcal{Q}_{=k}, i\in Q$. Therefore, by choice of $j$, the corresponding Borda scores $(S')_{i}^{\mathcal B}$ for $\mathbf{s}'$ fulfill 
\begin{align*}
	\binom{n-1}{k-1} (S')_{1}^{\mathcal B} &= \sum\nolimits_{Q\in \mathcal{Q}_{=k}(1)} S'_{1|Q} = \sum\nolimits_{Q\in \mathcal{Q}_{=k}(1)\cap \mathcal{W}'} S_{(1)} + \sum\nolimits_{Q\in \mathcal{Q}_{=k}(1) \cap \mathcal{W}} S_{(|Q|)}\\
	&= | \mathcal{Q}_{=k}(1) \cap \mathcal{W}'|\cdot S_{(1)} + |\mathcal{Q}_{=k}(1) \cap \mathcal{W}| \cdot S_{(|Q|)}\\
	&< |\mathcal{Q}_{=k}(j) \cap \mathcal{W}|\cdot S_{(1)} + |\mathcal{Q}_{=k}(j) \cap \mathcal{W}'|\cdot S_{(|Q|)} \\
	&\leq \sum\nolimits_{Q \in \mathcal{Q}_{=k}(j)} S'_{j|Q} = \binom{n-1}{k-1} (S')_{j}^{\mathcal B}, 
\end{align*} 
where we have used that $|\mathcal{Q}_{=k}(1) \cap \mathcal{W}'| + |\mathcal{Q}_{=k}(1) \cap \mathcal{W}| = |\mathcal{Q}_{=k}(1)| = |\mathcal{Q}_{=k}(j) \cap \mathcal{W}'| + |\mathcal{Q}_{=k}(j) \cap \mathcal{W}|$. This show $1\not\in \mathrm{GBW}(\mathbf{s}')$. But since $s_{\cdot|\cdot}(\cdot) = s'_{\cdot|\cdot}(\cdot)$ holds on $\{t<B'\}$ as well as on $\mathcal{W}'$, $\Alg$ observes for $\mathbf{s}$ until termination exactly the same feedback as for $\mathbf{s}'$. Consequently, it outputs for both instances the same decision. Since $\mathrm{GBW}(\mathbf{s}) = 1 \not\in \mathrm{GBW}(\mathbf{s}')$, it makes on at least one of the instances a mistake, which contradicts the correctness of $\Alg$.\\
Thus, $|\mathcal{W}'| \geq \frac{(1-1/n) k}{k+n-2} \binom{n}{k}$ has to hold and we conclude 
\begin{equation*}
	B(\Alg,\mathbf{s}) \geq \sum\nolimits_{Q\in \mathcal{W}'} B_{Q}(\Alg,\mathbf{s}) \geq |\mathcal{W}'| \cdot B' \geq \left(1-\frac{1}{n}\right) \frac{k}{k+n-2} \binom{n}{k} B'.
\end{equation*}
Since $1-\frac{1}{n} \geq 1/2$ and $k\leq n+2$ hold by assumption, we have in particular
\begin{equation*}
    B(\Alg,\mathbf{s}) \geq \frac{k}{4n} \binom{n}{k} \ol{\gamma}^{-1}\left( \frac{S_{(1)}-S_{(|Q|)}}{2}\right) = \frac{1}{4} \binom{n-1}{k-1} \ol{\gamma}^{-1}\left( \frac{S_{(1)}-S_{(|Q|)}}{2}\right) \in \Omega\left(\binom{n-1}{k-1}\right).
\end{equation*} \\[0.5em]
\noindent \textbf{Part 2: The statement holds for arbitrary $\Alg$.}\\
\noindent 
Similarly as for the proof of (i) in Theorem \ref{Thm_Main_LB_GCW}, the proof follows by means of  Yao's minimax principle. 
\end{proof}

\begin{Rem}\label{Rem_LB_GBW}
	\begin{itemize}
	\item[(i)]
	To compare the bounds for \textsc{RoundRobin} in Theorem \ref{Thm:RRsufficiency} with the lower bound from Theorem \ref{Thm_Main_LB_GCW} (ii) suppose in the following $\mathbf{S}$ to be homogeneous with $S_{(1)} > S_{(2)}$ and let $\bsym{\gamma}$ be  homogeneous in the sense that $\gamma_{i|Q}(t) = \gamma(t)$ for all $Q\in \mathcal{Q}_{=k}, i\in Q, t\in \N$ for some $\gamma:\N \ra [0,\infty)$.
	 Moreover, let $\mathbf{s}$ be the instance from the proof of Theorem~\ref{Thm_Main_LB_GCW} (ii), and denote by $\mathbf{S}$ the family of limits $S_{i|Q} = \lim_{t\ra \infty} s_{i|Q}(t)$, $Q\in \mathcal{Q}_{\leq k},i\in Q$. 
	 Let us write $S_{(1)}^{\mathcal B},\dots,S_{(n)}^{\mathcal B}$ for the order statistics of $\{S_{i}^{\mathcal B}\}_{i\in [n]}$, i.e., $S_{(1)}^{\mathcal B} \geq \dots \geq S_{(n)}^{\mathcal B}$. Then, \textsc{RoundRobin} returns a GBW of $\mathbf{s} \in \mathfrak{S}(\mathbf{S},\bsym{\gamma})$ if it is executed with a budget $B$ at least 
	\begin{equation*}
		z_{\mathrm{RR}} = \binom{n}{k} B_{1} \quad \text{with} \quad B_{1} \coloneqq  \ol{\gamma}^{-1} \left( \frac{S_{(1)}^{\mathcal B} - S_{(2)}^{\mathcal B}}{2}\right).
	\end{equation*}
	In comparison to this, the lower bound just shown reveals that any (possibly deterministic) solution to $\mathcal{P}_{\mathrm{GBW}}(\mathbf{S},\bsym{\gamma})$ fulfills 
	\begin{align*}
		\E[B(\Alg,\mathbf{s})] \geq 
 \left( 1-\frac{1}{n} \right) \frac{k}{k+n-2} \binom{n}{k} B_{2} \quad \text{with} \quad B_{2} \coloneqq  \ol{\gamma}^{-1}\left( \frac{S_{(1)}-S_{(|Q|)}}{2}\right).
	\end{align*}
	Consequently, the optimality-gap between the upper and lower bound is of the order 
	\begin{equation*}
		B_{1}^{-1} B_{2}  \left( 1-\frac{1}{n} \right) \frac{k}{k+n-2}.
	\end{equation*}
	\item[(ii)]
	In the proof of Theorem \ref{Thm_Main_LB_GCW} (ii), where we showed that $|\mathcal{W}'| \geq  \frac{(1-1/n) k}{k+n-2} \binom{n}{k}$ leads to a contradiction, we have constructed an instance $\mathbf{s}' \in \mathfrak{S}(\mathbf{S},\bsym{\gamma})$ with $\mathrm{GBW}(\mathbf{s}) = 1 \not\in \mathrm{GBW}(\mathbf{s}')$ such that $\Alg$ observes on $\mathbf{s}$ the same feedback as on $\mathbf{s}'$. To finish the proof, we have only used that $\Alg$ is correct for $\mathbf{s}$ and for $\mathbf{s}'$, but we did not require correctness of $\Alg$ on any instance $\mathbf{s}''\in \mathfrak{S}(\mathbf{S},\bsym{\gamma}) \setminus \{\mathbf{s},\mathbf{s}'\}$. The construction of $\mathbf{s}'$ therein dependeds on the behaviour of $\Alg$ only by means of the choices of $\mathcal{W}$ and $j$ in the proof, i.e., we have the dependence $\mathbf{s}' = \mathbf{s}'(\mathcal{W},j)$. 
	Recall that for constructing $\mathbf{s}'$ we used that $|\mathcal{W}| = |\mathcal{Q}_{=k}| - |\mathcal{W}'| \geq \binom{n}{k} - \frac{(1-1/n) k}{k+n-2} \binom{n}{k},$ so that for $j\in [n]\setminus \{1\}$, the set  
	\begin{equation*}
		\left\{ \mathbf{s}'(\mathcal{W},j) \, \Big| \, \mathcal{W} \subseteq \mathcal{Q}_{=k} \text{ with } |\mathcal{W}| \geq \binom{n}{k} - \frac{(1-1/n) k}{k+n-2} \binom{n}{k} \text{ and } j\in [n]\setminus \{1\} \right\}
	\end{equation*}
	of possible choices for $\mathbf{s}'$ has at most 
	\begin{equation*}
		N\coloneqq (n-1) \sum\nolimits_{l=\left\lceil \binom{n}{k} - \frac{(1-1/n) k}{k+n-2} \binom{n}{k}\right\rceil}^{\binom{n}{k}} \binom{\binom{n}{k}}{l} 
	\end{equation*}
	elements, say $\mathbf{s}'_{1},\dots,\mathbf{s}'_{N}$. Thus, the formulation of the theorem may be strengthened in the following way:\\
	If $\mathbf{S}$ is homogeneous and $\bsym{\gamma}$ fixed, then there exist $N+1$ instances $\mathbf{s},\mathbf{s}'_{1},\dots,\mathbf{s}'_{N}$ with the following property: Whenever a (possibly probabilistic) sequential testing algorithm $\Alg$ correctly identifies the GBW for any of these $N+1$ instances, then 
	\begin{equation*}
	\E \left[ B(\mathcal{A},\mathbf{s})\right] \geq 
 \left( 1-\frac{1}{n} \right) \frac{k}{k+n-2} \binom{n}{k} \ol{\gamma}^{-1}\left( \frac{S_{(1)}-S_{(|Q|)}}{2}\right).
 \end{equation*}
	\end{itemize}
\end{Rem}

\subsection{Proof of Theorem~\ref{Thm_Main_LB_GCW} (iii): Lower Bound for GCopeW Identification}

\noindent
Recall that $\mathrm{GCopeW}(\mathbf{s})$ is the set of elements $i\in [n]$, 
for which the limits $S_{i|Q} =   \lim_{t\ra \infty} s_{i|Q}(t)$ have the highest Copeland score
\begin{equation*}
	S_{i}^{\mathcal C} = \frac{\sum_{Q\in \mathcal{Q}_{=k}(i)} \id \{S_{i|Q} = S_{(1)|Q}\}}{|\mathcal{Q}_{=k}(i)|} = \frac{\sum_{Q\in \mathcal{Q}_{=k}(i)}\id \{S_{i|Q} = S_{(1)|Q}\}}{\binom{n-1}{k-1}}.
\end{equation*}

\begin{proof}[Proof of Theorem~\ref{Thm_Main_LB_GCW}.(iii)]
    Similarly as in the proofs (i) and (ii) Theorem~\ref{Thm_Main_LB_GCW}, we proceed in two steps.\\[0.5em]
\noindent \textbf{Part 1: The statement holds in case $\Alg$ is deterministic.}\\
\noindent 
Abbreviate $B' \coloneqq \min_{Q \in \mathcal{Q}_{\leq k}} \min_{i\in Q}\ \gamma_{i|Q}^{-1}\left( \frac{S_{(1)|Q}-S_{(|Q|)|Q}}{2}\right)$ and fix a family of permutations $(\pi_{Q})_{Q\in \mathcal{Q}_{\leq k}}$ with $S_{(1)|Q} = S_{\pi_{Q}(1)|Q}$ for all $Q\in \mathcal{Q}_{\leq k}(1)$. Exactly as in the proofs of the lower bounds for GCW and GBW identification, we define $\mathbf{s} = (s_{i|Q}(t))_{Q\in \mathcal{Q}_{\leq k},i\in Q,t\in \N}$ via 
\begin{equation*}
	s_{i|Q}(t) \coloneqq \begin{cases} 
	\frac{S_{(1)|Q} + S_{(|Q|)|Q}}{2}, \quad &\text{ if } t<B' \\  		
	S_{\pi_{Q}(i)|Q},\quad &\text{ if } t\geq B'.
	\end{cases}
\end{equation*}
In the proof of the lower bound of GCW identification we have already verified $\mathbf{s} \in \mathfrak{S}(\mathbf{S},\bsym{\gamma})$. 
For any $j\in \{2,\dots,m\}$ and $Q\in \mathcal{Q}_{=k}(1) \cap \mathcal{Q}_{=k}(j)$ we have $S_{1|Q} > S_{j|Q}$,  and using that $|\mathcal{Q}_{=k}(i') \setminus \mathcal{Q}_{=k}(j')|$ is the same for every distinct $i',j'\in [n]$ we thus have 
\begin{align*}
	\sum_{Q \in \mathcal{Q}_{=k}(1)} &\id\{S_{1|Q}= S_{(1)|Q}\}\\
	&=  \sum_{Q \in \mathcal{Q}_{=k}(1) \cap \mathcal{Q}_{=k}(j)} \id\{S_{1|Q} = S_{(1)|Q}\} + \sum_{Q \in \mathcal{Q}_{=k}(1) \setminus \mathcal{Q}_{=k}(j)} \id\{S_{1|Q} = S_{(1)|Q}\}\\
	&= \sum_{Q \in \mathcal{Q}_{=k}(1) \cap \mathcal{Q}_{=k}(j)} \id\{S_{1|Q} = S_{(1)|Q}\} + |\mathcal{Q}_{=k}(1) \setminus \mathcal{Q}_{=k}(j)|\\
	&> \sum_{Q \in \mathcal{Q}_{=k}(1) \cap \mathcal{Q}_{=k}(j)} \id\{S_{j|Q} = S_{(1)|Q}\} + |\mathcal{Q}_{=k}(j) \setminus \mathcal{Q}_{=k}(1)|\\
    &\geq \sum_{Q \in \mathcal{Q}_{=k}(1) \cap \mathcal{Q}_{=k}(j)} \id\{S_{j|Q} = S_{(1)|Q}\} + \sum_{Q \in \mathcal{Q}_{=k}(j) \setminus \mathcal{Q}_{=k}(1)} \id\{S_{j|Q} = S_{(1)|Q}\}\\
    &= \sum_{Q \in \mathcal{Q}_{=k}(j)} \id\{S_{j|Q}= S_{(1)|Q}\}.
\end{align*}
As $|\mathcal{Q}_{=k}(1)| = |\mathcal{Q}_{=k}(j)|$, this shows $\mathrm{GCopeW}(\mathbf{s}) = 1$.\par 
\noindent
Similarly as in the proof of (ii), we will show indirectly that 
\begin{equation*}
	\mathcal{W}' \coloneqq \{Q  \in \mathcal{Q}_{=k} \, : \, \Alg \text{ started on } \mathbf{s} \text{ queries } Q \text{ at least } B' \text{ times}\}
\end{equation*}
contains at least $\frac{(1-1/n) k}{k+n-2} \binom{n}{k}$ elements. For this purpose, let us assume on the contrary $|\mathcal{W}'| <\frac{(1-1/n) k}{k+n-2} \binom{n}{k}$ and write $\mathcal{W} \coloneqq \mathcal{Q}_{=k} \setminus \mathcal{W}'$.
Lemma~\ref{Le_LB_Borda} allows us to fix a $j\in [n]\setminus \{1\}$ with $|\mathcal{Q}_{=k}(j) \cap \mathcal{W}| > |\mathcal{Q}_{=k}(1) \cap \mathcal{W}'|$. Now, define $\mathbf{s}' = (s'_{i|Q}(t))_{Q\in \mathcal{Q}_{\leq k},i\in Q,t\in \N}$ analogously as in the proof of (ii), i.e.,  via 
$s'_{\cdot|Q}(\cdot) = s_{\cdot|Q}(\cdot)$ for any $Q\in (\mathcal{Q}_{\leq k} \setminus (\mathcal{Q}_{=k}(1) \cup \mathcal{Q}_{=k}(j))) \cup \mathcal{W}'$ and
\begin{equation*}
	s'_{i|Q}(t) \coloneqq \begin{cases} s_{i|Q}(t),\quad &\text{ if } t<B' \text{ or } \{1,j\} \not\subseteq Q,\\
	S_{(1)|Q}, \quad &\text{ if } i=j\in Q \text{ and }  t\geq B',\\
	S_{(|Q|)|Q}, \quad &\text{ if } i=1\in Q \text{ and }  t\geq B',\\
	S_{1|Q}, \quad &\text{ if } t\geq B',  i=\argmin_{l' \in Q} S_{l'|Q} \text{ and } 1\in Q \not\ni j,\\
	S_{j|Q}, \quad &\text{ if } t\geq B', i=\argmax_{l' \in Q} S_{l'|Q} \text{ and }  j\in Q \not\ni 1,\\
	S_{i|Q}, \quad &\text{ otherwise},\\
	\end{cases}
\end{equation*}
for $Q\in (\mathcal{Q}_{=k}(1) \cup \mathcal{Q}_{=k}(j)) \cap \mathcal{W}$. Similarly as for $\mathbf{s}$, we see $\mathbf{s}'\in \mathfrak{S}(\mathbf{S},\bsym{\gamma})$. The corresponding limit values $S'_{i|Q} = \lim_{t\ra \infty} s'_{i|Q}(t)$ fulfill 
\begin{align*}
	\forall Q\in \mathcal{Q}_{=k}(1) \cap \mathcal{W} : S'_{1|Q} = S_{(|Q|)|Q} \quad \text{ and } \quad 
	\forall Q\in \mathcal{Q}_{=k}(j) \cap \mathcal{W} : 
	S'_{j|Q} = S_{(1)|Q},
\end{align*}
and trivially also $S_{(|Q|)|Q} \leq S'_{i|Q} \leq S_{(1)|Q}$ for any $Q\in \mathcal{Q}_{=k}, i\in Q$. Therefore, by choice of $j$, the corresponding Copeland scores $(S')_{i}^{\mathcal C}$ for $\mathbf{s}'$ fulfill 
\begin{align*}
	\binom{n-1}{k-1} (S')_{1}^{\mathcal C}
	&= \sum\nolimits_{Q\in \mathcal{Q}_{=k}(1)} \id\{S'_{1|Q} = S'_{(1)|Q} \} \\
	&= \sum_{Q\in \mathcal{Q}_{=k}(1)\cap \mathcal{W}'} \id\{S'_{1|Q} = S'_{(1)|Q} \} + \sum_{Q\in \mathcal{Q}_{=k}(1) \cap \mathcal{W}} \id\{ S'_{1|Q} = S'_{(1)|Q}\}\\
	&= \sum_{Q\in \mathcal{Q}_{=k}(1)\cap \mathcal{W}'} \id\{S_{1|Q} = S_{(1)|Q} \} + \sum_{Q\in \mathcal{Q}_{=k}(1) \cap \mathcal{W}} \id\{ S_{(|Q|)|Q} = S_{(1)|Q}\}\\
    &= |\mathcal{Q}_{=k}(1)\cap \mathcal{W}'| \\
    &< |\mathcal{Q}_{=k}(j)\cap \mathcal{W}|\\
    &= \sum_{\mathcal{Q}_{=k}(j)\cap \mathcal{W}} \id \{S_{(1)|Q} = S_{(1)|Q} \}\\
    &= \sum_{\mathcal{Q}_{=k}(j)\cap \mathcal{W}} \id \{S'_{j|Q} = S'_{(1)|Q} \}\\
    &\leq \sum_{\mathcal{Q}_{=k}(j)\cap \mathcal{W}} \id \{S'_{j|Q} = S'_{(1)|Q} \} + \sum_{\mathcal{Q}_{=k}(j)\cap \mathcal{W}'} \id \{S'_{j|Q} = S'_{(1)|Q} \} \\
    &= \sum_{Q\in \mathcal{Q}_{=k}(j)} \id\{S'_{j|Q} = S'_{(1)|Q} \} \\
    &= \binom{n-1}{k-1} (S')_{j}^{\mathcal C},
\end{align*} 
where we used that $S_{(1)|Q} = S'_{(1)|Q}$. This shows $1\not\in \mathrm{GCopeW}(\mathbf{s}')$. But since $s_{\cdot|\cdot}(\cdot) = s'_{\cdot|\cdot}(\cdot)$ holds on $\{t<B'\}$ as well as on $\mathcal{W}'$, $\Alg$ observes for $\mathbf{s}$ until termination exactly the same feedback as for $\mathbf{s}'$. Consequently, it outputs for both instances the same decision. Since $\mathrm{GCopeW}(\mathbf{s}) = 1 \not\in \mathrm{GCopeW}(\mathbf{s}')$, it makes on at least one of the instances a mistake, which contradicts the correctness of $\Alg$.\\
Thus, $|\mathcal{W}'| \geq \frac{(1-1/n) k}{k+n-2} \binom{n}{k}$ has to hold and we conclude 
\begin{equation*}
	B(\Alg,\mathbf{s}) \geq \sum\nolimits_{Q\in \mathcal{W}'} B_{Q}(\Alg,\mathbf{s}) \geq |\mathcal{W}'| \cdot B' \geq \left(1-\frac{1}{n}\right) \frac{k}{k+n-2} \binom{n}{k} B'.
\end{equation*}
Since $1-\frac{1}{n} \geq 1/2$ and $k\leq n+2$ hold by assumption, we have in particular
\begin{align*}
    B(\Alg,\mathbf{s}) & \geq \frac{k}{4n} \binom{n}{k} \min_{Q \in \mathcal{Q}_{\leq k}} \min_{i\in Q} \gamma_{i|Q}^{-1}\left( \frac{S_{(1)|Q}-S_{(|Q|)|Q}}{2}\right) \\
    & = \frac{1}{4} \binom{n-1}{k-1} \min_{Q \in \mathcal{Q}_{\leq k}} \min_{i\in Q} \gamma_{i|Q}^{-1}\left( \frac{S_{(1)|Q}-S_{(|Q|)|Q}}{2}\right) \in \Omega\left(\binom{n-1}{k-1}\right).
\end{align*} \\[0.5em]
\noindent \textbf{Part 2: The statement holds for arbitrary $\Alg$.}\\
\noindent 
Similarly as for the proofs of the lower bound of (i) and (ii) of this theorem, the proof follows by means of  Yao's minimax principle. 
\end{proof}

\section{Generalized Borda Winner Identification} \label{sec:round_robin}
Let \textsc{RoundRobin} be the algorithm, which enumerates all possible subsets of the fixed subset size $k,$ chooses each subset in a round-robin fashion and returns the arm with the highest empirical Borda score $s_{i}^{\mathcal B}$ after the available budget is exhausted. 
It is a straightforward baseline method, which we analyze theoretically in terms of the sufficient and necessary budget to return a generalized Borda winner (GBW) $i_{\mathcal B}^{\ast}$. 
For this purpose, let $\hat{\gamma}_{i}(t) = \frac{1}{|\mathcal{Q}_{=k}(i)|} \sum_{Q \in \mathcal{Q}_{=k}(i)} \gamma_{i|Q}(t)$ and $\hat{\gamma}_{i,j}^{\max}(t) = \max\{\hat{\gamma}_{i}(t), \hat{\gamma}_{j}(t)\}$.
\begin{Thm} \label{Thm:RRsufficiency}
     \textsc{RoundRobin} returns  $i_{\mathcal B}^{\ast}$ if it  is executed with a budget $B\geq z_{\mathrm{RR}}$, where  $$ z_{\mathrm{RR}} \coloneqq  {n \choose k} \max_{\rho \in \mathcal{A}, \rho \neq i_{\mathcal B}^{\ast}} \left(\hat{\gamma}_{i_{\mathcal B}^{\ast},\rho}^{\max}\right)^{-1} \left(\frac{S_{i_{\mathcal B}^{\ast}}^{\mathcal B} - S_{\rho}^{\mathcal B}}{2}\right) .$$
\end{Thm}
The latter bound is tight in a worst-case scenario, as the following result shows (cf.\ Sec.\ \ref{sec_Proofs_round_robin} for the proofs).
\begin{Thm} \label{Thm:RRnecessity}
    For any asymptotical Borda scores $S_{1}^{\mathcal B},\dots, S_{n}^{\mathcal B}$, there exists a corresponding instance $\mathbf{s}$
	such that if $B < z_{\mathrm{RR}}$
    then \textsc{RoundRobin} will not return $i_{\mathcal B}^{\ast}$.
\end{Thm}
Thus, \textsc{RoundRobin} is already nearly-optimal (up to a factor $\mathcal{O}(n/k)$) with respect to worst-case scenarios due to Theorem \ref{Thm_Main_LB_GCW} (see Rem.~\ref{Rem_LB_GBW} for a more detailed discussion.). 

\section{Proofs of Section \ref{sec:Algos}} \label{proofs}
In this section we provide the detailed proofs of Section \ref{sec:Algos}.
We assume throughout that $\frac{B}{{n\choose k}}$ is a natural number, i.e., the budget is a multiple of $n \choose k.$
\subsection{Proof of Theorems \ref{Thm:RRsufficiency} and \ref{Thm:RRnecessity}} \label{sec_Proofs_round_robin}
\begin{proof}[Proof of Theorem \ref{Thm:RRsufficiency}]
    After relabeling the arms in round $r$ we may assume w.l.o.g. $i_{\mathcal B}^{\ast}   = 1$. We will prove the theorem by contradiction and therefore assume
    \begin{align*}
        & \rho = \argmax_{i \in \mathcal{A}}~ s_{i}^{\mathcal B}\left(\frac{B}{{n\choose k}}\right) \neq 1 \\
        \Rightarrow &~ s_{1}^{\mathcal B}\left(\frac{B}{{n\choose k}}\right) < \max_{j=2,\dots n} s_{j}^{\mathcal B}\left(\frac{B}{{n\choose k}}\right) = s_{\rho}^{\mathcal B}\left(\frac{B}{{n\choose k}}\right)\\
        \Rightarrow &~ S_{1}^{\mathcal B} - S_{\rho}^{\mathcal B} < s_{\rho}^B\left(\frac{B}{{n\choose k}}\right) - S_{\rho}^{\mathcal B} + S_{1}^{\mathcal B} - s_{1}^{\mathcal B}\left(\frac{B}{{n\choose k}}\right)\\
        &~ =\frac{1}{|\mathcal{Q}_{=k}(\rho)|} \sum_{Q \in \mathcal{Q}_{=k}(\rho)} \left( s_{\rho|Q}\left(\frac{B}{{n\choose k}}\right) - S_{\rho|Q} \right) + \frac{1}{|\mathcal{Q}_{=k}(1)|} \sum_{Q \in \mathcal{Q}_{=k}(1)} \left( S_{1|Q} - s_{1|Q}\left(\frac{B}{{n\choose k}}\right) \right)\\
        \Rightarrow &~ S_{1}^{\mathcal B} - S_{\rho}^{\mathcal B} < \hat{\gamma}_{\rho}\left(\frac{B}{{n\choose k}}\right) + \hat{\gamma}_{1}\left(\frac{B}{{n\choose k}}\right)\\
        \Rightarrow &~ S_{1}^{\mathcal B} - S_{\rho}^{\mathcal B} < 2 \cdot \hat{\gamma}_{1,\rho}^{\max}\left(\frac{B}{{n\choose k}}\right),
    \end{align*}
    where $\hat{\gamma}_{i}(t) = \frac{1}{|\mathcal{Q}_{=k}(i)|} \sum_{Q \in \mathcal{Q}_{=k}(i)} \gamma_{i|Q}(t)$ and $\hat{\gamma}_{i,j}^{\max}(t) = \max\{\hat{\gamma}_{i}(t), \hat{\gamma}_{j}(t)\}$. 
    With this, however, we can derive 
    \begin{align*}
        \Rightarrow ~  z_{RR} = \left(\hat{\gamma}_{1,\rho}^{\max}\right)^{-1} \left(\frac{S_{1}^{\mathcal B} - S_{\rho}^{\mathcal B}}{2}\right){n \choose k} \ge B,
    \end{align*}
    which contradicts the assumption we make on the budget $B.$
    Thus, it holds that the returned arm is $\rho=1.$ 
\end{proof}
\begin{proof}[Proof of Theorem \ref{Thm:RRnecessity}]
    Let $\beta(t)$ be an arbitrary, monotonically decreasing function of $t$ with $\lim_{t\rightarrow \infty} \beta(t) = 0$. 
    We define for all $j\in\mathcal{A}$ with $j \neq i_{\mathcal B}^{\ast}  $ the empirical Borda scores to be $s_{j}^{\mathcal B}(t) = S_{j}^{\mathcal B}+ \beta(t)$ and $s_{i_{\mathcal B}^{\ast}  }^{\mathcal B}(t) = S_{i_{\mathcal B}^{\ast}  }^{\mathcal B} - \beta(t),$ where $(S_{i}^{\mathcal B})_{i \in [n]}$ are arbitrary real values such that $S_{i_{\mathcal B}^{\ast}  }^{\mathcal B}$ is the unique maximum for some $i_{\mathcal B}^{\ast}  \in [n].$
    We can again assume after relabeling all arms that w.l.o.g.\ that $i_{\mathcal B}^{\ast}   = 1$ and $\argmax_{j=2,\dots,n} S_{j}^{\mathcal B}  = 2$. 
    Note that $\hat{\gamma}_{i}(t) = \beta(t)$ for all $i\in \mathcal{A}$. 
    In light of these considerations, \textsc{RoundRobin} returns $1$ as the best arm if and only if
    \begin{align*}
         s_{1}^{\mathcal B}\left(\frac{B}{{n\choose k}}\right) > \max_{j=2,\dots,n} s_{j}^{\mathcal B}\left(\frac{B}{{n\choose k}}\right) \quad 
        \Leftrightarrow &~ S_{1}^{\mathcal B} - \hat{\gamma}_{1}\left(\frac{B}{{n\choose k}}\right) > \max_{j=2,\dots,n} S_{j}^{\mathcal B} + \hat{\gamma}_{j}\left(\frac{B}{{n\choose k}}\right) \\
        \Leftrightarrow &~ S_{1}^{\mathcal B} - \hat{\gamma}_{1}\left(\frac{B}{{n\choose k}}\right) > S_{2}^{\mathcal B} + \hat{\gamma}_{2}\left(\frac{B}{{n\choose k}}\right) \\
        \Leftrightarrow &~ \hat{\gamma}_{1}\left(\frac{B}{{n\choose k}}\right) + \hat{\gamma}_{2}\left(\frac{B}{{n\choose k}}\right) < S_{1}^{\mathcal B} - S_{2}^{\mathcal B} \\
        \Leftrightarrow &~ 2\cdot \hat{\gamma}_{1,2}^{\max}\left(\frac{B}{{n\choose k}}\right) < S_{1}^{\mathcal B} - S_{2}^{\mathcal B} \\
        \Leftrightarrow &~ B \ge {n \choose k} \left(\hat{\gamma}_{1,2}^{\max}\right)^{-1} \left(\frac{S_{1}^{\mathcal B} - S_{2}^{\mathcal B}}{2}\right).
    \end{align*}
    Thus, the necessary budget is $z_{RR}$ in this case concluding the claim.
\end{proof}

\subsection{Proofs of Theorem \ref{Thm_generalSufficient} and \ref{Thm:general_necessity}} \label{sec_proofs_general_framework}

\begin{proof}[Proof of Theorem \ref{Thm_generalSufficient}]    
	For the sake of convenience, let us abbreviate $[R] \coloneqq \{1,\dots,R\}$ and $\mathbb{A}_{rj} \coloneqq \mathbb{A}_{r,j}$ in the following.
    By possibly relabeling the arms and query sets queried by the algorithm, we can assume w.l.o.g.  $i^{\ast}=1$ and $\mathbb{A}_{r}(1) = \mathbb{A}_{r1}$ for all $r\in [R]$ in the following. In particular, we have $S_{1|\mathbb{A}_{r1}} = S_{(1)|\mathbb{A}_{r1}}$ for all $r\in [R]$.
    We prove the correctness of the algorithm indirectly. Thus, we start by assuming that the best arm is not contained in the last partition (i.e., the remaining active arm):
    \begin{align*}
        & \mathbb{A}_{R+1} \neq \{1\} \\
        \Leftrightarrow & \exists r\in [R] :  1 \notin \mathbb{A}_{r+1} \wedge 1\in \mathbb{A}_{r} \\
        \Rightarrow &\exists r\in [R] : \sum\limits_{i\in \mathbb{A}_{r1}} \id\{s_{i|\mathbb{A}_{r1}}(b_r) \geq s_{1|\mathbb{A}_{r1}}(b_r)\} > f(|\mathbb{A}_{r1}|) \\
        \Rightarrow &\exists r\in [R] :   \sum\limits_{i\in \mathbb{A}_{r1}} \id\{S_{1|\mathbb{A}_{r1}} - S_{i|\mathbb{A}_{r1}} \leq S_{1|\mathbb{A}_{r1}} - s_{1|\mathbb{A}_{r1}}(b_r) - S_{i|\mathbb{A}_{r1}} + s_{i|\mathbb{A}_{r1}}(b_r)\} > f(|\mathbb{A}_{r1}|) \\
        \Rightarrow &\exists r\in [R] : \sum\limits_{i\in \mathbb{A}_{r1}} \id\{S_{1|\mathbb{A}_{r1}} - S_{i|\mathbb{A}_{r1}} \leq |S_{1|\mathbb{A}_{r1}} - s_{1|\mathbb{A}_{r1}}(b_r)| + |S_{i|\mathbb{A}_{r1}} - s_{i|\mathbb{A}_{r1}}(b_r)|\} > f(|\mathbb{A}_{r1}|)  \\
        \Rightarrow &\exists r\in [R] : \sum\limits_{i\in \mathbb{A}_{r1}} \id\{S_{1|\mathbb{A}_{r1}} - S_{i|\mathbb{A}_{r1}} \leq 2 \bar{\gamma}_{\mathbb{A}_{r1}}(b_{r})\} > f(|\mathbb{A}_{r1}|) \\
        \Rightarrow &\exists r\in [R] : S_{1|\mathbb{A}_{r1}} - S_{(f(\mathbb{A}_{r1})+1)|\mathbb{A}_{r1}} \leq 2 \bar{\gamma}_{\mathbb{A}_{r1}}(b_{r})\\
        \Rightarrow &\exists r\in [R] : \left\lfloor \frac{B}{P_{r}R} \right\rfloor = b_{r} \leq \bar{\gamma}_{\mathbb{A}_{r1}}^{-1}\left( \frac{S_{1|\mathbb{A}_{r1}} - S_{(f(\mathbb{A}_{r1})+1)|\mathbb{A}_{r1}}}{2} \right) \\
        \Rightarrow &\exists r\in [R] : B \leq  P_{r} R \left\lceil \bar{\gamma}_{\mathbb{A}_{r1}}^{-1}\left(\frac{S_{1|\mathbb{A}_{r1}} - S_{(f(|\mathbb{A}_{r1}|)+1)|\mathbb{A}_{r1}}}{2}\right) \right\rceil  \\
        \Rightarrow &B \leq R \max\nolimits_{r\in [R]} P_{r}\left\lceil \bar{\gamma}_{\mathbb{A}_{r1}}^{-1}\left(\frac{S_{1|\mathbb{A}_{r1}} - S_{(f(|\mathbb{A}_{r1}|)+1)|\mathbb{A}_{r1}}}{2}\right) \right\rceil = z\left(f,R,\{P_{r}\}_{1\leq r\leq R}\right),
    \end{align*} 
    which contradicts the assumption we make on the budget $B.$
    Thus, it holds that the remaining active arm in round $R+1$ is $i^*=1.$ 
\end{proof}

\begin{Rem} \label{remark:rough_suff_budget}
    Using the definition of $\bar{\gamma}_Q(t)$ and $\bar{\gamma}(t)$ we can derive the following more coarser bounds on the sufficient budget:
    \begin{align*}
        z_1\left(f,R,\{P_{r}\}_{1\leq r\leq R}\right) &= R \max\nolimits_{r\in [R]} P_{r}\left\lceil \bar{\gamma}^{-1}\left(\frac{S_{1|\mathbb{A}_{r1}} - S_{(f(|\mathbb{A}_{r1}|)+1)|\mathbb{A}_{r1}}}{2}\right) \right\rceil, \\
        z_2\left(f,R,\{P_{r}\}_{1\leq r\leq R}\right) &= R \left(\max\nolimits_{r\in [R]} P_{r}\right) \max\nolimits_{Q \in \mathcal{Q}_{\leq k}}\left\lceil \bar{\gamma}^{-1}\left(\frac{S_{1|Q} - S_{(f(|Q|)+1)|Q}}{2}\right) \right\rceil .
    \end{align*}
\end{Rem}

\begin{proof}[Proof of Theorem~\ref{Thm:general_necessity}]
	After relabeling, we may suppose w.l.o.g. $i^{\ast} = 1$. Let $\beta:\N \ra (0,\infty)$ be an arbitrary strictly decreasing function with $\beta(t) \ra 0$ as $t\ra \infty$ and
	\begin{equation*}
		\left\{ \frac{S_{1|Q}-S_{j|Q}}{2} : Q\in \mathcal{Q}_{\leq k}, j\in Q\right\}  \subseteq \beta(\N).
	\end{equation*}
	Then, $\beta$ is invertible on $\beta(\N)$ and its inverse function $\beta^{-1} : \beta(\N) \ra \N$ trivially fulfills $\beta^{-1}(\alpha) = \min \{t\in \N \, : \, \beta(t) \leq \alpha\}$ for all $\alpha \in \beta(\N)$.
	Define for any $Q\in \mathcal{Q}_{\leq k}$ and $i\in Q$ the family of statistics by means of 
	\begin{equation*}
		s_{i|Q}(t) \coloneqq \begin{cases} S_{i|Q}-\beta(t), \quad &\text{if } i=\argmax_{j\in Q} S_{j|Q}, \\ S_{i|Q} + \beta(t), \quad &\text{otherwise,} \end{cases}
	\end{equation*}
	and note that $\ol{\gamma}_{Q}(t) = \beta(t)$ for all $Q\in \mathcal{Q}_{\leq k}$ and $t\in \N$. Writing $b_{r} = \left\lfloor \frac{B}{RP_{r}}\right\rfloor$ we obtain due to the choice of $\beta$ that   
	{\small 
	\begin{align*}
		&B <R \max\nolimits_{r\in [R]} P_{r} \bar{\gamma}_{\mathbb{A}_{r1}}^{-1}\left(\frac{S_{1|\mathbb{A}_{r1}} - S_{(f(|\mathbb{A}_{r1}|)+1)|\mathbb{A}_{r1}}}{2}\right) \\
		\Rightarrow & \exists r\in [R] : B < R P_{r} \min \left\{ t\in \N \, : \, \bar{\gamma}_{\mathbb{A}_{r1}}(t) \leq \frac{S_{1|\mathbb{A}_{r1}} - S_{(f(|\mathbb{A}_{r1}|)+1)|\mathbb{A}_{r1}}}{2} \right\}\\
		\Rightarrow & \exists r\in [R] : b_{r} < \min \left\{ t\in \N \, : \, \beta(t) \leq \frac{S_{1|\mathbb{A}_{r1}} - S_{(f(|\mathbb{A}_{r1}|)+1)|\mathbb{A}_{r1}}}{2} \right\}=\beta^{-1}\left( \frac{S_{1|\mathbb{A}_{r1}} - S_{(f(|\mathbb{A}_{r1}|)+1)|\mathbb{A}_{r1}}}{2} \right) \\
		\Rightarrow & \exists r\in [R] : 2\beta\left( b_{r} \right) > S_{1|\mathbb{A}_{r1}} - S_{(f(|\mathbb{A}_{r1}|)+1)|\mathbb{A}_{r1}} 
		= s_{1|\mathbb{A}_{r1}}(b_{r})+\beta(b_{r}) - \left( s_{(f(|\mathbb{A}_{r1}|)+1)|\mathbb{A}_{r1}}(b_{r}) - \beta(b_{r}) \right) \\		
		\Rightarrow & \exists r\in [R] :  s_{1|\mathbb{A}_{r1}}(b_{r}) <   s_{(f(|\mathbb{A}_{r1}|)+1)|\mathbb{A}_{r1}}(b_{r})\\	
		\Rightarrow & \exists r\in [R] :  1\not\in \mathbb{A}_{r+1}\\
		\Rightarrow &1\not\in \mathbb{A}_{R+1}.
	\end{align*}}
	This shows that $z\left(f,R,\{P_{r}\}_{1\leq r\leq R}\right)$ is the necessary budget for returning the best arm $i^{\ast}$ in this scenario. 
\end{proof}

\subsection{Proof of Corollary \ref{Cor:CSXsufficiency}} \label{sec_proofs_corollaries}
For sake of convenience, we provide the entire pseudo-code of CSWS in Algorithm \ref{alg:CSWS}, which results by using $f(x)=x-1$ as well as  $P_{r}^{\mathrm{CSWS}}$ and $R^{\mathrm{CSWS}}$ as defined in Section \ref{sec:gen_framework} in Algorithm \ref{alg:Framework}.
    \begin{algorithm}[!ht]
    \footnotesize
	\caption{Combinatorial Successive Winner Stays (CSWS)} \label{alg:CSWS}
	\hspace*{1pt} 
	\textbf{Input:} set of arms $[n]$, subset size $k\le n$, sampling budget $B$\\
	\hspace*{1pt} 
	\textbf{Initialization:}  
	For each $r\in \{1,\dots, \lceil \log_k(n) \rceil +1\}$ let $b_{r} \coloneqq \left\lfloor \frac{B}{\ \left\lceil \frac{n}{k^r} \right\rceil \cdot (\lceil \log_k(n) \rceil +1 ) } \right\rfloor$, 
	$\mathbb{A} \leftarrow [n]$, \\
	\hphantom{\textbf{Initialization:} \ } $r \leftarrow 1$
	\begin{algorithmic}[1]
	    \WHILE {$|\mathbb{A}_r| \ge k$}
	        \STATE $J = \lceil\frac{n}{k^r}\rceil$
	        \STATE $\mathbb{A}_{r1}, \mathbb{A}_{r2},\dots,\mathbb{A}_{r,J} \leftarrow Partition(\mathbb{A}_r, k)$
	        \IF {$|\mathbb{A}_{r,J}| < k $}
	            \STATE $\mathcal{R} \leftarrow \mathbb{A}_{r,J}$, $J \leftarrow J-1$
            \ELSE 
	            \STATE $\mathcal{R} \leftarrow \emptyset$
	       \ENDIF
	       \STATE $\mathbb{A}_{r+1} \leftarrow \emptyset$
	       \FOR {$j \in [J]$}
	            \STATE Play the set $\mathbb{A}_{r,j}$ for $b_{r}$ times
	            \STATE For all $i\in \mathbb{A}_{r,j}$, update $s_{i|\mathbb{A}_{r,j}}(b_r)$
	            \STATE Let $w \in \argmax_i~ s_{i|\mathbb{A}_{r,j}}(b_r)$
	            \STATE $\mathbb{A}_{r+1} \leftarrow \mathbb{A}_{r+1} \cup \{w\}$
	       \ENDFOR
	       \STATE $\mathbb{A}_{r+1} \leftarrow \mathbb{A}_{r+1} \cup \mathcal{R}$
	       \STATE $r \leftarrow r+1$
	   \ENDWHILE
	   \STATE $\mathbb{A}_{r+1} \leftarrow \emptyset$
	   \WHILE {$|\mathbb{A}_r| > 1$}
	        \STATE Play the set $\mathbb{A}_r $ for $b_{r}$ times
	        \STATE For all $i \in \mathbb{A}_r$, update $s_{i|\mathbb{A}_r}(b_r)$
	        \STATE Let $w \in \argmax_i~ s_{i|\mathbb{A}_{r}}(b_r)$
	        \STATE $\mathbb{A}_{r+1} \leftarrow \mathbb{A}_{r+1} \cup \{w\}$
	        \STATE $r \leftarrow r+1$
	   \ENDWHILE
	\end{algorithmic} 
	\hspace*{1pt} \textbf{Output:} The remaining item in $\mathbb{A}_r$
\end{algorithm}

    \begin{proof}[Proof of Corollary \ref{Cor:CSXsufficiency} (CSWS case)]
    Suppose $B >0$ to be arbitrary but fixed.
    First, note that there are at most $\lceil \log_k(n) \rceil$ rounds within the first while-loop and at most 1 in the second, so that we have at most $\lceil \log_k(n) \rceil +1 $ many rounds in total.
    The total number of partitions in round $r\in \{1,\ldots, \lceil \log_k(n) \rceil +1\}$ is at most $\left\lceil \frac{n}{k^r} \right\rceil.$ 
    Abbreviating $R\coloneqq R^{\mathrm{CSWS}}$ and $P_{r} \coloneqq P_{r}^{\mathrm{CSWS}}$ for the moment, the budget allocated to a partition in round $r$ is by definition $b_{r} = \lfloor \frac{B}{RP_{r}} \rfloor = \left\lfloor\frac{B}{\ \left\lceil \frac{n}{k^r} \right\rceil \cdot (\lceil \log_k(n) \rceil +1) }\right\rfloor$. Hence, the total budget used by CSWS is 
     \begin{align*}
        \sum\limits_{r=1}^{\lceil\log_k(n)\rceil+1} \#\{\text{partitions in round r}\} \cdot b_{r}
        = \sum\limits_{r=1}^{\lceil\log_k(n)\rceil+1} \left\lceil \frac{n}{k^r} \right\rceil \left\lfloor\frac{B}{\ \left\lceil \frac{n}{k^r} \right\rceil \cdot (\lceil \log_k(n) \rceil +1) }\right\rfloor
        \le B.
    \end{align*}
    Thus, the stated correctness of CSWS follows directly from Theorem \ref{Thm_generalSufficient}.

    \end{proof} 
    For sake of convenience, we provide the entire pseudo-code of CSR in Algorithm \ref{alg:CSR}, which results by using $f(x)=1$ as well as  $P_{r}^{\mathrm{CSR}}$ and $R^{\mathrm{CSR}}$ as defined in Section \ref{sec:gen_framework} in Algorithm \ref{alg:Framework}.
\begin{algorithm}[!ht]
\footnotesize
	\caption{Combinatorial Successive Reject (CSR)} \label{alg:CSR}
	\hspace*{1pt} \textbf{Input:} set of arms $[n]$, subset size $k\le n$, sampling budget $B$\\
	\hspace*{1pt} \textbf{Initialization:}  For each $r\in \{0,\dots,\lceil\log_{1-\frac{1}{k}}\left(\frac{1}{n}\right)\rceil \}$ let $b_{r} \coloneqq \left\lfloor\frac{B}{ \left\lceil\frac{n(1-\frac{1}{k})^{r-1}}{k}\right\rceil\left(\left\lceil\log_{1-\frac{1}{k}}\left(\frac{1}{n}\right)\right\rceil+k-1\right)}\right\rfloor$, \\
	\hphantom{\hspace*{1pt} \textbf{Initialization:}}$\mathbb{A} \leftarrow [n]$, $r \leftarrow 1$
	\begin{algorithmic}[1]
	    \WHILE {$|\mathbb{A}_r| \ge k$}
	        \STATE $J = \lceil\frac{n(1-\frac{1}{k})^{r-1}}{k}\rceil$
	        \STATE $\mathbb{A}_{r1}, \mathbb{A}_{r2},\dots,\mathbb{A}_{r,J} \leftarrow Partition(\mathbb{A}_r, k)$
	        \IF {$|\mathbb{A}_{r,J}| < k $}
	            \STATE $\mathcal{R} \leftarrow \mathbb{A}_{r,J}$, $J \leftarrow J-1$
            \ELSE 
	            \STATE $\mathcal{R} \leftarrow \emptyset$
	       \ENDIF
	       \STATE $\mathbb{A}_{r+1} \leftarrow \mathbb{A}_{r}$
	       \FOR {$j \in [J]$}
	            \STATE Play the set $\mathbb{A}_{r,j}$ for $b_{r}$ times
	            \STATE For all $i\in \mathbb{A}_{r,j}$, update $s_{i|\mathbb{A}_{r,j}}(b_r)$
	            \STATE Let $w \in \argmin_i~ s_{i|\mathbb{A}_{r,j}}(b_r)$
	            \STATE $\mathbb{A}_{r+1} = \mathbb{A}_{r+1} \backslash \{w\}$
	       \ENDFOR
	       \STATE $\mathbb{A}_{r+1} \leftarrow \mathbb{A}_{r+1} \cup \mathcal{R}$
	       \STATE $r \leftarrow r+1$
	   \ENDWHILE
	   \STATE$\mathbb{A}_{r+1} \leftarrow \mathbb{A}_{r}$
	   \WHILE {$|\mathbb{A}_r| > 1$}
	        \STATE Play the set $\mathbb{A}_r $ for $b_{r}$ times
	        \STATE For all $i \in \mathbb{A}_r$, update $s_{i|\mathbb{A}_r}(b_r)$
	        \STATE Let $w \in \argmin_i~ s_{i|\mathbb{A}_{r}}(b_r)$
	        \STATE $\mathbb{A}_{r+1} = \mathbb{A}_{r+1} \backslash \{w\}$
	        \STATE $r \leftarrow r+1$
	   \ENDWHILE
	\end{algorithmic} 
	\hspace*{1pt} \textbf{Output:} The remaining item in $\mathbb{A}_r$
\end{algorithm}
    \begin{proof}[Proof of Corollary \ref{Cor:CSXsufficiency} (CSR case)]
    Suppose $B >0$ to be arbitrary but fixed.
    First, note that there are at most $\left\lceil \log_{1-\frac{1}{k}}\left(\frac{1}{n}\right) \right\rceil$ rounds within the first while-loop and at most $k-1$ in the second, so that we have at most $\left\lceil \log_{1-\frac{1}{k}}\left(\frac{1}{n}\right) \right\rceil +k-1 $ many rounds in total.
    The total number of partitions in round $r\in \{1,\ldots, \left\lceil \log_{1-\frac{1}{k}}\left(\frac{1}{n}\right) \right\rceil +k-1 \}$ is at most $ \left\lceil \frac{n(1-\frac{1}{k})^{r-1}}{k} \right\rceil$. 
    The budget allocated to a partition in round $r$ (i.e., $b_r$) is by definition given by 
    \begin{equation*} 
    b_{r} = \lfloor B/(R^{\mathrm{CSR}}P_{r}^{\mathrm{CSR}})\rfloor = \left\lfloor \frac{B}{ \left\lceil\frac{n(1-\frac{1}{k})^{r-1}}{k}\right\rceil \left(\left\lceil\log_{1-\frac{1}{k}}\left(\frac{1}{n}\right)\right\rceil+k-1\right)} \right\rfloor.
    \end{equation*} 
    Consequently, the total budget used by CSR is 
     \begin{align*}
        &\sum\limits_{r=1}^{\left\lceil \log_{1-\frac{1}{k}}\left(\frac{1}{n}\right) \right\rceil +k-1 } \#\{\text{partitions in round r}\} \cdot b_{r}\\
        &= \sum\limits_{r=1}^{\left\lceil \log_{1-\frac{1}{k}}\left(\frac{1}{n}\right) \right\rceil +k-1 } \left\lceil \frac{n(1-\frac{1}{k})^{r-1}}{k} \right\rceil \left\lfloor \frac{B}{ \left\lceil\frac{n(1-\frac{1}{k})^{r-1}}{k}\right\rceil \left(\left\lceil\log_{1-\frac{1}{k}}\left(\frac{1}{n}\right)\right\rceil+k-1\right)} \right\rfloor \\
        &\le B.
    \end{align*}
    Therefore, the statement follows from Theorem \ref{Thm_generalSufficient}.

\end{proof}
    For sake of convenience, we provide the entire pseudo-code of CSH in Algorithm \ref{alg:CSH}, which results by using $f(x)=\lceil x/2 \rceil$ as well as  $P_{r}^{\mathrm{CSH}}$ and $R^{\mathrm{CSH}}$ as defined in Section \ref{sec:gen_framework} in Algorithm \ref{alg:Framework}.
\begin{algorithm}[!ht]
\footnotesize
	\caption{Combinatorial Successive Halving (CSH)} \label{alg:CSH}
	\hspace*{1pt} \textbf{Input:} set of arms $[n]$, subset size $k\le n$, sampling budget $B$\\
	\hspace*{1pt} \textbf{Initialization:}  For each $r\in \{0,\dots,\left\lceil \log_2(n) \right\rceil + \lceil \log_2(k) \rceil\}$ let $b_{r} \coloneqq \left\lfloor \frac{Bk}{\lceil \frac{n}{2^{r-1}} \rceil \left(\left\lceil \log_2(n) \right\rceil + \lceil \log_2(k) \rceil \right)}\right\rfloor$, \\
	\hphantom{\hspace*{1pt} \textbf{Initialization:}} $\mathbb{A} \leftarrow [n]$, $r \leftarrow 1$
	\begin{algorithmic}[1]
	    \WHILE {$|\mathbb{A}_r| \ge k$}
	        \STATE $J = \lceil\frac{n}{2^{r-1}k}\rceil$
	        \STATE $\mathbb{A}_{r1}, \mathbb{A}_{r2},\dots,\mathbb{A}_{r,J} \leftarrow Partition(\mathbb{A}_r, k)$
	        \IF {$|\mathbb{A}_{r,j}| < k $}
	            \STATE $\mathcal{R} \leftarrow \mathbb{A}_{r,j}$, $J \leftarrow J-1$
            \ELSE 
	            \STATE $\mathcal{R} \leftarrow \emptyset$
	       \ENDIF
	       \FOR {$j \in [J]$}
	            \STATE Play the set $\mathbb{A}_{r,j}$ for $b_{r}$ times
	            \STATE For all $i\in \mathbb{A}_{r,j}$, update $s_{i|\mathbb{A}_{r,j}}(b_r)$
	            \STATE Define $\bar{s} \leftarrow \text{Median}(\{s_{i|\mathbb{A}_{r,j}}(b_r)\}_{i\in \mathbb{A}_{r,j}})$
	            \STATE $\mathbb{A}_{r+1} \leftarrow \{i \in \mathbb{A}_{r,j} | s_{i|\mathbb{A}_{r,j}}(b_r) \le \bar{s}\}$
	       \ENDFOR
	       \STATE $\mathbb{A}_{r+1} \leftarrow \mathbb{A}_{r+1} \cup \mathcal{R}$
	       \STATE $r \leftarrow r+1$
	   \ENDWHILE
	   \STATE $\mathbb{A}_r \leftarrow \mathbb{A}_r \cup \{k - |\mathbb{A}_r|$ random elements from $[n]\backslash \mathbb{A}_r \}$
	   \WHILE {$|\mathbb{A}_r| > 1$}
	        \STATE Play the set $\mathbb{A}_r $ for $b_{r}$ times
	        \STATE For all $i \in \mathbb{A}_r$, update $s_{i|\mathbb{A}_r}(b_r)$
	        \STATE Define $\bar{s} \leftarrow \text{Median}(\{s_{i|\mathbb{A}_r}(b_r)\}_{i\in\mathbb{A}_r})$
	        \STATE $\mathbb{A}_{r+1} \leftarrow \{i \in \mathbb{A}_r | s_{i|\mathbb{A}_r}(b_r) \le \bar{s}\}$
	        \STATE $r \leftarrow r+1$
	   \ENDWHILE
	\end{algorithmic} 
	\hspace*{1pt} \textbf{Output:} The remaining item in $\mathbb{A}_r$
\end{algorithm}
\begin{proof}[Proof of Corollary \ref{Cor:CSXsufficiency} (CSH case)]
    Suppose $B >0$ to be arbitrary but fixed.
    First, note that there are at most $\lceil \log_2(n) \rceil$ rounds within the first while-loop and at most $\lceil \log_2(k) \rceil$ in the second, so that we have at most $\lceil \log_2(n) \rceil + \lceil \log_2(k) \rceil$ many rounds in total.
    The total number of partitions in round $r = 1,\ldots, \lceil \log_2(n) \rceil + \lceil \log_2(k) \rceil $ is at most $ \left\lceil \frac{n}{2^{r-1} k} \right\rceil.$ 
    The budget allocated to a partition in round $r$ is 
    \begin{equation*} 
        b_{r} = \lfloor B/(R^{\mathrm{CSH}}P_{r}^{\mathrm{CSH}}) \rfloor = \left\lfloor \frac{B}{\ \left\lceil \frac{n}{2^{r-1} k} \right\rceil \cdot (\lceil \log_2(n) \rceil + \lceil \log_2(k) \rceil) } \right\rfloor.
    \end{equation*}
    In particular, the total budget used by CSH is
    \begin{align*}
        &\sum\limits_{r=1}^{\lceil\log_2(n)\rceil + \lceil \log_2(k) \rceil} \#\{\text{partitions in round r}\} \cdot b_{r}\\
        &= \sum\limits_{r=1}^{\lceil\log_2(n)\rceil + \lceil \log_2(k) \rceil} \left\lceil \frac{n}{2^{r-1} k} \right\rceil \cdot \left\lfloor \frac{Bk}{\lceil \frac{n}{2^{r-1}} \rceil (\lceil \log_2(n) \rceil + \lceil \log_2(k) \rceil)}\right\rfloor\\
        &\le B.
    \end{align*}
    Once again, Theorem \ref{Thm_generalSufficient} allows us to conclude the proof. 
\end{proof}

\section{Proofs of Section \ref{sec:application_stochastic}}\label{Appendix:proofsAppl} 

\subsection{Stochastic Numerical Feedback: Proof of Corollary \ref{Cor:ApplReward}}

A rich class of statistics can be obtained by applying a linear functional $U(F)= \int r(x) \mathrm{d}F(x),$ where $F$ is a cumulative distribution function and $r: \R \rightarrow \mathbb{R}$ some measurable function, on the empirical distribution function \citep{wasserman2013all}, i.e.,  for any $x \in \R$ and any multiset of (reward) observations $O$ %
\begin{center}
    $\tilde s(O,x) = \frac{1}{|O|}\sum_{o \in O} \id\{ x \leq o \}.$
\end{center}
This leads to the statistics 
$$s_{i|Q}(t) = U(\tilde s(o_{i|Q}(1), \dots, o_{i|Q}(t),\cdot)) =  \sum_{s=1}^{t}  \frac{r(o_{i|Q}(s))}{t},$$ 
which converge to $S_{i|Q} = \E_{X\sim \nu_{i|Q}}[r(X)]$ by the law of large numbers, provided these expected values exist.
In this section we show the following result which generalizes Corollary \ref{Cor:ApplReward} for statistics of the above kind.
\begin{Cor} \label{Cor:ApplReward_others}
	Let $f$, $R$ and $\{P_{r}\}_{r\in [R]}$ be as in Theorem \ref{Thm_generalSufficient} and suppose that $r(o_{i|Q}(t))$ are $\sigma$-sub-Gaussian and such that their means $S_{i|Q} := \E_{X\sim \nu_{i|Q}}[r(X)]$ satisfy (A2). 
	Then, there is a function 
	\begin{equation*} 
		C(\delta,\eps,k,R, \sigma)\in  \mathcal{O}\left(\sigma^{2} \eps^{-2}\ln\left( \nicefrac{kR}{\delta}\ln\left(\nicefrac{kR\sigma}{\eps \delta}\right)\right)\right)
	\end{equation*}
	with the following property:
	If $i^{\ast}$ is the GCW and  $\sup_{Q\in \mathcal{Q}_{\leq k}(i^{\ast})} \Delta_{(f(|Q|)+1)|Q} \leq \eps$, then Algorithm \ref{alg:Framework} used with a budget $B$ larger than $  C(\delta,\eps,k,R,\sigma) \cdot R \max\nolimits_{r\in [R]} P_{r} $
	returns $i^{\ast}$ with probability at least $1-\delta$.
\end{Cor}
Note that we immediately obtain the proof for Corollary \ref{Cor:ApplReward} as a special case of  Corollary \ref{Cor:ApplReward_others} by using the the identity function $r(x) = x.$

The following two lemmata serve as a preparation for the proof of Corollary \ref{Cor:ApplReward_others}. The proof of Lemma~\ref{AnytimeBoundUtility} is an adaptation of the proof of Lemma 3 in \cite{Jamieson14}.

\begin{Le} \label{AnytimeBoundUtility}
    Let $X_1, X_2, \ldots \sim \mathcal{X}$ be iid real-valued random variables and $r: \R \rightarrow \mathbb{R}$ such that $r(\mathcal{X})$ is $\sigma^2$-sub-Gaussian. 
    For any $\epsilon \in (0,1)$ and $\delta \in (0, \log(1+\epsilon)/e)$ one has with probability at least $1 - \frac{(2+\epsilon)}{\epsilon}\left( \frac{\delta}{\log(1+\epsilon)}\right)^{(1+\epsilon)}$ for any $t\geq 1$
    \begin{align*}
        &\sum_{i=1}^t r(X_i) - t \cdot \mathbb{E}_{X\sim \mathcal{X}}[r(X)] \leq (1 + \sqrt{\epsilon}) \sqrt{2 \sigma^2(1+\epsilon)t\log\left(\frac{\log((1+\epsilon)t)}{\delta}\right)}.
    \end{align*}
    Moreover, the same concentration inequality holds for $-\left(\sum_{i=1}^{t} r(X_i) - t \cdot \mathbb{E}_{X\sim\mathcal{X}}[r(X)] \right)$ as well.
\end{Le}
\begin{proof}
    We denote in the following $\psi(x) = \sqrt{2\sigma^{2} x \log\left(\frac{\log(x)}{\delta}\right)}$ and $R_t = \sum_{i=1}^{t} r(X_i) - t \cdot \mathbb{E}_{X \sim \mathcal{X}}[r(X)]$ and define a sequence of integers $(u_{k})$ as $u_0=1$ and $u_{k+1}=\lceil(1+\epsilon) u_{k}\rceil$. 
   The maximal \textit{Azuma-Hoeffding Inequality} states that for any martingale difference sequence $S_1,S_2,\ldots$ with each element being $\sigma^2$-sub-Gaussian, it holds that for any $\alpha >0$, $n\geq 1$:
    \begin{align*}
        \P\left(  \max\nolimits_{i\in [n] } S_i - S_0 \geq \alpha \right) \leq \exp\left(-\frac{\alpha^2}{2\sum_{j=1}^n \sigma^2_j}\right).
    \end{align*}
   In the following let  $\mathcal{F}_0 = \{\emptyset, \Omega\}$ be the trivial $\sigma$-algebra and for $k\in \{1, \dots, n\}$ let $\mathcal{F}_k = \sigma(X_1, \dots, X_k)$ be the $\sigma$-algebra generated by the observations $X_1,\dots,X_k$. 
   Then 
    \begin{align*}
        \mathbb{E}[R_{t+1} | \mathcal{F}_t] &= \mathbb{E}[r(X_{t+1}) - \mathbb{E}_{X \sim \mathcal{X}}[r(X)] + R_t |\mathcal{F}_t]\\
        &= \mathbb{E}[r(X_{t+1})|\mathcal{F}_t] - \mathbb{E}_{X \sim \mathcal{X}}[r(X)] + \mathbb{E}[R_t | \mathcal{F}_t]\\
        &= R_t
    \end{align*}
    which shows the martingale property of $R_t$. Note, that $R_0 = 0$ and $R_{t+1}-R_t = r(X_{t+1}) - \mathbb{E}_{X\sim \mathcal{X}}[r(X)],$ which is according to the assumption $\sigma^2$-sub-Gaussian and has zero mean, for any $t\in \N$.
    Thus, we can apply the maximal Azuma-Hoeffding inequality for $R_1,R_2,\ldots,R_t.$ 
    \\[0.5em]
    \noindent \textbf{Step 1.}\\
    \noindent  In the first step of the proof we derive a bound for the probability of a lower bound of $R_{u_k}$ for $k\geq 1$. For this we use the union bound, the maximal Azuma-Hoeffding inequality, the fact that $u_k \geq (1+\epsilon)^k$, a sum-integral comparison and some simple transformations and obtain
    \allowdisplaybreaks
    \begin{align*}
        \P &\left( \exists k \geq 1 ~:~ R_{u_k} \geq \sqrt{1+\epsilon} \psi(u_k)\right) \\
        &\leq \sum_{k=1}^{\infty} \P\left(R_{u_k} \geq \sqrt{1+\epsilon} \psi(u_k)\right)\\
        &\leq \sum_{k=1}^{\infty} \exp\left(-\frac{(1+\epsilon) \psi(u_k)^2}{2 u_k \sigma^2}\right)\\
        &= \sum_{k=1}^{\infty} \exp\left(-(1+\epsilon) \log\left(\frac{\log(u_k)}{\delta}\right)\right)\\
        &\leq \sum_{k=1}^{\infty} \exp\left(-(1+\epsilon) \log\left(\frac{\log((1+\epsilon)^k)}{\delta}\right)\right) \\
        &= \sum_{k=1}^{\infty} \left(\frac{\delta}{k\log((1+\epsilon))}\right)^{(1+\epsilon)} \\
        &= \left(\frac{2\delta}{\log((1+\epsilon))}\right)^{(1+\epsilon)} \sum_{k=1}^{\infty} \left(\frac{1}{k}\right)^{(1+\epsilon)}\\
        &= \left(\frac{\delta}{\log((1+\epsilon))}\right)^{(1+\epsilon)} \left(1+ \sum_{k=2}^{\infty} \left(\frac{1}{k}\right)^{(1+\epsilon)}\right) \\
        &\leq \left(\frac{\delta}{\log((1+\epsilon))}\right)^{(1+\epsilon)} \left(1+ \int_{k=1}^{\infty} \left(\frac{1}{k}\right)^{(1+\epsilon)}\right)\\
        &= \left(\frac{\delta}{\log((1+\epsilon))}\right)^{(1+\epsilon)} \left(1+ \left[-\frac{1}{\epsilon} \left(\frac{1}{k}\right)^{\epsilon}\right]_{1}^{\infty}\right)\\
        &= \left(\frac{\delta}{\log((1+\epsilon))}\right)^{(1+\epsilon)} \left(1+\frac{1}{\epsilon}\right).
    \end{align*}
    \noindent \textbf{Step 2.}\\
    \noindent 
    Next, we bound the probability that the difference between some $R_s$ and $R_t$ exceeds a lower bound for some $s=u_k$, $k\in \mathbb{N}$ and $s\le t \le u_{k+1}$. 
    Note that $R_t-R_{u_k}$ and $R_{t-u_k}$ have the same distribution, such that we obtain
    \begin{align*}
        \P &\left( \exists t \in \{ u_k+1,\dots , u_{k+1}-1\} ~:~ R_t-R_{u_k} \geq \sqrt{\epsilon} \psi (u_{k+1})\right) \\
        &= \P\left( \exists t \in [u_{k+1}-u_k-1] ~:~ R_t \geq \sqrt{\epsilon} \psi(u_{k+1}) \right)\\
        &\leq \exp \left( - \frac{\epsilon \psi(u_{k+1})^2}{2\sigma^2(u_{k+1}-u_k-1)}\right)\\
        & = \exp \left(- \frac{\epsilon u_{k+1}}{u_{k+1} - u_k -1} \log\left(\frac{\log(u_{k+1})}{\delta}\right)\right)\\
        &\leq \exp\left(- \frac{\epsilon u_{k+1}}{(1+\epsilon)u_k + 1 - u_k -1} \log\left(\frac{\log(u_{k+1})}{\delta}\right)\right)\\
        &= \exp\left(- \frac{u_{k+1}}{u_k } \log\left(\frac{\log(u_{k+1})}{\delta}\right)\right)\\
        &\leq \exp\left(- (1+\epsilon) \log\left(\frac{\log(u_{k+1})}{\delta}\right)\right)\\
        &\leq \left(\frac{\delta}{(k+1)\log(1+\epsilon)}\right)^{1+\epsilon},
    \end{align*}
    where we used once again the maximal Azuma-Hoeffding inequality and that $u_{k+1} \geq (1+\epsilon)u_k$ as well as that $\frac{u_{k+1}}{u_k} \geq 1+ \epsilon$. 
    For all possible $k\in \mathbb{N}$ we get with the union bound and a similar sum-integral comparison as above
    \begin{align*}
        \P & \left( \exists k \in \mathbb{N},~\exists t \in \{u_k+1,\dots, u_{k+1}-1\} ~:~ R_t-R_{u_k} \geq \sqrt{\epsilon} \psi(u_{k+1})\right) \\
        &\leq \sum_{k=1}^{\infty}\left(\frac{\delta}{(k+1)\log(1+\epsilon)}\right)^{1+\epsilon}\\
        &= \sum_{k=2}^{\infty}\left(\frac{\delta}{k\log(1+\epsilon)}\right)^{1+\epsilon}\\
        &\leq \int_{k=1}^{\infty}\left(\frac{\delta}{k\log(1+\epsilon)}\right)^{1+\epsilon}\\
        &= \left(\frac{\delta}{\log(1+\epsilon)}\right)^{1+\epsilon}\frac{1}{\epsilon}
    \end{align*}
    \noindent \textbf{Step 3.}\\
    \noindent 
    Finally, by combining Step 1 and 2 we can infer that for any $k\geq 0$ and $t \in \{u_k+1, \dots, u_{k+1}-1\}$ it holds  
    \begin{align*}
        R_t &= R_t - R_{u_k} + R_{u_k} \\
        & \leq \sqrt{\epsilon} \psi(u_{k+1}) + \sqrt{1+\epsilon} \psi(u_k) \\
        & \leq \sqrt{\epsilon} \psi((1+\epsilon)t) + \sqrt{1+\epsilon} \psi(t)\\
        &\leq (1+ \sqrt{\epsilon}) \psi((1+\epsilon)t),
    \end{align*}
    with probability at least $1-\frac{2+\epsilon}{\epsilon}\left(\frac{\delta}{\log(1+\epsilon)}\right)^{1+\epsilon}$ leading to the first claim of the lemma.\\[0.5em]
    \noindent \textbf{Step 4.}\\
    \noindent 
    Note that $\tilde R_t = t \cdot \mathbb{E}_{X \sim \mathcal{X}}[r(X)] - \sum_{i=1}^{t} r(X_i)$ is a martingale difference sequence with $\tilde R_{t+1} - \tilde R_t = - R_t + R_{t+1} =  \mathbb{E}_{X\sim \mathcal{X}}[r(X)] - r(X_{t+1}),$ which is according to the assumption $\sigma^2$-sub-Gaussian and has zero mean, for any $t\in \N$.
    Thus, repeating Step 1--3 for $\tilde R_1,\tilde R_2,\ldots,\tilde R_t$ shows the second claim of the lemma.
\end{proof}
\begin{Le}\label{Le_LilBound}
Let $X_1, X_2, \ldots \sim \mathcal{X}$ be iid real-valued random variables and $r: \R \rightarrow \mathbb{R}$ such that $r(\mathcal{X})$ is $\sigma^2$-sub-Gaussian. 
For any $\gamma \in (0,1)$ we have  
\begin{equation*}
	\P \left(  \exists t\in \N: \left| \sum_{i=1}^{t} r(X_i) - t \cdot \mathbb{E}_{X\sim\mathcal{X}}[r(X)] \right| > (1+\sqrt{1/2})\sqrt{3\sigma^{2}t\ln \left( \frac{10^{2/3}\ln(3t/2)}{\gamma^{2/3}\ln(3/2)} \right)} \right) \leq \gamma.
\end{equation*}
\end{Le}
\begin{proof}
	Let $\gamma \in (0,1)$ be fixed and $\eps \coloneqq 1/2$. Then, $\gamma'\coloneqq \left( \frac{\gamma}{10}\right)^{2/3} \ln(3/2)$ fulfills
	\begin{equation*}
		\frac{2+\eps}{\eps} \left( \frac{\gamma'}{\ln(1+\eps)}\right)^{1+\eps} = 5 \left( (\gamma/10)^{2/3} \right)^{3/2} = \gamma/2
	\end{equation*} 
	and moreover $\gamma' < (1/10)^{2/3}\ln(3/2) < e^{-1}\ln(3/2)$. Consequently, Lemma \ref{AnytimeBoundUtility} yields with 
	\begin{equation*}
		\tilde{c}_{\gamma}(t) \coloneqq (1+\sqrt{\eps})\sqrt{2\sigma^{2}(1+\eps)t\ln\left( \frac{\ln((1+\eps)t)}{\gamma'}\right)} =  (1+\sqrt{1/2})\sqrt{3\sigma^{2}t\ln \left( \frac{10^{2/3}\ln(3t/2)}{\gamma^{2/3}\ln(3/2)} \right)}
	\end{equation*} 
	that 
	\begin{equation*}
		\P\left( \exists t\in \N :  \sum_{i=1}^{t} r(X_i) - t \cdot \mathbb{E}_{X\sim\mathcal{X}}[r(X)]  > \tilde{c}_{\gamma}(t) \right) \leq \gamma/2.
	\end{equation*}
	as well as
	\begin{equation*}
			\P\left( \exists t\in \N : - \left( \sum_{i=1}^{t} r(X_i) - t \cdot \mathbb{E}_{X\sim\mathcal{X}}[r(X)] \right)  > \tilde{c}_{\gamma}(t) \right) \leq \gamma/2.
	\end{equation*}
	Thus, we obtain 
	\begin{align*}
		\P \Bigg( & \exists t\in \N : \left| \sum_{i=1}^{t} r(X_i) - t \cdot \mathbb{E}_{X\sim\mathcal{X}}[r(X)]  \right| > \tilde{c}_{\gamma}(t) \Bigg) \\
		&\leq \P\left( \exists t\in \N :  \sum_{i=1}^{t} r(X_i) - t \cdot \mathbb{E}_{X\sim\mathcal{X}}[r(X)]  > \tilde{c}_{\gamma}(t) \right)  \\
		&\ + \P\left( \exists t\in \N : - \left( \sum_{i=1}^{t} r(X_i) - t \cdot \mathbb{E}_{X\sim\mathcal{X}}[r(X)] \right)  > \tilde{c}_{\gamma}(t) \right)\\
		&\leq \gamma/2 + \gamma/2 = \gamma.
	\end{align*} 
\end{proof}
We are now ready to prove  Corollary \ref{Cor:ApplReward_others}.
\begin{proof}[Proof of Corollary \ref{Cor:ApplReward_others}]
    Recall the definition of $\tilde{c}_{\gamma}(t)$ from the proof of Lemma \ref{Le_LilBound} and let
\begin{equation*}
      c_{\gamma}(t) \coloneqq \frac{2}{t} \tilde{c}_{\gamma}(t) = 2(1+\sqrt{1/2})\sqrt{ \frac{3\sigma^{2}}{t} \ln \left(  \frac{10^{2/3}\ln(3t/2)}{\gamma^{2/3}\ln(3/2)}\right)}
\end{equation*}
for any $\gamma\in (0,1)$, $t\in \N$. For any fixed $\gamma$, $c_{\gamma}:\N \ra (0,\infty), t\mapsto c_{\gamma}(t)$ 
is strictly monotonically decreasing with $\lim_{t\ra \infty} c_{\gamma}(t) = 0$. Contraposition of (1) in \cite{Jamieson14} states
\begin{equation*}
    \quad t> \frac{1}{c} \ln \left( \frac{ 2\ln((1+\eps)/(c\omega))}{\omega} \right) \ \Rightarrow \ c> \frac{1}{t} \ln\left( \frac{\ln((1+\eps)t)}{\omega}\right) \qquad  \forall t\geq 1, \eps\in (0,1), c>0, \omega \leq 1.
\end{equation*}
For any $\alpha>0$ and $\gamma\in (0,1)$, using this with  $\omega = \frac{\gamma^{2/3} \ln(3/2)}{10^{2/3}}, c=\frac{\alpha^2}{12(1+\sqrt{1/2})^{2}\sigma^{2}}$ and  $\eps=1/2$ reveals
\begin{align*}
    c_{\gamma}^{-1}(\alpha) &= \min \left\{ t\in \N \, : \, c_{\gamma}(t) \leq \alpha \right\} \\
    &= \min\left\{ t\in \N \, : \, \frac{1}{t} \ln  \left( \frac{10^{2/3}\ln(3t/2)}{\gamma^{2/3}\ln(3/2)}\right) \leq \frac{\alpha^{2}}{12(1+\sqrt{1/2})^{2}\sigma^{2}}  \right\}.
\end{align*}
Thus, we have $c \geq \frac{1}{t} \ln\left( \frac{\ln((1+\eps)t)}{\omega}\right)$ and we know, that this statement is true if $t\geq \frac{1}{c} \ln \left( \frac{ 2\ln((1+\eps)/(c\omega))}{\omega} \right)$. In particular also for the smallest such $t$, for which holds $t \leq \left\lceil \frac{1}{c} \ln \left( \frac{ 2\ln((1+\eps)/(c\omega))}{\omega} \right) \right\rceil$ + 1.
It follows
\begin{align*}
    c_{\gamma}^{-1}(\alpha) \leq \left\lceil \frac{12(1+\sqrt{1/2})^{2}\sigma^{2}}{\alpha^{2}} \ln \left( \frac{2\cdot 10^{2/3}}{\gamma^{2/3} \ln(3/2)} \ln \left( \frac{18 \cdot 10^{2/3}(1+\sqrt{1/2})^{2}\sigma^{2}}{ \gamma^{2/3}\ln(3/2) \, \alpha^{2}}\right) \right) \right\rceil + 1,
\end{align*}
which is of the order $\mathcal{O}(\sigma^2\alpha^{-2} \ln \ln (\alpha^{-1}\sigma) \ln \gamma^{-1})$.\par 
Now, suppose $\max_{Q\in \mathcal{Q}_{\leq k}(i^{\ast})} \Delta_{(f(|Q|)+1)|Q} \leq \eps$ and that  Algorithm \ref{alg:Framework} is started with a budget $B$ larger than
\begin{align*}
    c_{\delta/(kR)}^{-1}(\eps/2) \cdot R \max\nolimits_{r\in [R]} P_{r}.
\end{align*}
Recall that $ \gamma_{i|Q}(t) = | s_{i|Q}(t)  - S_{i|Q}|$, $s_{i|Q}(t) = \frac{1}{t} \sum_{s=1}^{t} r(o_{i|Q}(s))$ and $S_{i|Q} = \mathbb{E}_{X\sim \nu_{i|Q}}[r(X)].$
With this, we obtain for any possible sequence of partitions $(E_{r})_{r\in [R]} \in (\mathcal{Q}_{\leq k})^{R}$ with $\P(\mathbb{A}_{r}(i^{\ast}) = E_{r} \, \forall r\in [R]) >0$ that
{\small
    \begin{align*}
    &\P\left( \exists t\in \N, r\in [R], i\in E_{r} :  \gamma_{i|E_{r}}(t)
    \geq c_{\delta/(kR)}(t)
    \Big| \, \mathbb{A}_{r}(i^{\ast}) = E_{r} \, \forall r\in [R] \right) \\
    &\leq \sum\limits_{r\in [R]} \sum\limits_{i \in E_{r}} \P\left( \exists t\in \N :  \gamma_{i|E_{r}}(t) \geq c_{\delta/(kR)}(t) 
    \Big| \, \mathbb{A}_{r}(i^{\ast}) = E_{r} \, \forall r\in [R] \right) \\
    &= \sum\limits_{r\in [R]} \sum\limits_{i \in E_{r}} \P\left( \exists t\in \N :   \Big| \frac{1}{t}\sum_{t'=1}^{t} r(o_{i|E_{r}}(t')) - \mathbb{E}_{X\sim \nu_{i|E_r}}[r(X)] \Big|  \geq c_{\delta/(kR)}(t)  \,
    \Big| \, \mathbb{A}_{r}(i^{\ast}) = E_{r} \, \forall r\in [R] \right) \\
    &= \sum\limits_{r\in [R]} \sum\limits_{i \in E_{r}} \P\left( \exists t\in \N :   \frac{1}{t} \Big| \sum_{t'=1}^{t} r(o_{i|E_{r}}(t')) - t \cdot \mathbb{E}_{X\sim \nu_{i|E_r}}[r(X)] \Big|  \geq c_{\delta/(kR)}(t)  \,
    \Big| \, \mathbb{A}_{r}(i^{\ast}) = E_{r} \, \forall r\in [R] \right) \\
    &= \sum\limits_{r\in [R]} \sum\limits_{i \in E_{r}} \P\left( \exists t\in \N :  \Big|  \sum_{t'=1}^{t} r(o_{i|E_{r}}(t')) - t \cdot \mathbb{E}_{X\sim \nu_{i|E_r}}[r(X)] \Big|  \geq \tilde{c}_{\delta/(kR)}(t)  \,
    \Big| \, \mathbb{A}_{r}(i^{\ast}) = E_{r} \, \forall r\in [R] \right) \\
    &\leq \sum\limits_{r\in [R]} \sum\limits_{i\in E_{r}} \frac{\delta }{k R} \leq \delta,
\end{align*}}
where we used Lemma \ref{Le_LilBound} for the second last inequality.
Using the law of total probability for all possible sequences of partitions $(E_{r})_{r\in [R]},$ we see that the event 
\begin{equation*}
    \mathcal{E} \coloneqq \left\{\exists t\in \N,  r\in [R], i\in \mathbb{A}_{r}(i^{\ast}) :  \gamma_{i|\mathbb{A}_{r}(i^{\ast})}(t)
    \geq c_{\delta/(kR)}(t)  \right\}
\end{equation*}
occurs with probability
\begin{align*} 
    &\P(\mathcal{E} ) \\ 
    &= \sum_{(E_{r})_{r\in [R]}}\P\left( \exists t\in \N, r\in [R], i\in E_{r} :
    \gamma_{i|E_{r}}(t) \geq c_{\delta/(kR)}(t) \, \Big| \, \mathbb{A}_{r}(i^{\ast}) = E_{r} \, \forall r\in [R] \right) \\
    &\quad \times \ 
    \P(\mathbb{A}_{r}(i^{\ast}) = E_{r} \, \forall r\in [R])\\
    &\leq \delta \, \sum\nolimits_{(E_{r})_{r\in [R]}} \P(\mathbb{A}_{r}(i^{\ast}) = E_{r} \, \forall r\in [R] ) = \delta.
\end{align*}
On $\mathcal{E}^{c}$ we have $\ol{\gamma}_{\mathbb{A}_{r}(i^{\ast})}(t) < c_{\delta/(kR)}(t)$ for all $t\in \N, r\in [R]$ and thus in particular $\ol{\gamma}_{\mathbb{A}_{r}(i^{\ast})}^{-1}(\alpha) \geq c_{\delta/(kR)}^{-1}(\alpha)$ for any $\alpha \in (0,\infty)$. Since  $\max_{Q\in \mathcal{Q}_{\leq k}(i^{\ast})} \Delta_{(f(|Q|)+1)|Q} \leq \eps$, Theorem \ref{Thm_generalSufficient} thus lets us conclude
\begin{align*}
    \P\left( \text{Alg. 1 returns } i^{\ast} \right)
    &\geq  \P\left( B > R \max\nolimits_{r\in [R]} P_{r} \ol{\gamma}_{\mathbb{A}_{r}(i^{\ast})}^{-1} \left( \frac{\Delta_{(f(|\mathbb{A}_{r}(i^{\ast})|)+1)|\mathbb{A}_{r}(i^{\ast})}}{2}\right) \right)\\
    &\geq  \P\left( \left\{ B > R \max\nolimits_{r\in [R]} P_{r} \ol{\gamma}_{\mathbb{A}_{r}(i^{\ast})}^{-1} \left( \eps/2\right) \right\} \cap \mathcal{E}^{c}\right)\\
     &\geq  \P\left( \left\{ B > R \max\nolimits_{r\in [R]} P_{r} c_{\delta/(kR)}^{-1} \left( \eps/2\right) \right\} \cap \mathcal{E}^{c}\right)\\
     &= \P(\mathcal{E}^{c}) \geq 1-\delta,
\end{align*}
where the equality holds due to the assumption on $B$. 
Consequently, we can conclude the proof by defining
{\small
\begin{align*}
      C(\delta,\eps,k,R) &\coloneqq c_{\delta/(kR)}^{-1}(\eps/2)\\
      &\leq \left\lceil \frac{48(1+\sqrt{1/2})^{2}\sigma^{2}}{\eps^{2}} 
      \ln \left( \frac{2(10kR)^{2/3}}{\delta^{2/3} \ln(3/2) } \ln \left( \frac{ 72 \cdot (10kR)^{2/3}(1+\sqrt{1/2})^{2}\sigma^{2}}{\delta^{2/3}\eps^{2} \ln(3/2) }\right) \right) \right\rceil +1 \\
      &\in \mathcal{O}\left( \frac{\sigma^2}{\eps^{2}}\ln\left( \frac{kR}{\delta}\ln\left(\frac{kR\sigma}{\eps \delta}\right)\right)\right).
\end{align*}}
\end{proof}

\subsection{Stochastic Preference Winner Feedback: Proof of Corollary \ref{Cor:ApplPB}}

The following two lemmata serve as a preparation for the proof of Corollary \ref{Cor:ApplPB}. But first let us introduce the $(k-1)$-simplex 
$$ \mathcal{S}_k = \left\{(p_{i})_{i\in [k]} \in [0,1]^k  \, : \, \sum\nolimits_{i=1}^k p_{i}=1 \wedge \forall i: p_{i} \geq 0 \right\}.$$
\begin{Le}[Dvoretzky-Kiefer-Wolfowitz inequality for categorical random variables]\label{Le_DKW}
    Let $\{X_{t}\}_{t\in \N}$ be a sequence of iid random variables $X_{t} \sim \mathrm{Cat}(\mathbf{p})$ for some $\mathbf{p} \in \mathcal{S}_{k}$. For $t\in \N$ let  $\mathbf{\hat{p}}^{t}$ be the corresponding empirical distribution  after the $t$ observations $X_{1},\dots,X_{t}$, i.e., $\hat{p}^{t}_{i} = \frac{1}{t} \sum_{s=1}^{t} \bold{1}_{\{X_{s} = i\}}$ for all $i\in [k]$. Then, we have for any $\eps>0$ and $t\in \N$ the estimate 
	\begin{equation*}
		\P\left( \norm{ \mathbf{\hat{p}}^{t} - \mathbf{p}}_{\infty} > \eps \right) \leq 4 e^{-t\eps^{2}/2}.
	\end{equation*}
\end{Le}
\begin{proof}
	Confer \citep{Dvoretzky1956,Massart1990} as well as Theorem 11.6 in \cite{Kosorok2008}. Moreover, note that the cumulative distribution functions $F$ resp. $\hat{F}^{t}$ of $X_{1}\sim \mathrm{Cat}(\mathbf{p})$ resp. $\mathbf{\hat{p}}^{t}$   fulfill $p_{j} = F(j) - F(j-1)$ and $\hat{p}^{t}_{j} = \hat{F}^{t}(j) - \hat{F}^{t}(j-1)$ and thus 
	\begin{align*}
		|\hat{p}^{t}_{j} - p_{j}| \leq |\hat{F}^{t}(j) - F(j)| + |\hat{F}^{t}(j-1) - F(j-1)|.
	\end{align*}
	for each $j\in [k]$.
\end{proof}

\begin{Le}\label{Le:LambertWBound}
For every $\beta \in [1,e/2]$, $c_{1},c_{2}>0$ the number 
\begin{equation*}
    x\coloneqq \frac{\beta}{c_{1}} \left( \ln \left( 
    \frac{c_{2}e}{c_{1}^{\beta}}\right) + \ln\ln \left( \frac{c_{2}}{c_{1}^{\beta}} \right) \right)
\end{equation*}
fulfills $c_{1}x\geq \ln(c_{2}x^{\beta})$.
\end{Le}
\begin{proof}
    This is Lemma 18 in \cite{Garivier2016}. 
\end{proof}

\begin{proof}[Proof of Corollary \ref{Cor:ApplPB}]
    For $t\in \N$ and $\gamma \in (0,1)$ define  
    \begin{equation*}
        c_{\gamma}(t) \coloneqq \sqrt{\frac{4\ln(2\pi^{2}t^{2}/(3\gamma))}{t}}
    \end{equation*}
    and note that, for any fixed $\gamma$, the function     $c_{\gamma}:\N\ra (0,\infty), t\mapsto c_{\gamma}(t)$ is strictly monotonically decreasing with $\lim_{t\ra \infty} c_{\gamma}(t) =0$. 
    For any $\alpha>0$, $\gamma\in (0,1)$, we obtain via  Lemma \ref{Le:LambertWBound} with the choices $\beta = 1$, $c_{1} = \frac{\alpha^2}{8}$ and $c_{2} = \sqrt{2/(3\gamma)}\pi$ the estimate
    \begin{align*} 
        c_{\gamma}^{-1}(\alpha) &= \min \left\{t\in \N \, : \, 4\ln(2\pi^{2}t^{2}/(3\gamma)) \leq t\alpha^{2}\right\}\\
        &= \min \left\{t\in \N \, : \, \ln\left( \sqrt{2/(3\gamma)}\pi t\right) \leq \frac{\alpha^2 }{8} t\right\}\\
        &\leq \left\lceil \frac{8}{\alpha^{2}} \left( \ln \left( \frac{8\sqrt{2/(3\gamma)}\pi e}{\alpha^{2}} \right) + \ln \ln \left( \frac{8\sqrt{2/(3\gamma)}\pi}{\alpha^{2}} \right) \right) \right\rceil +1.
    \end{align*}
    Now, suppose $\max_{Q\in \mathcal{Q}_{\leq k}(i^{\ast})} \Delta_{(f(|Q|)+1)|Q} \leq \eps$ and that  Algorithm \ref{alg:Framework} is started with a budget $B$ larger than
\begin{align*}
    c_{\delta/R}^{-1}(\eps/2) \cdot R \max\nolimits_{r\in [R]} P_{r}.
\end{align*}
Recall that in this preference-based setting we use as the statistic the empirical mean of the (winner) observations we obtained for arm $i$ after querying $Q$ (with $i\in Q$) for $t$ many times. 
In particular, we set
$$ s_{i|Q}(t) = \frac{w_{i|Q}(t) }{t}   = \frac{1}{t} \sum\nolimits_{t'=1}^{t} o_{i|Q}(t'), $$
where $o_{i|Q}(t') =1$ if arm $i$ is the preferred (or winning) arm among the arms in $Q,$ if $Q$ is queried for the $t'$-th time, and $0$ otherwise.
Thus, $w_{i|Q}(t)$ is the total number of times arm $i$ has won in the query set $Q$ after $t$ queries.
Moreover, $ \gamma_{i|Q}(t) = | s_{i|Q}(t)  - S_{i|Q}|,$ where $S_{i|Q} = p_{i|Q}$ and $o_{i|Q}(t') \sim\mathrm{Cat}(\mathbf{p_{Q}}).$
With this, we obtain for any $t\in \N$ and  any possible sequence of partitions  $(E_{r})_{r\in [R]} \in (\mathcal{Q}_{\leq k})^{R}$ with $\P(\mathbb{A}_{r}(i^{\ast}) = E_{r} \, \forall r\in [R]) >0$ that
    \begin{align*}
    &\P\left( \exists r\in [R] :  \gamma_{i|E_{r}}(t)
    \geq c_{\delta/R}(t)
    \, \Big| \, \mathbb{A}_{r}(i^{\ast}) = E_{r} \, \forall r\in [R] \right) \\
    &\leq \sum\limits_{r\in [R]} \P\left(  \gamma_{i|E_{r}}(t) \geq c_{\delta/R}(t) 
    \, \Big| \, \mathbb{A}_{r}(i^{\ast}) = E_{r} \, \forall r\in [R] \right) \\
    &= \sum\limits_{r\in [R]}\P\left(  \max\nolimits_{i\in E_{r}} \big| \frac{1}{t} \sum\limits_{t'=1}^{t} o_{i|E_{r}}(t')- S_{i|E_{r}} \big|  > \sqrt{\frac{4\ln(2\pi^{2}t^{2}/(3\gamma))}{t}}  
    \, \Big| \, \mathbb{A}_{r}(i^{\ast}) = E_{r} \, \forall r\in [R] \right) \\
    &\leq \frac{6\delta}{\pi^{2}t^{2}},
\end{align*}
where we used Lemma \ref{Le_DKW} in the last inequality.
Using the law of total probability for all possible sequences of partitions $(E_{r})_{r\in [R]},$  we see that the event 
\begin{equation*}
    \mathcal{E} \coloneqq \left\{\exists t\in \N,  r\in [R] :  \ol{\gamma}_{\mathbb{A}_{r}(i^{\ast})}(t)
    \geq c_{\delta/R}(t)  \right\}
\end{equation*}
occurs with probability
\begin{align*} 
    \P(\mathcal{E} ) 
    &\leq \sum_{t\in \N} \sum_{(E_{r})_{r\in [R]}}\P\left( \exists r\in [R]  :
    \ol{\gamma}_{E_{r}}(t) \geq c_{\delta/R}(t) \, \Big| \, \mathbb{A}_{r}(i^{\ast}) = E_{r} \, \forall r\in [R] \right) \\
    &\quad \times \ \P(\mathbb{A}_{r}(i^{\ast}) = E_{r} \, \forall r\in [R])\\
    &\leq \sum\nolimits_{t\in \N} \frac{6\delta}{\pi^{2}t^{2}} \sum\nolimits_{(E_{r})_{r\in [R]}} \P(\mathbb{A}_{r}(i^{\ast}) = E_{r} \, \forall r\in [R] ) \leq \delta.
\end{align*}
On $\mathcal{E}^{c}$ we have $\ol{\gamma}_{\mathbb{A}_{r}(i^{\ast})}(t) < c_{\delta/R}(t)$ for all $t\in \N, r\in [R]$ and thus in particular $\ol{\gamma}_{\mathbb{A}_{r}(i^{\ast})}^{-1}(\alpha) \geq c_{\delta/R}^{-1}(\alpha)$ for any $\alpha \in (0,\infty)$. Since  $\max_{Q\in \mathcal{Q}_{\leq k}(i^{\ast})} \Delta_{(f(|Q|)+1)|Q} \leq \eps$, Theorem \ref{Thm_generalSufficient} thus lets us conclude
\begin{align*}
    \P\left( \text{Alg. 1 returns } i^{\ast} \right)
    &\geq  \P\left( B > R \max\nolimits_{r\in [R]} P_{r} \ol{\gamma}_{\mathbb{A}_{r}(i^{\ast})}^{-1} \left( \frac{\Delta_{(f(|\mathbb{A}_{r}(i^{\ast})|)+1)|\mathbb{A}_{r}(i^{\ast})}}{2}\right) \right)\\
    &\geq  \P\left( \left\{ B > R \max\nolimits_{r\in [R]} P_{r} \ol{\gamma}_{\mathbb{A}_{r}(i^{\ast})}^{-1} \left( \eps/2\right) \right\} \cap \mathcal{E}^{c}\right)\\
     &\geq  \P\left( \left\{ B > R \max\nolimits_{r\in [R]} P_{r} c_{\delta/R}^{-1} \left( \eps/2\right) \right\} \cap \mathcal{E}^{c}\right)\\
     &= \P(\mathcal{E}^{c}) \geq 1-\delta,
\end{align*}
where the equality holds due to the assumption on $B$. Consequently, the statement holds with
\begin{align*}
      C(\delta,\eps,k,R) &\coloneqq c_{\delta/R}^{-1}(\eps/2)\\
      &\leq \left\lceil \frac{32}{\eps^{2}} \left( \ln \left( \frac{32\sqrt{2R/(3\delta)}\pi e}{\eps^{2}} \right) + \ln \ln \left( \frac{32\sqrt{2R/(3\delta)}\pi}{\eps^{2}} \right) \right) \right\rceil +1 \\
      &\in \mathcal{O}\left( \frac{1}{\eps^{2}}\ln\left( \frac{R}{\delta \eps^{4}}\right)\right).
\end{align*}
\end{proof}

\clearpage

\section{Comparisons of the Algorithms} \label{sec_comparison}
In the following we summarize the theoretical results obtained for our proposed algorithms in a concise way. 
First of all, we give an overview of the individual key quantities of each algorithm in Table \ref{tab:AlgoComparison2}, where we assume w.l.o.g.\ that ${n \choose k}$ is a divisor of $B$ in \textsc{RoundRobin} to make the assignments of $R$, $P_r$ and $f(s)$ for \textsc{RoundRobin} well-defined.
The maximal number of different query sets is derived in Section \ref{sec:diff_query_sets}.
\begin{table*} [!ht]
    \caption{Comparison of the maximal number of rounds, the maximal number of partitions per round, the amount of retained arms from each partition and the maximal number of query sets for \textsc{RoundRobin} and our proposed algorithms CSWS, CSR and CSH.}
    \label{tab:AlgoComparison2}
    \centering
	\resizebox{0.99\textwidth}{!}{
    \begin{tabular}{lllll}
        \toprule[1pt]
        Alg. & $R$ & $P_{r}$ & $f(x)$ & $\max \#$query\_sets \\
        \midrule[0.5pt]
        \textsc{RoundRobin} 
        & $1$ 
        & ${n \choose k}$ 
        & $x \mapsto x$ 
        & ${n \choose k}$ \\
        \textsc{CSWS} 
        & $\left\lceil \log_k(n) \right\rceil +1$ 
        & $\left\lceil \frac{n}{k^r} \right\rceil$ 
        & $x \mapsto 1$ 
        & $ R^{CSWS} + n \cdot \left(\frac{1 - 1/k^{\lceil\log_k(n)\rceil + 1}}{k-1 } \right) $ \\
        \textsc{CSR} 
        & $\left\lceil\log_{1-\frac{1}{k}}\left(\frac{1}{n}\right)\right\rceil+k-1$ 
        & $\left\lceil\frac{n(1-\frac{1}{k})^{(r-1)}}{k}\right\rceil$ 
        & $x \mapsto x-1$ 
        & $  R^{CSR} + n\left(1-\left(1-\frac{1}{k}\right)^{\lceil\log_{1-\frac{1}{k}}\left(\frac{1}{n}\right)\rceil +k-1}\right) $ \\
        \textsc{CSH} 
        & $\left\lceil \log_2(n) \right\rceil + \left\lceil \log_2(k) \right\rceil$ 
        & $\left\lceil \frac{n}{2^{r-1} k} \right\rceil$ 
        & $x \mapsto \left\lceil \frac{x}{2} \right\rceil$ 
        & $R^{CSH} + \frac{2n}{k}\left( 1-1/2^{\lceil \log_2(n) \rceil + \lceil \log_2(k) \rceil} \right) $ \\
        \bottomrule[1pt]
    \end{tabular}}
\end{table*}

Using Remark \ref{remark:rough_suff_budget} we can derive the following sufficient budgets of the algorithms summarized in the following table, where $\pi(Q) \in Q$ be the $\lfloor \frac{|Q|}{2} \rfloor +1 $- th best arm with respect to $(S_{i |Q})_{i \in Q}.$ 

\begin{table*}   [!ht]
\caption{Comparison of the sufficient budget for \textsc{RoundRobin} and our proposed algorithms CSWS, CSR and CSH.}
  \label{tab:AlgoComparison}
      \centering
      
	\resizebox{0.99\textwidth}{!}{
    \begin{tabular}[t]{ll}
        \toprule[1pt]
        Algorithm & Sufficient budget \\
        \midrule[0.5pt]
        \rule{0pt}{20pt} \textsc{RoundRobin}
        & ${n \choose k} \max_{i\in \mathcal{A}, i \neq i_{\mathcal B}^{\ast}  } \left(\hat{\gamma}_{i_{\mathcal B}^{\ast}  ,i}^{\max}\right)^{-1} \left(\frac{S_{i_{\mathcal B}^{\ast}  }^{\mathcal B} - S_{i}^{\mathcal B}}{2}\right)$ \\
        \rule{0pt}{20pt} CSWS 
        & $ \left\lceil \frac{n}{k} \right\rceil (\lceil \log_k(n) \rceil +1)  \cdot \max_{Q \in \mathcal{Q}_{\leq k} : i^{\ast}\in Q} \max_{i\in Q \setminus \{i^{\ast}\}} \left\lceil \bar{\gamma}^{-1} \left( \frac{S_{i^{\ast}|Q} - S_{i|Q}}{2}\right) \right\rceil$  \\
        \rule{0pt}{20pt} CSR 
        & $\left\lceil\frac{n}{k}\right\rceil \left(\left\lceil\log_{1-\frac{1}{k}}\left(\frac{1}{n}\right)\right\rceil+k-1\right)$
        $ \cdot \max_{Q \in \mathcal{Q}_{\leq k}: i^{\ast}\in Q}  \min_{i\in Q\backslash \{i^{\ast}\}} \left\lceil \bar{\gamma}^{-1} \left( \frac{S_{i^{\ast}|Q} - S_{i|Q}}{2}\right) \right\rceil$\\
        \rule{0pt}{20pt} CSH 
        & $ \lceil \frac{n}{k} \rceil \left(\left\lceil \log_2(n) \right\rceil + \lceil \log_2(k) \rceil \right)$
        $\cdot\max_{Q \in \mathcal{Q}_{\leq k}: i^{\ast}\in Q} \left\lceil \bar{\gamma}^{-1} \left( \frac{S_{i^{\ast}|Q} - S_{\pi(Q)|Q}}{2}\right) \right\rceil$  \\
         \bottomrule[1pt]
    \end{tabular}}
\end{table*}

In Section \ref{sec:comp_suff_budget} we compare these quantities for the special case, in which the gaps $\Delta_{i|Q} = S_{i^{\ast}|Q} - S_{i|Q}$ are all equal to some $\Delta>0,$ while in Section \ref{sec:application_stoch} we derive the sufficient budgets resulting from Corollaries \ref{Cor:ApplReward} and \ref{Cor:ApplPB} for the reward setting and preference-based setting, respectively, to return the best arm with high probability in the stochastic setting.
Note that if $\gamma_{i|Q}(t) = \gamma(t)$ and $S_{(2)|Q} = \dots = S_{(|Q|)|Q}$ are fulfilled for all $Q\in \mathcal{Q}_{\leq k}, i\in Q$ and $t\in \N$, then 
the lower bound in Theorem \ref{Thm_Main_LB_GCW} (i) matches the above upper bound for CSWS up to a factor $C = \lceil \log_{k}(n)\rceil +1$.

\subsection{Maximal Number of Different Query Sets} \label{sec:diff_query_sets}
The maximal number of required query sets for each algorithm is $\sum_{r=1}^R P_r$. Note that this is a geometric series and thus the partial sum can easily be computed for each of our proposed algorithms.
\paragraph{CSWS}
By using the specified valued of $R$ and $P_r$ for CSWS, we obtain that the number of different query set is at most
\begin{align*}
    \sum_{r=1}^{R^{CSWS}} P_r^{CSWS} &= \sum_{r=1}^{\lceil\log_{k}(n)\rceil +1} \left\lceil \frac{n}{k^r} \right\rceil \\
    &\le  \lceil \log_{k}(n) \rceil + 1 + \sum_{r=1}^{\lceil\log_{k}(n)\rceil +1} \frac{n}{k^r} \\
    &=  \lceil \log_{k}(n) \rceil + 1 + n \cdot \left(\sum_{r=0}^{\lceil\log_{k}(n)\rceil +1} \left(\frac{1}{k} \right)^{r} -  1 \right) \\
    &=  \lceil \log_{k}(n) \rceil + 1 + n \cdot \left(\frac{1 - 1/k^{\lceil\log_k(n)\rceil + 2}}{1- 1/k } -1\right)  \\
    &= \lceil \log_{k}(n) \rceil + 1 + n \cdot \left(\frac{1 - 1/k^{\lceil\log_k(n)\rceil + 1}}{k-1 } \right) ,
\end{align*}
where we used for the inequality that $\lceil x \rceil \leq x+1$ for any $x\in \R.$

\paragraph{CSR} For CSR we get as an upper bound on the number of different query sets:
\begin{align*}
    \sum_{r=1}^{R^{CSR}} P_r^{CSR} &= \sum_{r=1}^{\lceil\log_{1-\frac{1}{k}}\left(\frac{1}{n}\right)\rceil +k-1} \left\lceil \frac{n\left(1-\frac{1}{k}\right)^{r-1}}{k} \right\rceil \\
    & \le  \left\lceil\log_{1-\frac{1}{k}}\left(\frac{1}{n}\right)\right\rceil +k-1 + \sum_{r=1}^{\lceil\log_{1-\frac{1}{k}}\left(\frac{1}{n}\right)\rceil +k-1} \frac{n\left(1-\frac{1}{k}\right)^{r-1}}{k}  \\
    &= \left\lceil\log_{1-\frac{1}{k}}\left(\frac{1}{n}\right)\right\rceil +k-1 + \frac{n}{k} \sum_{r=0}^{\lceil\log_{1-\frac{1}{k}}\left(\frac{1}{n}\right)\rceil +k-2} \left(1-\frac{1}{k}\right)^{r} \\
    &= \left\lceil\log_{1-\frac{1}{k}}\left(\frac{1}{n}\right)\right\rceil +k-1 + \frac{n}{k} \frac{ \left(1-\left(1-\frac{1}{k}\right)^{\lceil\log_{1-\frac{1}{k}}\left(\frac{1}{n}\right)\rceil +k-1}\right)}{\left(1- \left(1 - \frac{1}{k}\right)\right)} \\
    &= \left\lceil\log_{1-\frac{1}{k}}\left(\frac{1}{n}\right)\right\rceil +k-1 + n\left(1-\left(1-\frac{1}{k}\right)^{\lceil\log_{1-\frac{1}{k}}\left(\frac{1}{n}\right)\rceil +k-1}\right) .
\end{align*}

\paragraph{CSH}
Similarly, we can obtain for CSH the following maximum number of different query sets:
\begin{align*}
    \sum_{r=1}^{R^{CSH}} P_r^{CSH}
    &= \sum_{r=1}^{\lceil \log_2(n) \rceil + \lceil \log_2(k) \rceil} \left\lceil \frac{n}{2^{r-1}k} \right\rceil \\
    &\le \lceil \log_2(n) \rceil + \lceil \log_2(k) \rceil + \sum_{r=1}^{\lceil \log_2(n) \rceil + \lceil \log_2(k) \rceil} \frac{n}{2^{r-1}k}   \\
    &= \lceil \log_2(n) \rceil + \lceil \log_2(k) \rceil + \frac{n}{k} \sum_{r=0}^{\lceil \log_2(n) \rceil + \lceil \log_2(k) \rceil - 1} \left(\frac{1}{2}\right)^{r} \\
    &= \lceil \log_2(n) \rceil + \lceil \log_2(k) \rceil + \frac{2n}{k}\left( 1-1/2^{\lceil \log_2(n) \rceil + \lceil \log_2(k) \rceil} \right).
\end{align*}

These upper bounds on the maximum number of different query sets are summarized in Table \ref{tab:AlgoComparison2}.
Note that \textsc{RoundRobin} by design queries the possible query sets of $\mathcal{Q}_{=k}$ in a round-robin fashion, so that the number of different query sets is indeed  $ |\mathcal{Q}_{=k}| =  {n \choose k}.$

\subsection{Comparison of Sufficient Budgets} \label{sec:comp_suff_budget}

In order to compare the derived sufficient budgets of the different algorithms (see Table \ref{tab:AlgoComparison}), we consider in the following the setting where the generalized Condorcet winner coincides with the generalized Borda winner. In addition we assume that the limit statistic $S_{i|Q}$ for each arm $i \in \mathcal{A}$ has always the same difference to the limit of the optimal arm $S_{i^{\ast}|Q}$ if $i^{\ast} \in Q$. 
More precisely, for each arms $i \in \mathcal{A}$ and each query set $Q \in \mathcal{Q}_{\le k}$ we have $\Delta_{i|Q} = \Delta$ for some fixed $\Delta > 0$. 
In this way, the $\gamma$-dependent term present in the sufficient budget for each algorithm  is simply $\left\lceil \bar{\gamma}^{-1} \left( \frac{\Delta}{2}\right) \right\rceil$. 
As a consequence, we can neglect this term as it has no influence on the differences in the desired budgets for the various algorithms and the remaining term based on the product of the number of rounds, i.e.\ $R,$ and the number of partitions in round $1$, i.e. $P_1,$ is driving the (rough) sufficient budget bounds (see Table \ref{tab:AlgoComparison2}).
However, the number of partitions in round $1$ is the same for all algorithms, so that we can neglect this term as well.
With a slight abuse of denotation, we refer to this remainder term simply as the sufficient budget in the following.
With these considerations, it is easy to see that \textsc{RoundRobin} requires the highest sufficient budget even for moderate sizes of $n$ if $k$ is sufficiently lower than $n$. 
To get an impression how the sufficient budget behaves for the more sophisticated algorithms based on the successive elimination strategy, we plot these in Figure \ref{fig:SuffBudgetComparison} as curves depending on the number of arms $n$ for different  subset sizes $k.$
Note, that in contrast to CSWS and CSH, the sufficient budget of CSR is higher for bigger subset sizes $k$, since only a smaller proportion of all arms is discarded after each round.
In the case $k=2$ the number of rounds are all the same, so that consequently the sufficient budget is the same for all three algorithms.

\begin{figure}[!ht]
    \centering
    \includegraphics[width=0.9\linewidth]{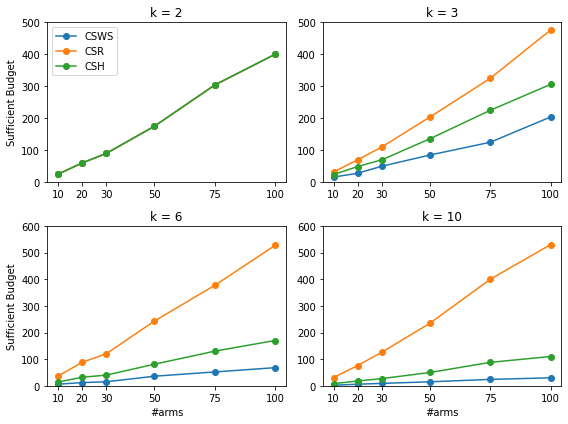}
    \caption{Comparison of required budget for our proposed algorithms for different values of the number of arms $n$ and the subset size $k$.}
    \label{fig:SuffBudgetComparison}
\end{figure}

\clearpage
\subsection{Applications to Stochastic Settings} \label{sec:application_stoch}
In Table \ref{tab:application_budgets} the sufficient budgets for our proposed algorithms in the stochastic setting with reward feedback and preference-based feedback are listed. 
Note, that these results are simply derived by applying Corollary \ref{Cor:ApplReward} and resp.\ Corollary \ref{Cor:ApplPB} with the specific instantiations of $R$ and $P_r$ for our algorithms (see Tables \ref{tab:AlgoComparison2} and \ref{tab:AlgoComparison}).
\begin{table*}[!ht]
    \caption{Comparison of the sufficient budgets for our proposed algorithms CSWS, CSR and CSH in the reward and preference-based setting.}
    \label{tab:application_budgets}
    \centering
	\resizebox{0.99\textwidth}{!}{
    \begin{tabular}[t]{ll}
        \toprule[1pt]
        Alg. & Budget in reward setting  \\
        \midrule[0.5pt]
        \textsc{CSWS} 
        & $\frac{1}{\epsilon^2} \ln\left( \frac{k\left(\left\lceil \log_k(n) \right\rceil +1\right)}{\delta} \ln\left( \frac{k \left(\left\lceil \log_k(n) \right\rceil +1\right)}{\epsilon\delta} \right)\right) \cdot \left(\left\lceil \log_k(n) \right\rceil +1\right) \left\lceil \frac{n}{k} \right\rceil$ \\
        \textsc{CSR} 
        & $\frac{1}{\epsilon^2} \ln\left( \frac{k \left(\left\lceil\log_{1-\frac{1}{k}}\left(\frac{1}{n}\right)\right\rceil+k-1\right) }{\delta} \ln\left( \frac{k \left(\left\lceil\log_{1-\frac{1}{k}}\left(\frac{1}{n}\right)\right\rceil+k-1\right)}{\epsilon\delta} \right)\right) \cdot \left(\left\lceil\log_{1-\frac{1}{k}}\left(\frac{1}{n}\right)\right\rceil+k-1\right) \left\lceil \frac{n}{k} \right\rceil $ \\
        \textsc{CSH} 
        & $\frac{1}{\epsilon^2} \ln\left( \frac{k \left(\left\lceil \log_2(n) \right\rceil + \left\lceil \log_2(k) \right\rceil \right)}{\delta} \ln\left( \frac{k \left(\left\lceil \log_2(n) \right\rceil + \left\lceil \log_2(k) \right\rceil \right)}{\epsilon\delta} \right)\right) \cdot \left(\left\lceil \log_2(n) \right\rceil + \left\lceil \log_2(k) \right\rceil \right) \left\lceil \frac{n}{k} \right\rceil$ \\
        \hline
        \rule{0pt}{15pt} Alg. & Budget in preference-based setting \\
        \hline
        \textsc{CSWS}
        & $\frac{1}{\epsilon^2} \ln \left( \frac{\left\lceil \log_k(n) \right\rceil +1}{\delta \epsilon^4} \right) \cdot \left(\left\lceil \log_k(n) \right\rceil +1\right) \left\lceil \frac{n}{k} \right\rceil$ \\
        \textsc{CSR}
        & $\frac{1}{\epsilon^2} \ln \left( \frac{\left\lceil\log_{1-\frac{1}{k}}\left(\frac{1}{n}\right)\right\rceil+k-1}{\delta \epsilon^4} \right) \cdot \left(\left\lceil\log_{1-\frac{1}{k}}\left(\frac{1}{n}\right)\right\rceil+k-1\right) \left\lceil \frac{n}{k} \right\rceil $\\
        \textsc{CSH}
        & $\frac{1}{\epsilon^2} \ln \left( \frac{\left\lceil \log_2(n) \right\rceil + \left\lceil \log_2(k) \right\rceil}{\delta \epsilon^4} \right) \cdot \left(\left\lceil \log_2(n) \right\rceil + \left\lceil \log_2(k) \right\rceil \right) \left\lceil \frac{n}{k} \right\rceil$ \\
        \bottomrule[1pt]
    \end{tabular}}
\end{table*}

\clearpage
\section{Further Experiments} \label{sec:further_exp}
In the following, we present some further experiments comparing our proposed algorithms with each other on synthetic data including a detailed description of the data generation and the experiment setting.
\subsection{Synthetic Data} \label{sec:syntheticdata}
For each $Q \in \mathcal{Q}_{\leq k}$ with $Q=\{i_1,\ldots,i_{|Q|}\}$ we consider the case where the observation vector $\mathbf{o}_{Q}$ is a random sample from a multivariate Gaussian distribution with mean $\mu_Q =( \mu_{i_1|Q}, \ldots,      \mu_{i_{|Q|}|Q} )^\top$ and a diagonal covariance matrix $\mathrm{diag}(\sigma_{i_1|Q}, \ldots,\sigma_{i_{|Q|}|Q} ).$
Here, $\mu_{i_j|Q}$ are values in $[0,1]$ for $i_j \neq i^*$ and $\sigma_{i|Q}$ in $[0.05, 0.2]$ (all randomly sampled).
For any $Q$ with $i^* \in Q$ we set $\mu_{i^*|Q} = \max_{j \in Q, j\neq i^*} \mu_{j|Q} + \varepsilon$ for some $\varepsilon>0,$ which ensures (A2) to hold for the expected values.  In our experiments we always use a value of $\varepsilon=0.1$. 
In the following we vary the values of $n \in \{50,100\}$, $k \in \{2,4,6,8,10\}$ and $B \in \{50, 100, 200, 300, 500\}$.

We consider a reward setting and use the empirical mean as the statistic (see Section \ref{sec:application_stochastic}). 
We do not force the generalized Borda winner to be the same as the generalized Condorcet winner, but they naturally coincidence in most of the runs by sampling the observation vector as defined above. \par
\begin{figure*}[ht]
\centering
    \includegraphics[width=0.99\linewidth]{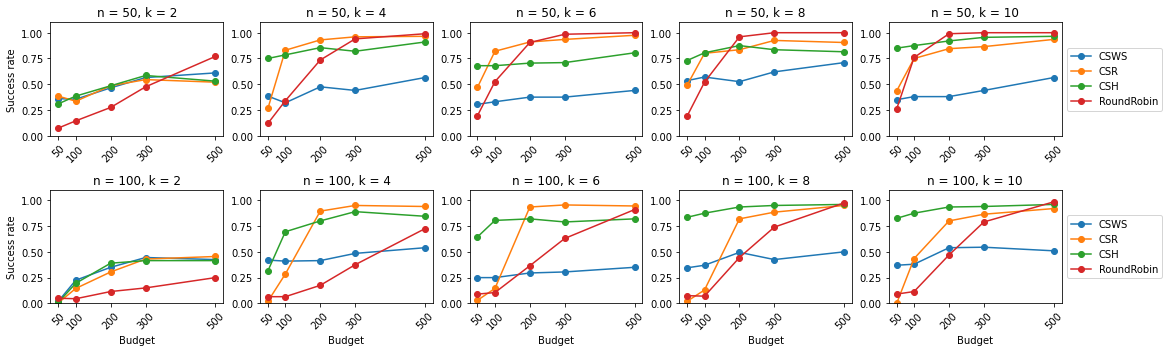}
    \caption{Success rates of our proposed algorithms for varying $n$, $k$ and budget $B$ in the reward setting.}
    \label{fig:SuccessRateRewardSetting}
\end{figure*}
The success rates of our proposed algorithms for identifying $i^*$ given a budget $B$ are shown in Figure \ref{fig:SuccessRateRewardSetting}. 
It is visible, that in particular for the challenging scenario, where the budget $B$ and the subset size $k$ are small and the number of arms $n$ is large, both CSH and CSR perform well.
Especially CSH has overall a solid performance.

\paragraph{Reward setting.} In contrast to the experiments with reward feedback shown in the main paper, we try in the following experiments to force the generalized Borda winner to be different from the generalized Condorcet winner. For this purpose, we fix one random arm $i_{\mathcal B}^{\ast}   \in [n]\backslash \{i^{\ast}\}$ as the prospective generalized Borda winner and set its expected value to $\mu_{i_{\mathcal B}^{\ast}  |Q} = \max_{j\in Q, j\neq i_{\mathcal B}^{\ast}  } \mu_{j|Q} + 2\epsilon$ for any $Q \in \mathcal{Q}_{\leq k}$ with $i_{\mathcal B}^{\ast}   \in Q $ and $i^{\ast} \notin Q$. Thus, $i_{\mathcal B}^{\ast}  $ is likely the generalized Borda winner and is different from the generalized Condorcet winner. Since our goal is to find the generalized Condorcet winner $i^{\ast}$, \textsc{RoundRobin} will probably fail most of the times in finding $i^{\ast}.$ 
This is due to the fact that \textsc{RoundRobin} focuses on identifying $i_{\mathcal B}^{\ast}  ,$ i.e., the the generalized Borda winner, which, however, does not coincidence with the generalized Condorcet winner $i^{\ast}$.

This suspicion is confirmed by the results of the experiments shown in Figure \ref{fig:SuccessRewardBordaCondorcetDifferent} illustrating the empirical success rates for finding the generalized Condorcet winner in the setting described above.
Except for some cases where the subset size $k$ is relatively large in comparison to the total number of arms, such that the generalized Condorcet winner is already contained in most of the seen subsets and hence is automatically also the generalized Borda winner, \textsc{RoundRobin} performs poorly in finding the generalized Condorcet winner and is always outperformed by the algorithms based on the combinatorial successive elimination strategy in Section \ref{sec:gen_framework}.

\begin{figure}[!ht]
    \centering
    \includegraphics[width=\linewidth]{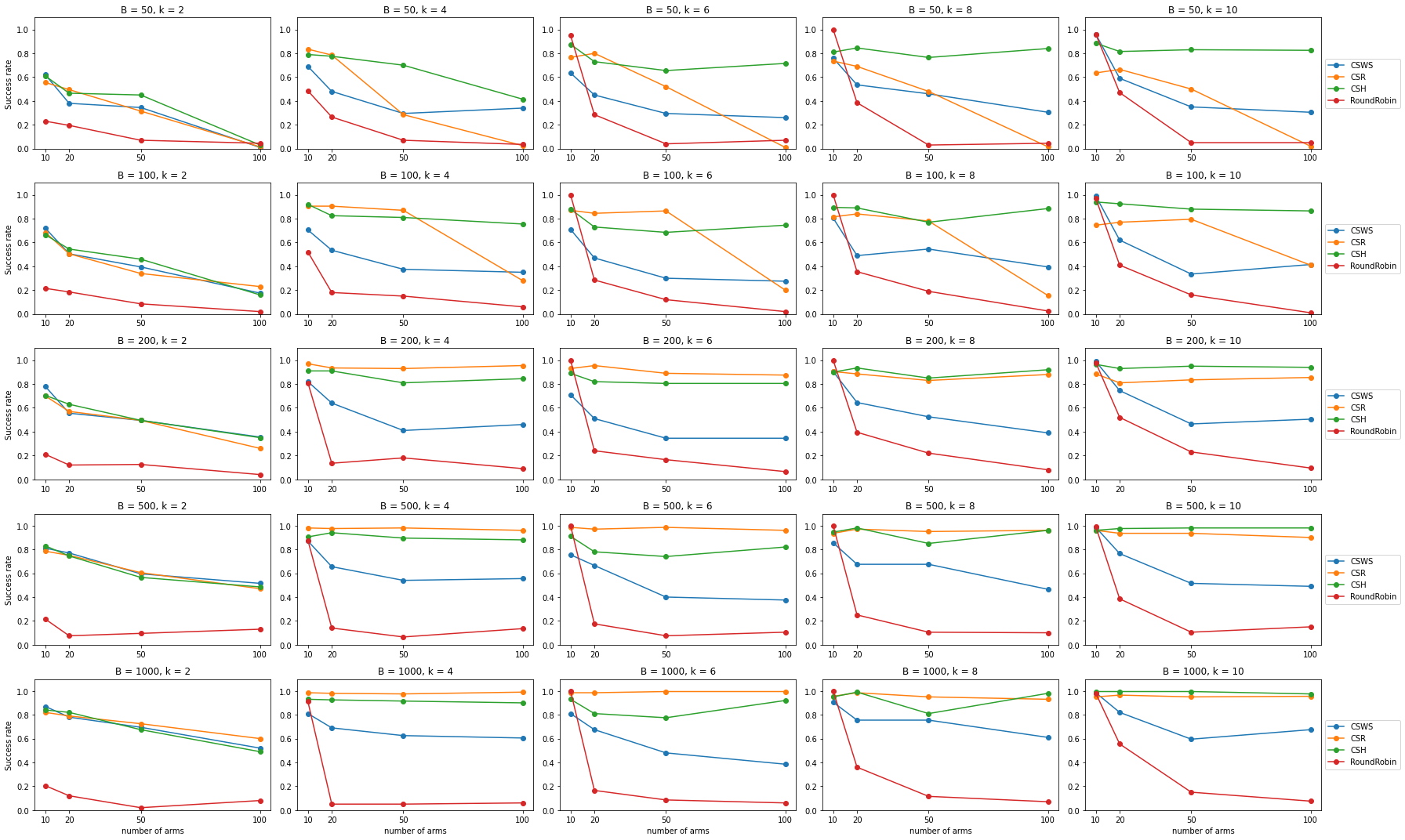}
    \caption{Success rates of our proposed algorithms for varying $n$, $k$ and budget $B$ in the reward setting with different generalized Condorcet winner and generalized Borda winner.}
    \label{fig:SuccessRewardBordaCondorcetDifferent}
\end{figure}

\paragraph{Preference-based setting with different GCW and GBW.}
In the preference-based setting we ignore the explicit numerical values of the observation vector and only use the information which arm was (not) the winner, i.e., which had (not) the highest observation value in the query set used , formally $s_{i_{j}|Q}(t) = \frac{1}{t} \sum_{s=1}^{t} \mathds{1}\{o_{i_{j}|Q}(s) = \max_{i = i_1, \dots i_{|Q|}} o_{i|Q}(s)\}.$ 
Additionally, we fix one arm $i_{\mathcal B}^{\ast}   \in [n] \backslash \{i^{\ast}\}$ and set $\mu_{i_{\mathcal B}^{\ast}  |Q} = \max_{j \in Q, j\neq i_{\mathcal B}^{\ast}  } \mu_{j|Q} + 2\varepsilon$ for any $Q$ with $i_{\mathcal B}^{\ast}   \in Q$ and $i^{\ast} \notin Q.$
In this way, $i_{\mathcal B}^{\ast}  $ is the generalized Borda winner and different from $i^{\ast}.$
\par
The success rates of our proposed algorithms for identifying $i^*$ in this setting are shown in Figure \ref{fig:SuccessRatePBSetting}. 
As expected our methods outperform \textsc{RoundRobin} in all scenarios.

\begin{figure*}[t]
    \includegraphics[width=\linewidth]{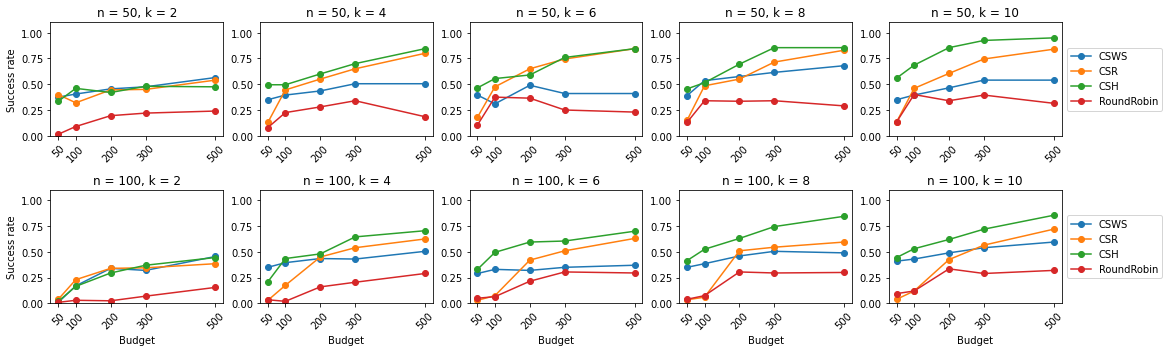}
    \caption{Success rates of our proposed algorithms for varying $n$, $k$ and budget $B$ in the preference-based setting with different generalized Condorcet winner and generalized Borda winner.}
    \label{fig:SuccessRatePBSetting}
\end{figure*}

\paragraph{Preference-based setting.} 
We now investigate the case, in which we do not force the generalized Borda winner and the generalized Condorcet winner to be different, thus they will naturally coincidence in most of the cases. 
This is achieved by considering the problem configuration as in the reward setting specified in Section \ref{sec:experiments}, and ignoring the explicit numerical values (as in the preference-setting above).

The resulting success rates for finding the generalized Condorcet winner illustrated in Figure \ref{fig:SuccessPBWinnerEqual} are similar to the results in the reward setting for matching generalized Condorcet winner and generalized Borda winner.
This means that, in particular, when the budget is small, the number of arms is large and the  subset size is  small, the algorithms following the combinatorial successive elimination strategy outperform \textsc{RoundRobin}.
Note that this setting is arguably the most relevant setup for practical applications.
Moreover, Figure \ref{fig:SuccessPBWinnerEqual} illustrates the natural effect one would expect for the number of arms $n$ on success rates, namely that success rates decrease with a larger number of arms.

\begin{figure}[ht]
    \centering
    \includegraphics[width=\linewidth]{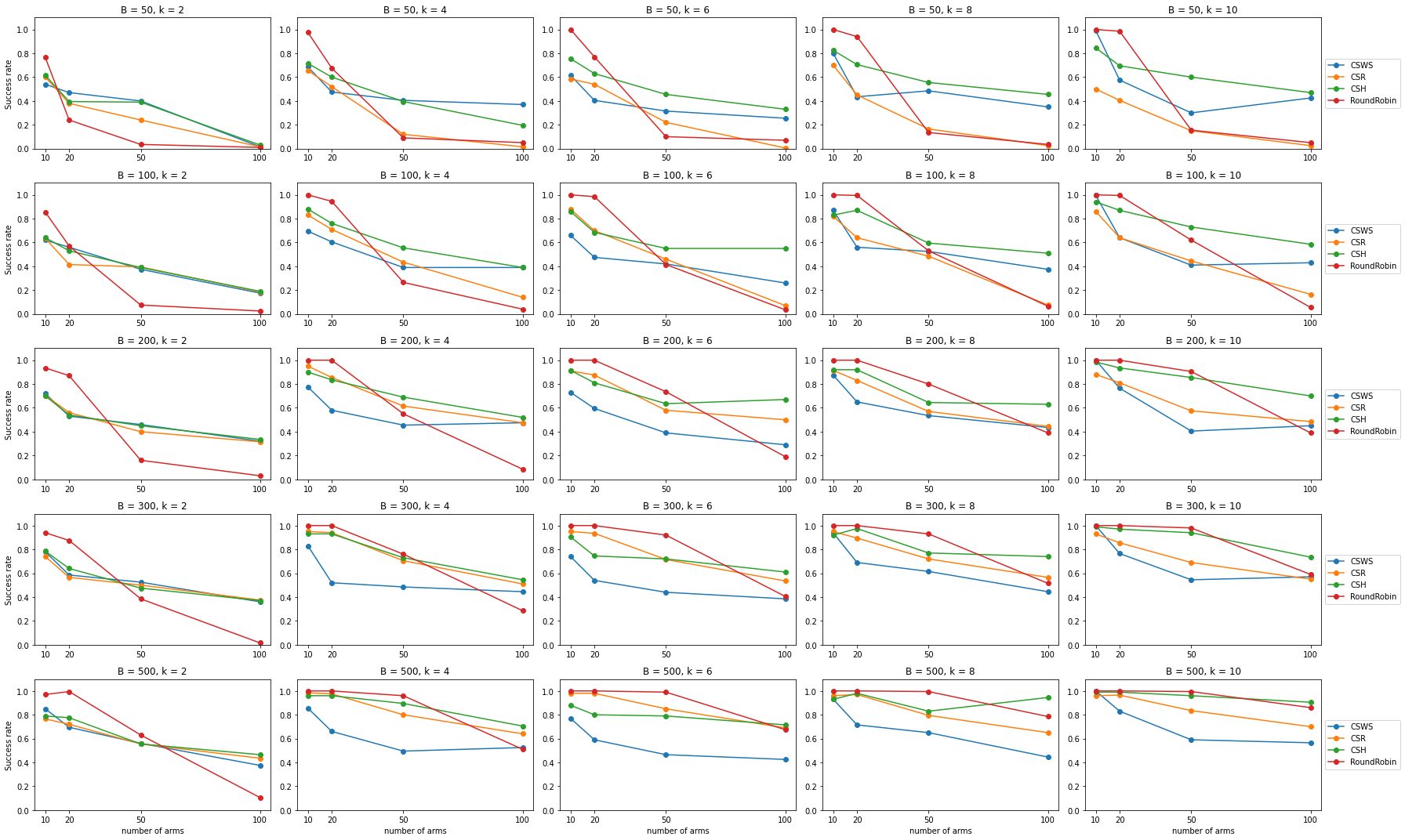}
    \caption{Success rates of our proposed algorithms for varying $n$, $k$ and budget $B$ in the preference-based setting with (mostly) matching generalized Condorcet winner and generalized Borda winner.}
    \label{fig:SuccessPBWinnerEqual}
\end{figure}

\clearpage

\subsection{Statistics beyond the Arithmetic Mean}
We consider in the following the reward setting, where each observation is random sampled from the following distribution
\begin{align*}
    \bm{o}_Q(t) \sim \mathcal{N}\left( \left( \begin{array}{c}
         \mu_{1|Q}  \\
         \vdots \\
         \mu_{|Q| | Q}
    \end{array} \right), \left( \begin{array}{c}
         \sigma_{1|Q} \\
         \vdots \\
         \sigma_{|Q| | Q}
    \end{array} \right)
    \right)
\end{align*}
for $\mu_{i|Q}$ is sampled randomly from $[0,1]$ and $\sigma_{i|Q}$ from $[0.05, 0.2]$ for each arm $i \in Q$.
\paragraph{Median}
An alternative to the arithmetic mean would be to measure the quality of the arms by the median of the seen observations. In particular, when the observations are prone to outliers, the median provides a more robust statistic: $s_{i|Q}(t) = \textsc{Median}(o_{i|Q}(1), \dots, o_{i|Q}(t))$ for each arm $i \in Q$.
The results for this setting are illustrated in Figure \ref{Median}.

\begin{figure}[h]
    \centering
    \includegraphics[width=\linewidth]{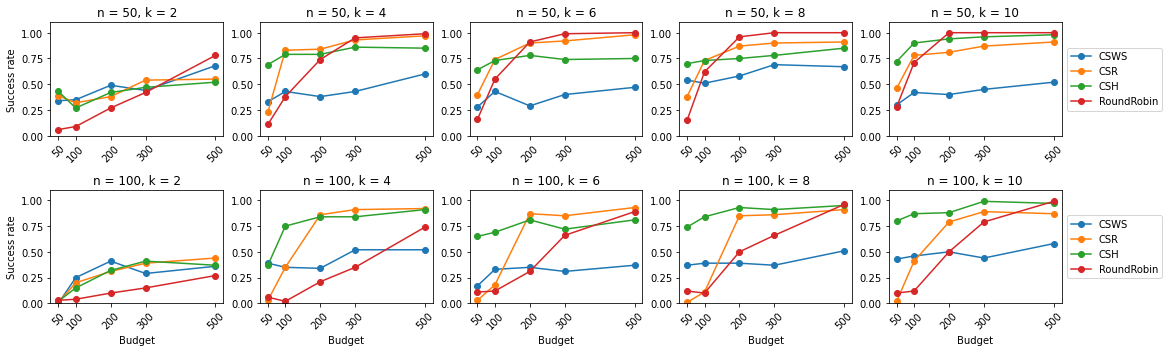}
    \caption{Success rates of our proposed algorithms for varying $n$, $k$ and budget $B$ in the rewards setting with (mostly) matching generalized Condorcet winner and generalized Borda winner and using the median as the statistic.}
    \label{Median}
\end{figure}

\paragraph{Power-Mean}
Another possibility is to use the so called power-mean, which is a compromise between the maximum and the arithmetic mean for a (multi)set of observations. 
Since the arithmetic mean is known to underestimate the true quality of an arm, while the maximum overestimates it, the power mean is often a good compromise, as it lies between the two.
It is defined by $s_{i|Q}(t) = \left(\frac{1}{t}\sum_{t'=1}^t o_{i|Q}(t')^q\right)^{1/q}$ for each arm $i \in Q$ and a fixed $q \in \mathbb{N}$. We use in the following $q = 2$.
The results for this setting are illustrated in Figure \ref{PowerMean}.

\begin{figure}[h]
    \centering
    \includegraphics[width=\linewidth]{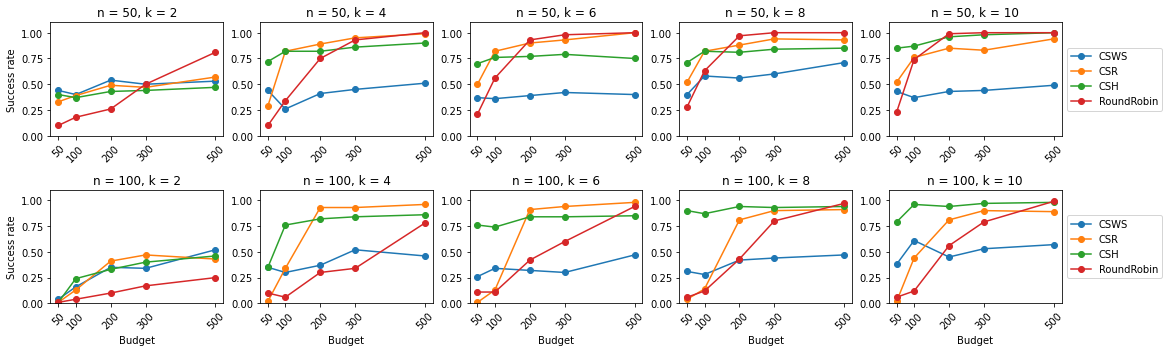}
     \caption{Success rates of our proposed algorithms for varying $n$, $k$ and budget $B$ in the rewards setting with (mostly) matching generalized Condorcet winner and generalized Borda winner and using the power mean as the statistic.}
    \label{PowerMean}
\end{figure}

\end{document}